\documentclass{amsart}
\usepackage[a4paper, margin=1in]{geometry}

\usepackage[utf8]{inputenc} 
\usepackage[T1]{fontenc}    
\usepackage{hyperref}       
\usepackage{url}            
\usepackage{booktabs}       
\usepackage{amsfonts}       
\usepackage{nicefrac}       
\usepackage{microtype}      
\usepackage{xcolor}         

\usepackage{graphicx}       
\usepackage{subcaption}      
\usepackage{thmtools}        
\usepackage{thm-restate}    

\usepackage[round, authoryear]{natbib}

\newtheorem{theorem}{Theorem}                   
\newtheorem{lemma}[theorem]{Lemma}              
    
\newtheorem{definition}{Definition}             
\newtheorem{proposition}{Proposition}           
             
\newtheorem{corollary}{Corollary}               

\theoremstyle{definition} 
\newtheorem{example}{Example}

\title[Global linear convergence of softmax policy gradient]{Global linear convergence of entropy-regularized softmax policy gradient beyond tabular MDPs}

\author{Ziyue Chen}
\address{School of Mathematics\\University of Edinburgh\\Edinburgh, UK, EH9 3FD}
\email{s2674476@ed.ac.uk}

\author{David {\v{S}}i{\v{s}}ka}
\address{School of Mathematics\\University of Edinburgh\\Edinburgh, UK, EH9 3FD}
\email{d.siska@ed.ac.uk}

\author{{\L}ukasz Szpruch}
\address{School of Mathematics\\University of Edinburgh\\Edinburgh, UK, EH9 3FD}
\email{L.Szpruch@ed.ac.uk}

\begin{document}

\begin{abstract}
We study the global convergence of policy gradient for infinite-horizon entropy-regularized Markov decision processes (MDPs) with continuous state and action spaces.
We consider log-linear softmax policies with linear function approximation, which extend the tabular softmax parameterization while retaining a tractable  policy class. 
Under $Q^\pi_\tau$-realizability up to a state-dependent baseline for the regularized state-action value function, we establish a non-uniform Polyak--{\L}ojasiewicz (P\L) inequality for two different feature regimes.
The non-uniformity arises through degeneracy of the state-averaged Fisher information matrix in the first feature regime and the uncentered feature second-moment matrix in the second. 
In both regimes we prove that these non-uniform constants can be bounded uniformly from below along the gradient flow.
For full-affine-span features, we prove radial unboundedness of the KL regularizer and show that the state averaged smallest eigenvalue of the Fisher information matrix remains bounded below by an initialization-dependent positive constant.
For simplex-valued features, we prove an analogous radial unboundedness result in the subspace orthogonal to the one vector and obtain an initialization dependent lower bound for the state averaged smallest eigenvalue of the uncentered second-moment matrix.
These results imply global linear convergence of the regularized objective along the gradient flow, i.e., suboptimality decaying as $\mathcal{O}(e^{-Ct})$ for some $C>0$.
Our analysis complements the global convergence theory of entropy-regularized softmax policy gradient beyond the tabular setting of~\citet{agarwal2020optimality}, \citet{bhandari2024global},\citet{mei2020global}.
\end{abstract}

\maketitle

\section{Introduction}
\label{sec:intro}

\textbf{Overview:} The policy gradient (PG) is a fundamental concept in reinforcement learning (RL), underpinning policy search and actor-critic methods \citep{sutton1999policy}.
The map from Markov policies to the value function is in general non-convex and for parametrized softmax policies, the normalization factor induces further non-convexity in the parameters.
As a result, despite the prevalence of softmax policies in practical RL application, theoretical understanding of convergence of PG has remained limited until recently, with existing results confined to the tabular setting \citep{agarwal2020optimality,agarwal2021theory,bhandari2024global,mei2020global}.

We consider a discounted infinite horizon Markov decision model $(S,A,P,c,\gamma)$, where $S$ and $A$ are general, possibly continuous, state and action spaces, respectively.
Let $\mu$ be the fixed reference probability measure, $\rho$ the distribution of the initial state, $P\in \mathcal{P}(S|S\times A)$ the transition probability kernel, $c$ a bounded cost function, and $\gamma\in [0,1)$ the discount factor.
For a given stochastic policy $\pi \in \mathcal{P}(A|S)$, we define the entropy regularized value function $V^{\pi}_{\tau}: \mathcal P(S)\rightarrow \mathbb{R}\cup\{+\infty\}$  by
\begin{equation}\label{eq:intro:objective}
V^{\pi}_{\tau}(\rho)= \textstyle \mathbb{E}_{s_0 \sim \rho,s_{n+1}\sim P(\cdot|s_n,a_n)}^{a_n\sim \pi(\cdot|s_n)}\left[\sum_{n=0}^{\infty}\gamma^{n} \left(c(s_n,a_n)+\tau\,\log\tfrac{\mathrm d\pi}{\mathrm d\mu}(a_n|s_n)\right)\right],
\end{equation}
where $\tau > 0$ determines the intensity of the entropy regularization.
For full details on our notation, we refer to Section \ref{sec:nota and def}.

There are two implications of having $\tau>0$.
The first is that the optimal policy satisfies
\begin{equation}\label{eq:intro:optimal_policy}
\pi_\tau^*(\mathrm{d}a|s)
=\textstyle \exp\left(-\frac{Q_\tau^*(s,a)-V_\tau^*(s)}{\tau}\right)\mu(\mathrm{d}a),
\end{equation}
where $V_\tau^*$ and $Q_\tau^*$ denote the (bounded) optimal value and state-action value functions, respectively. This is an immediate consequence of the Bellman principle (see Theorem~\ref{thm:app:dpp}; see also \citealp{ziebart2010modeling,haarnoja2017reinforcement, geist2019theory}).

Second, since the entropy term is strictly convex~\citep[Section 1.4]{dupuis2011weak}, its addition in~\eqref{eq:intro:objective} is expected to improve the convergence when optimizing $V_\tau^\pi(\rho)$ over $\pi$ using policy gradient.
While the latter point may seem intuitive, the analysis is far from being straightforward even in the tabular case when $S$ and $A$ are finite and
direct parametrization 
\begin{equation}\label{eq:softmax}  
\pi_\theta(a|s)\propto \exp (\theta(s,a))
\end{equation}
with $\theta: S\times A\rightarrow \mathbb{R}$ is employed.
Two main difficulties arise: first, the map $\mathbb{R}^p \ni \theta \mapsto V^{\pi_\theta}_\tau(\rho)$ is generally non-convex, even in the bandit case (see, e.g., \citealp[Section~4.2.1]{mei2020global}).
Second, $\mathcal P(A|S) \ni \pi \mapsto V_\tau^\pi(\rho)$ is in general non-convex~\citep{agarwal2021theory, Giegrich2024} even when dynamics are linear and costs convex.
Nevertheless, convergence with good rates of policy gradient with softmax policies in the tabular setting has been shown in~\citet{mei2020global}.
In general, the suboptimality $V^{\pi_{\theta_t}}_0 - V^\ast_0$ converges sub-linearly, i.e., $\mathcal O(1/t)$, while with the additional entropy regularization, the suboptimality $V^{\pi_{\theta_t}}_\tau - V^\ast_\tau$ converges linearly, i.e., $\mathcal O(e^{-Ct})$.
The key insight of~\citet{mei2020global} is to use the non-uniform Polyak--\L ojasiewicz (P\L) inequality.

This approach becomes computationally intractable as the cardinalities of $S$ and $A$ grow large or when $S$ or $A$ are continuous. 
To overcome this one parametrizes the log densities. 
In this paper we assume the following linear function approximation: given basis functions $g: S \times A \rightarrow \mathbb{R}^p$, for all $a \in A$ and $s \in S$, let
\begin{equation*}
\pi_\theta(\mathrm{d}a|s) \propto \textstyle \exp \left(\langle\theta,g(s,a)\rangle\right)\mu(\mathrm{d}a),
\end{equation*}
where $\theta \in \mathbb{R}^p$. 

The continuous-time policy gradient is 
\begin{equation}\label
{eq:intro:continuous_GF}
\frac{\mathrm d}{\mathrm dt}\theta_t = \textstyle - \nabla_{\theta} V_\tau^{\pi_{\theta_t}}(\rho)\,,\,\,\, t \geq 0\,,\,\,\, \theta_0 \in \mathbb R^p\,\,\,\text{given}.
\end{equation}
The expression for the gradient is given by the well-known policy gradient theorem, which we restate for convenience later, in Proposition~\ref{prp:app:mdp:Gradient_of_the_Objective}.

\textbf{Outline of the argument to obtain linear convergence:} First, we wish to obtain a gradient dominance property for the objective in the form of a P{\L} inequality.
To proceed we assume $Q^\pi_\tau$-realizability up to a state-dependent baseline (see Assumption~\ref{asp:app:mdp:Q_realisability}), which provides $\boldsymbol\theta(\pi)\in \mathbb R^p$ and  $b^\pi \in B_b(S)$ satisfying $\langle \boldsymbol{{\theta}}(\pi), g \rangle + b^\pi = -\tfrac{1}{\tau} Q^{\pi}_\tau$. 
Under this assumption a simple calculation shows that 
\begin{equation}
\label{eq:pl_calc0_intro}
\nabla_\theta V^{\pi_\theta}_\tau(\rho)
=
\textstyle \frac{\tau}{1-\gamma}
\left(
\int_S
G^{\pi_\theta}(s)
d^{\pi_\theta}_\rho(\mathrm{d}s)
\right)
\bigl(
\theta-\boldsymbol{\theta}(\pi_\theta)
\bigr),
\end{equation}
where the Fisher information matrix (FIM) of $\pi_\theta$ for fixed $s$ is defined as
\begin{equation}\label
{eq:intro:FIM}
G^{\pi_\theta}(s) = \textstyle\int_A \nabla_\theta \log \tfrac{\mathrm d \pi_{\theta}}{\mathrm d \mu}(a|s)\big(\nabla_\theta \log \tfrac{\mathrm d \pi_{\theta}}{\mathrm d \mu}(a|s)\big)^{\top}\pi_{\theta}(\mathrm{d}a|s)\,.
\end{equation}
Left multiplying~\eqref{eq:pl_calc0_intro} by $(\theta-\boldsymbol{\theta}(\pi_\theta))^\top$, using that the FIM is positive semi-definite, using the Cauchy--Schwarz inequality and finally dividing by $\|\theta-\boldsymbol{\theta}(\pi_\theta)\|_2 >0$ (otherwise the claim is trivial), we have
\begin{equation}
\label{eq:pl_calc1_intro}
\|\nabla_\theta V^{\pi_\theta}_\tau(\rho)\|_2 \geq \tfrac{\tau}{1-\gamma}  \textstyle\int_S \lambda_{\text{min}}(G^{\pi_\theta}(s)) d^{\pi_\theta}_\rho(\mathrm{d}s)\|\theta - \boldsymbol{{\theta}}(\pi_\theta)\|_2\,.
\end{equation}

Separately, using the KL-sandwich inequality (see Lemma~\ref{lem:sandwich-inequality}) and local KL--logit inequality (Lemma~\ref{lem:app:kl-logit_ineq_affine}; see also \citealp[Lemma~27]{mei2020global} for the tabular case) we get that for some $C_{\tau,\gamma,\theta}>0$
\[
0 \leq V_\tau^{\pi_{\theta}}(\rho)-V_\tau^{\pi^*}(\rho) \leq C_{\tau,\gamma,\theta} \|\theta - \boldsymbol{{\theta}}(\pi_\theta)\|_2^2\,.
\]
This, together with~\eqref{eq:pl_calc1_intro} shows that there is $ C: \mathbb{R}^p \to (0, \infty]$ such that for all
$\theta \in \mathbb{R}^p$
\begin{equation}
\label{eq:pl_ineq_intro}
0 \leq V_\tau^{\pi_{\theta}}(\rho)-V_\tau^{\pi^*}(\rho) \leq  C(\theta)\|\nabla_\theta V^{\pi_\theta}_\tau(\rho)\|_2^2 \,.
\end{equation}
This is stated fully as Theorem~\ref{thm:main: non-uniform_PL_ineq} below with a complete proof given later.
Note that the strength of the gradient dominance is parameter dependent, i.e., we only have non-uniform P{\L} inequality.

Nevertheless with $\{\theta_t\}_{t\geq 0}$ given by the gradient flow~\eqref{eq:intro:continuous_GF} we have
\[
\tfrac{d}{d t}\Big[V_\tau^{\pi_{\theta_t}}(\rho)-V_\tau^{\pi^*}(\rho)\Big]=\textstyle-\left\|\nabla V_\tau^{\pi_{\theta_t}}(\rho)\right\|^2 \leq- C^{-1}\left(\theta_t\right)\Big[V_\tau^{\pi_{\theta_t}}(\rho)-V_\tau^{\pi^*}(\rho)\Big]\,.
\]
Hence, from Grönwall's inequality we immediately get that
\[
0 \leq V_\tau^{\pi_{\theta_t}}(\rho)-V_\tau^{\pi^*}(\rho) \leq \textstyle \big[V_\tau^{\pi_{\theta_0}}(\rho)-V_\tau^{\pi^*}(\rho)\big] \exp \left(-\textstyle\int_0^t  C^{-1}(\theta_s) \mathrm d s\right) \,.
\]
If we can show that the non-uniform term can be lower bounded along the flow, i.e., if we can show that $\inf _t C^{-1}\left(\theta_t\right) > 0$ then this is the required linear convergence rate.

We notice that if $\{\theta_t\}_{t\geq 0}$ given by the gradient flow~\eqref{eq:intro:continuous_GF} produces policies with uniformly bounded log densities, i.e., if $\sup_{t} |\log \frac{\mathrm d \pi_{\theta_t}}{\mathrm d\mu}|_{B_b(S\times A)} < \infty$, it yields a positive lower bound for the relevant $\rho$-averaged spectral coefficient.
In the full-affine-span case this coefficient involves the statewise FIM; in the simplex case it involves the uncentered feature second-moment matrices.
It therefore suffices to prove that the parameter trajectory remains in a compact set.
We will use Lyapunov function techniques to that end.

With that in mind, we note that a simple calculation using the chain rule yields value improvement along the gradient flow~\eqref{eq:intro:continuous_GF}, namely
$
\frac{\mathrm d}{\mathrm dt}V^{\pi_{\theta_t}}_\tau(\rho) = - \|\nabla_\theta V^{\pi_{\theta_t}}_\tau(\rho)\|_2^2 \leq 0\,.
$
The value function $\theta \mapsto V^{\pi_{\theta}}_\tau(\rho)$ can then be used as a Lyapunov function as long as it is radially unbounded.
Since the cost itself is bounded the radial unboundedness can only come as a consequence of the KL term. 
Example~\ref{ex:hat_func_short}, which uses the hat functions as a basis (P1 finite elements) shows that not every reasonable feature basis leads to a radially unbounded KL term.
However, we identify conditions on the feature basis which ensure radial unboundedness, i.e., that $\operatorname{KL}(\pi_\theta|\mu) \to \infty$ whenever $\|\theta\|_2 \to \infty$.
In Example~\ref{ex:trig_func_mdp} we show that the trigonometric basis satisfies our Assumption~\ref{asp:app:mdp:max_set_affine}.

Additionally, in the case of simplex features (see Assumption~\ref{asp:app:mdp:max_set_simplex}), we can extend the P{\L} inequality by replacing the FIM with a version built using uncentered features which allows us to carry out a similar convergence argument but capturing further types of basis functions, e.g., the Bernstein polynomials and softmax kernel features, see Examples~\ref{ex:Bernstein_Poly_MDP} and ~\ref{ex:ntk}.

\textbf{Key differences with the tabular case:}

In the tabular case when $\tau > 0$~\citet{mei2020global} prove P{\L} inequality~\eqref{eq:pl_ineq_intro} with $C(\theta) = C (\min_{s,a} \pi_{\theta}(a|s))^{-2}$ with $C$ independent of $\theta$.  
Their methods then can be used to show that along the flow~\eqref{eq:intro:continuous_GF} it holds that 
$\sum_a \tfrac{\partial V^{\pi_{\theta_t}}_\tau(\rho)}{\partial \theta_t(s, a)} =0$ for all $t$. 
From this they then derive $\inf_t \min_{s,a} \pi_{\theta_t}(a|s) > 0$ which, as we have seen, leads to linear convergence.

In our setting, with general log-linear policies, their way to formulate the property that the gradient of the value function summed up over $a$ is $0$ along the flow cannot be employed. 
Even though with simplex features we have $\sum_{i} \nabla_{\theta_i} V^{\pi_{\theta_t}}_\tau(\rho) = 0$ for all $t$, weights $\theta$ impact the entire conditional log-density and therefore this zero-sum identity alone does not provide a lower bound for the action density; the statewise separation used in the tabular argument is unavailable.
Thus, a fundamentally different approach has to be used.
On the other hand, in the tabular setting, one cannot expect radial unboundedness of the KL divergence as this term will be finite for any $\mu$ s.t. $\mu(a) > 0$ for all $a$. 
Thus the use of KL divergence to supply radial unboundedness of the objective as a Lyapunov function is novel in our setting and our results complement those of~\citet{mei2020global}.

\textbf{Literature review:} 
There is a tremendous amount of research literature on convergence of RL methods, which underscores its importance. 
Here we focus on the subset of the RL literature that we think is most related to our work.

Entropy-regularized RL has demonstrated both good algorithmic performance and desirable theoretical properties~\citep{haarnoja2017reinforcement,haarnoja2018soft, geist2019theory,vieillard2020leverage, neu2017unified, fox2015taming,ziebart2010modeling}. 
It has been shown that the softmax policies are optimal in the entropy regularized setting.

For policy gradient, the work of \citet{agarwal2020optimality} initiated a global analysis in discounted MDPs in tabular setting and compatible function approximation parameterizations, making explicit the roles of distribution mismatch, approximation error, and statistical error; in the tabular softmax case this gives an asymptotic global
convergence for vanilla policy gradient. 
The work of \citet{mei2020global} subsequently gave a sharper analysis of tabular softmax policy gradient with direct parametrization as discussed above. 
These upper bounds are complemented by the lower bound of~\citet{li2021softmax}, who show that for standard softmax policy gradient, the constants in the convergence rate suffer from exponential dependence on state-space and horizon sizes.
Moving beyond the realizability assumption, ~\citet{lin2025rethinking} derive some  ordering conditions in the bandit case (singleton state space) which guarantee the convergence of softmax policy gradient using linear function approximation.
However as of now it is unclear how this applies to MDPs and also fundamentally depends on the finite cardinality of the action space.

Although MDPs with continuous state and action spaces are widely used in practical applications \citep{doya2000reinforcement, van2012reinforcement, manna2022learning}, the convergence analysis of policy gradient methods in this setting remains less developed compared to its discrete counterparts.
Most existing works have focused on discrete-time linear quadratic regulator (LQR) problems with linearly parametrizeed policies.
The linear-quadratic structure leads to the P{\L} inequality~\citep{polyak1963gradient, lojasiewicz1963topological, kurdyka1998gradients} with a uniform constant and we have seen above how this then gives linear convergence~\citep{fazel2018global, bu2019lqr, hu2023toward}.
These results have been extended to continuous-time LQR systems~\citep{sontag2022remarks,Giegrich2024}. 

One may choose to use the inverse of the integrated FIM~\eqref{eq:intro:FIM} as a pre-conditioner in the gradient flow~\eqref{eq:intro:continuous_GF} which gives rise to the natural policy gradient (NPG).  
For NPG with log-linear policies,~\citet{cayci2024convergence} prove linear convergence of the NPG in the entropy regularized setting.
Moreover for log-linear policies the NPG is in fact identical to mirror descent (MD) (in the sense of producing the same policy updates) under exact (compatible) realizability and MD is known to converge linearly for entropy regularized MDPs~\citep{lan2023policy, ju2022policy, kerimkulov2025fisher}.
It is perhaps interesting to note that unlike policy gradient, MD / NPG automatically ensure that log densities remain bounded along the optimization and thus the FIM is invertible almost regardless of the feature basis (one cannot have linearly dependent features).
This is a consequence of the policy improvement of NPG / MD under exact evaluations. 

A linear convergence rate is proved by~\citet{liu2023polyak} under an additional P{\L} inequality for the continuous-time Fisher--Rao flow on the space of measures.
Continuous-time Fisher--Rao flows in the entropy regularized MDPs have been studied by~\citet{kerimkulov2025fisher}, where the linear convergence to the optimal policy has been established and the insights into the natural policy gradient flow with linear function approximation were given. 

\textbf{Main contributions of this paper:}
We prove that under suitable conditions on the features and under $Q^\pi_\tau$-realizability up to a state-dependent baseline, the policy gradient for the KL-regularized MDP on general state and action spaces with exact evaluations converges linearly. 
In particular we:
\begin{enumerate}
\item Prove a non-uniform P{\L} inequality in this setting, see Theorem~\ref{thm:main: non-uniform_PL_ineq}.
\item Establish conditions on the feature basis under which the KL divergence is radially unbounded and hence Lyapunov function techniques can be used to obtain a bound on the non-uniform constant, see Theorems~\ref{thm:main:mdp:Radial_unboundedness_of_relative_entropy} and~\ref{thm:main:mdp:Radial_unboundedness_of_relative_entropy_in_subspace_simplex}.
This is of independent interest to anyone analyzing algorithms that employ log-linear densities and KL regularization as it opens up Lyapunov function techniques to them.
\item Obtain the desired linear convergence, see Theorems~\ref{thm:main:lin_cov_affine} and~\ref{thm:main:lin_cov_simplex}.
\item Provide examples of basis functions which satisfy the conditions, see Examples~\ref{ex:trig_func_mdp},~\ref{ex:Bernstein_Poly_MDP} and~\ref{ex:ntk}. 
\end{enumerate}

\section{Formulation and the statement of the main results}

Let $S$ and $A$ be Polish spaces with $A$ compact.
Let $P\in \mathcal{P}(S|S\times A)$.
Let  $\gamma\in [0,1)$ be the discount factor.
Let $c\in B_b(S\times A)$ be the cost function and we will assume, without loss of generality, that  $|c|_{B_b(S \times A)}\leq 1$. 
Let $\tau > 0$. 
Let $\mu \in \mathcal P(A)$ have full support on $A$.
The seven-tuple $(S,A,P,c,\gamma,\tau, \mu)$ determines a $\gamma$-discounted infinite horizon $\tau$-entropy regularized Markov decision process model.

For a given randomized Markov policy $\pi\in \mathcal P(A|S)$, we define the $\tau$-entropy regularized value function $V^{\pi}_{\tau}:\mathcal P(S)\rightarrow \mathbb{R}\cup\{+\infty\}$  by
\[
V^{\pi}_{\tau}(\rho)= \textstyle \mathbb{E}_{s_0 \sim \rho,s_{n+1}\sim P(\cdot|s_n,a_n)}^{a_n\sim \pi(\cdot|s_n)}\left[\sum_{n=0}^{\infty}\gamma^{n} \left(c(s_n,a_n)+\tau\,\log\tfrac{\mathrm d\pi}{\mathrm d\mu}(a_n|s_n)\right)\right].
\]
Note $V_\tau^\pi(s) = V_\tau^\pi(\delta_s)$.
The aim is to minimize $\mathcal P(A|S)\ni \pi \mapsto V^{\pi}_{\tau}(\rho)$.
By the Bellman principle, recalled in
Theorem~\ref{thm:app:dpp} in Section~\ref{sec:basic} (see also \citealp{ziebart2010modeling,haarnoja2017reinforcement, geist2019theory}), the optimal policy is of the form~\eqref{eq:intro:optimal_policy}, and the optimal value and state-action value functions are bounded.
Hence, without loss of generality we may restrict our minimization problem to softmax policies from the class
\begin{align}\label
{def:Pi_class}
\Pi_{\mu} &= \textstyle \left\{\tfrac{\exp(f(s,a))\mu(\mathrm{d}a)}{\int_{A} \exp(f(s,a^\prime))\mu(\mathrm{d}a^\prime)}\in \mathcal{P}_{\mu}(A|S) | f\in B_b(S\times A)\right\}  \,.
\end{align}

For a given policy $\pi\in\Pi_{\mu}$, we define the regularized state-action value function $Q^{\pi}_{\tau}\in B_b(S\times A)$ by
\begin{equation}\label{def:intro:Qpi}
Q^{\pi}_{\tau}(s,a)= \textstyle c(s,a)+\gamma\int_{S}V^{\pi}_{\tau}(s')P(\mathrm{d}s'|s,a)\,.
\end{equation}

The occupancy kernel $d^{\pi}\in \mathcal{P}(S|S)$ is defined by $d^{\pi}(\mathrm{d}s'|s)=(1-\gamma)\sum_{t=0}^{\infty}\gamma^t P^t_{\pi}(\mathrm{d}s'|s)\,, d^{\pi}_{\rho}(\mathrm{d}s)=\int_{S}d^{\pi}(\mathrm{d}s|s')\rho(\mathrm{d}s')$
where $P_\pi(\mathrm ds'|s) = \textstyle \int_A P(\mathrm ds'|s,a)\pi(\mathrm da |s)$ and $P^0_{\pi}(\mathrm{d}s'|s):=\delta_s(\mathrm{d}s')$, $P^t_{\pi}$ is understood as a product of kernels, and convergence is understood in $b\mathcal{K}(S|S)$. 
The on-policy Bellman equation holds (see, e.g., \citealp[Lemma~B.2]{kerimkulov2025fisher}):
\begin{equation}
\label{eq:intro:policy_Bellman}
V^{\pi}_{\tau}(s) = \textstyle \int_{A}\left(c(s,a)+\tau\log\frac{\mathrm d\pi}{\mathrm d\mu}(a|s) +\gamma\int_{S}V^{\pi}_{\tau}(s')P(\mathrm{d}s'|s,a)\right)\pi(\mathrm{d}a|s) \,\,\, \forall s\in S\,.
\end{equation}
Moreover one can see that this has the stochastic representation 
\begin{equation}
\label{eq:intro:value_function_occ}
V^{\pi}_{\tau}(s) = \textstyle \frac{1}{1-\gamma}\int_{S}\int_{A}\left(c(s',a')+\tau \log \frac{\mathrm d\pi}{\mathrm d\mu}(a'|s')\right)\pi(\mathrm{d}a'|s')d^{\pi}(\mathrm{d}s'|s)\,\,\, \forall s\in S\,.
\end{equation}
For a given fixed initial distribution $\rho\in\mathcal{P}(S)$, we have
\[
V^{\pi}_{\tau}(\rho)=\textstyle\int_{S}V^{\pi}_{\tau}(s)\rho(\mathrm{d}s).
\]

When $S$ or $A$ are not of finite cardinality the minimization over the class~\eqref{def:Pi_class} is intractable. 
Thus, instead of looking for the optimal policy in~\eqref{def:Pi_class}, we parametrize the softmax policies using linear function approximation.
To that end let $g: S \times A \rightarrow \mathbb{R}^p$ be our basis functions (or features). 
We take $g$ to be measurable and assume $\sup_{s,a}\|g(s,a)\|_2 \leq 1$.
Now given parameters $\theta \in \mathbb R^p$ consider parametrized policies of the form
\[
\pi_\theta(\mathrm{d}a|\cdot)=\textstyle \tfrac{\exp \left(\langle\theta,g(\cdot,a)\rangle\right)\mu(\mathrm{d}a)}{\int_{A} \exp \left(\langle\theta,g(\cdot,a^\prime)\rangle\right)\mu(\mathrm{d}a^\prime)}.
\]
Thus we wish to solve the minimization problem
\begin{equation}\label{eq:intro:entropy regularized mdp0}
\min_{\theta \in \mathbb{R}^p} V_\tau^{\pi_\theta}(\rho)\,.
\end{equation}

We will work under the assumption of $Q^\pi_\tau$-realizability up to a state-dependent baseline.

\begin{restatable}[$Q^\pi_\tau$-realizability up to a state-dependent baseline]{assumption}{aspQpiRealizability}
\label{asp:app:mdp:Q_realisability}
For any $\pi \in \Pi_\mu$, there exist $b^\pi \in B_b(S)$ and $\boldsymbol{\theta}(\pi) \in \mathbb R^p$ such that $\langle \boldsymbol{\theta}(\pi), g(s,a) \rangle + b^\pi(s) = -\frac{1}{\tau} Q^{\pi}_\tau(s,a) $ for all $s \in S, a \in A$.
\end{restatable}

Our assumption is closely related to compatible function approximation in NPG methods~\citep{agarwal2020optimality,liu2020improved,yuan2022linear,cayci2024convergence,kerimkulov2025fisher}.
Note that for $\pi \in \Pi_\mu$ the pair $(\boldsymbol{\theta}(\pi), b^\pi)$ given by Assumption~\ref{asp:app:mdp:Q_realisability} is not unique. 
Indeed, for any nonzero $\theta \in \mathbb R^p$ and $\bar k\in B_b(S)$ which satisfy $\bar k(s) = \langle \theta, g(s,a) \rangle$ for every $(s,a)$, we have with any $t\in\mathbb R$ that $\bigl(\boldsymbol\theta(\pi)+t\theta,\ b^\pi-t \bar k\bigr)$.

An example when Assumption~\ref{asp:app:mdp:Q_realisability} holds besides the tabular case are linear MDPs (see, e.g.,~\citealp{yang2019sample, zanette2020frequentist, li2021sample}). 
An MDP is linear if there exists $w\in \mathbb{R}^p$ and a sequence $\{\psi_i\}_{i=1}^p$ with $\psi_i \in \mathcal{M}(S)$ such that for all $(s,a)\in S \times A$, $c(s,a)=\langle w,g(s,a)\rangle , P( \mathrm{d}s'| s,a)=\sum_{i=1}^p g_i(s,a)\psi_i( \mathrm{d}s')$.
In this case, given $\pi \in \Pi_\mu$ we can take $\boldsymbol{\theta}(\pi)_i = -\tfrac1\tau\big(w_i + \gamma\int_{S} V^{\pi}_\tau(s')\psi_i(\mathrm{d}s')\big)$ 
so that 
\begin{equation}
\label{eq:linear_MDP}
\langle \boldsymbol{\theta}(\pi), g(s,a) \rangle = \textstyle -\tfrac1\tau\langle w, g(s,a)\rangle -\tfrac\gamma\tau \textstyle \int_S V^{\pi}_\tau(s') \sum_{i=1}^p g_i(s,a)\psi_i(\mathrm{d}s') = -\tfrac1\tau Q^\pi_\tau(s,a)\,.
\end{equation}

Here Assumption~\ref{asp:app:mdp:Q_realisability} is satisfied with $b^\pi \equiv 0.$
Although Assumption~\ref{asp:app:mdp:Q_realisability} is weaker than assuming that the MDP is linear, we know it will not always hold in practice.
However, ~\citet{ma2024skill,ren2025spectral} demonstrate a method building linear MDPs from data that achieves empirical success.

Recall that the Bellman principle gives the bounded optimal state-action value $Q^*_\tau=Q^{\pi^*}_\tau$ and the unique Bellman-optimal policy $\pi^*\in\Pi_\mu$ (optimal at every state). 

Applying Assumption~\ref{asp:app:mdp:Q_realisability} to $\pi^*$, choose a representative $\theta^*\in\mathbb R^p$ and a baseline $b^*\in B_b(S)$ such that
$
\langle\theta^*,g(s,a)\rangle
+b^*(s)= \textstyle-\frac{1}{\tau}Q^*_\tau(s,a).
$
The baseline cancels in the softmax normalization, and therefore
\[
\pi^*(\mathrm{d}a| s)
\propto \textstyle \exp\!\left(-\frac{1}{\tau}Q^*_\tau(s,a)\right)\mu(\mathrm{d}a)
\propto \exp\!\left(\langle\theta^*,g(s,a)\rangle\right)\mu(\mathrm{d}a).
\]
Thus the parametrizeed minimization problem is solvable:
$\min_{\theta\in\mathbb R^p}V_\tau^{\pi_\theta}(\rho)=V^*_\tau(\rho)$.
The full-space representative $\theta^*$ need not be unique because of policy-invariant directions.

The remainder of the paper is devoted to the argument that under suitable assumptions on the features the gradient flow~\eqref{eq:intro:continuous_GF} converges linearly in the sense that $0 \leq V_\tau^{\pi_{\theta_t}}(\rho) - V_\tau^{\pi^\ast}(\rho) \leq \mathcal O(e^{-Ct})$.
Recall that we will do this by first obtaining a non-uniform P{\L} inequality and then demonstrating that along the gradient flow we in fact can bound the constant uniformly, using radial unboundedness of $\theta \mapsto \operatorname{KL}(\pi_\theta|\mu)$ which holds for suitable feature maps.

\begin{restatable}[Non-uniform P{\L} inequality]{theorem}{PL}\label
{thm:main: non-uniform_PL_ineq}
Let Assumption~\ref{asp:app:mdp:Q_realisability} hold.
Let $R > 0$.
Let
\[
\textstyle
\Theta_R = \{\theta: \pi_\theta \in \Pi_\mu \ \text{and}\  \vert \log \tfrac{\mathrm d \pi_\theta}{\mathrm d\mu}\vert_{B_b(S\times A)} \leq R \}\,.
\]
Then for any $\theta \in \Theta_R$ there exists $C_R(\theta)>0$ such that
\begin{equation}
\textstyle
0\leq V_\tau^{\pi_\theta}(\rho) - V_\tau^{\pi^*}(\rho) \leq C_R(\theta)\left\|\nabla V_\tau^{\pi_\theta}(\rho)\right\|^2_2.
\end{equation}
\end{restatable}
Depending on the assumptions on the feature basis, this will be proved in Section \ref{sec:pf_non_uniform_PL},
where the exact form of $C_R(\theta)$ will be given.
We now need to work with specific assumptions on the basis functions.
Below, we will define ``full affine span features'' and ``simplex features'' and discuss how they allow us to prove the linear convergence.

\subsection*{Full affine span features in \texorpdfstring{$\mathbb R^p$}{R^p}}
We will say that the feature basis $g:S\times A\to \mathbb{R}^p$ has full affine dimension if $\operatorname{span}\{{g(s,a)} - {g(s,a')}:a,a'\in A \}=\mathbb{R}^p$, for every $s\in S$.
Recall that $u\in\mathbb S^{p-1}$ if $u\in\mathbb R^{p}$ and $\|u\|_2 = 1$.
\begin{restatable}[]{assumption}{aspAffineMaxSet}\label{asp:app:mdp:max_set_affine} 
For each fixed
$s\in S$, assume that $a\mapsto g(s,a)\in\mathbb R^p$ is continuous
and that, 
for every $u\in\mathbb S^{p-1}$,
\[
\mu \Big(\arg\max_{a\in A} u^\top {g(s,a)}\Big)=0 .
\]
\end{restatable}
We see that Assumption~\ref{asp:app:mdp:max_set_affine} implies \textup{1.} in  Lemma~\ref{rem:app:mdp:affine_span}.
Thus ${g(s, A)}$ has full affine dimension $p$.
We will work under Assumption~\ref{asp:app:mdp:max_set_affine} as it will be needed later to get the radial unboundedness in $\mathbb R^{p}$.

Clearly, for a linear MDP,  $1=P(S|s,a)=\sum_{i=1}^{p} g_i(s,a)\psi_i(S)$.
Taking $u_i = \tfrac{\psi_i(S)}{\|\psi(S)\|_2}$ we see  Assumption~\ref{asp:app:mdp:max_set_affine} is violated.
Nevertheless, in Example~\ref{ex:trig_func_mdp} 
we provide features and an MDP close to linear (but not linear in the chosen $g$) where Assumptions~\ref{asp:app:mdp:Q_realisability} and~\ref{asp:app:mdp:max_set_affine}  are jointly satisfied.
More generally, if one were to relax Assumption~\ref{asp:app:mdp:Q_realisability} to allow for an approximation error term then Assumption~\ref{asp:app:mdp:max_set_affine} becomes much less restrictive.
However one would now need to track how the approximation error impacts the convergence.

\begin{restatable}[Radial unboundedness of KL divergence in $\mathbb R^{p}$]{theorem}{KLUBAFFINE}\label{thm:main:mdp:Radial_unboundedness_of_relative_entropy}
Let Assumption \ref{asp:app:mdp:max_set_affine} hold.
Fix $s\in S$. 
Then
\[
\lim_{\|{\theta}\|_2\to\infty} \operatorname{KL}\bigl(\pi_{{\theta}}(\cdot| s)\,|\,\mu\bigr)=\infty.
\]
\end{restatable}

The proof of Theorem \ref{thm:main:mdp:Radial_unboundedness_of_relative_entropy} can be found in Section \ref{sec:pf_Radial_unboundedness_of_relative_entropy}.
Theorem~\ref{thm:main:mdp:Radial_unboundedness_of_relative_entropy}
combined with classical arguments for constructing ODE solutions with Lyapunov functions will give existence of the solution to the gradient flow~\eqref{eq:intro:continuous_GF}.

\begin{restatable}[Existence of solution to the gradient flow with full affine span features]{lemma}{ExistSolAffine}
\label{lem:main:mdp:exist_sol_gf_affine}
Let Assumption~\ref{asp:app:mdp:max_set_affine} hold.
Then there exists a unique global solution $\{\theta_t\}_{t \geq 0}$ to~\eqref{eq:intro:continuous_GF}. 
Moreover $\sup_{t \geq 0} \big|\log \tfrac{\mathrm d \pi_{\theta_t}}{\mathrm d\mu}\big|_{B_b(S \times A)} < \infty$.
\end{restatable}

The proof of Lemma \ref{lem:main:mdp:exist_sol_gf_affine} can be found in Section \ref{sec:pf_exist_sol_gf_affine}.

\begin{restatable}{lemma}{aspInitAffine}
\label{lemma:app:mdp:initial_singularity_mdp}
Let Assumption~\ref{asp:app:mdp:max_set_affine} hold.
Then for all $s \in S$, the FIM $G^{\pi_{{\theta}_0}}(s)$ given by~\eqref{eq:intro:FIM} is positive definite for any ${\theta}_0 \in \mathbb R^{p}$.
\end{restatable}

The proof of Lemma \ref{lemma:app:mdp:initial_singularity_mdp} can be found in Section  \ref{sec:pf_lin_cov_affine}.

\begin{restatable}[Linear convergence with full affine span features]{theorem}{LinCovAffine}\label{thm:main:lin_cov_affine}
Let Assumptions~\ref{asp:app:mdp:Q_realisability} and \ref{asp:app:mdp:max_set_affine} hold.
Let $\{\theta_t\}_{t \geq 0}$ be the solution to gradient flow~\eqref{eq:intro:continuous_GF} and $R := \textstyle \sup_{t\ge0}\big|\log\tfrac{\mathrm d\pi_{\theta_t}}{\mathrm d\mu}\big|_{B_b(S\times A)}$ given by Lemma~\ref{lem:main:mdp:exist_sol_gf_affine}.
Then there exists $C_{\theta_0} > 0$ such that 
for $C_R(\theta_t)$ from Theorem~\ref{thm:main: non-uniform_PL_ineq} we have $\sup_{t \geq 0} C_R(\theta_t) \leq C_{\theta_0}$ and
\[
\textstyle
0 \leq V_\tau^{\pi_{\theta_t}}(\rho) - V_\tau^{\pi^*}(\rho) \leq e^{-t C_{\theta_0}^{-1}} 
(V_\tau^{\pi_{\theta_0}}(\rho) - V_\tau^{\pi^*}(\rho)).
\]
\end{restatable}
The exact form of $C_{\theta_0}$ can be found in Section \ref{sec:pf_lin_cov_affine}.

To conclude this part, we introduce an example of full affine span features that satisfy Assumption \ref{asp:app:mdp:max_set_affine}. 
Details can be found in Section~\ref{sec:trig_mdp}.

\begin{example}[Trigonometric features]\label{ex:trig_func_mdp}
Let $S=[-\pi,\pi]^{d_1}$ and $A=[-\pi,\pi]^{d_2}$, where $d_1,d_2\geq1$, and let $\mu(\mathrm da)=\frac{1}{(2\pi)^{d_2}}\,\mathrm da, \nu(\mathrm ds')=\frac{1}{(2\pi)^{d_1}}\,\mathrm ds',$ where $\mu$ is the reference probability measure on the action space and $\nu$ is the uniform probability measure on the state space.
Fix $\gamma\in[0,1)$ and $\tau>0$.
Fix an integer $p\geq2$, and put $m:=\lceil p/2\rceil$ and $n:=\lfloor p/2\rfloor$.
Choose action-frequency vectors $k_1,\ldots,k_m\in\mathbb Z^{d_2}\setminus\{\mathbf0\}$ such that $k_i\neq\pm k_j$ whenever $i\neq j$, and arbitrary state-frequency vectors $\ell_1,\ldots,\ell_m\in\mathbb R^{d_1}$.
Define $g:S\times A\to\mathbb R^p$ by $g_{2j-1}(s,a)=\frac{1}{\sqrt m}\cos(k_j^\top a+\ell_j^\top s),j=1,\ldots,m$ and $g_{2j}(s,a) =\frac{1}{\sqrt m}\sin(k_j^\top a+\ell_j^\top s), j=1,\ldots,n.$
For example, when $d_1=d_2=2$ and $p=3$, one may take $g(s,a)=\frac{1}{\sqrt2}\bigl( \cos(a_1+a_2+s_1),\sin(a_1+a_2+s_1),\cos(a_1-a_2+s_2)\bigr)^\top.$

Since the nonzero integer action frequencies are distinct modulo sign, Fourier orthogonality on $A$ gives, for every fixed $s\in S$, $\int_Ag(s,a)\,\mu(\mathrm da)=\mathbf0, \int_Ag(s,a)g(s,a)^\top\,\mu(\mathrm da)=\frac{1}{2m}I_p.$
Thus, for every $\theta\neq\mathbf0$, the function $a\mapsto\langle\theta,g(s,a)\rangle$ has zero mean and second moment $\|\theta\|_2^2/(2m)>0$, and hence is nonconstant.
The difference between this function and its maximum on $A$ is therefore a nonconstant real-analytic function on $\mathbb R^{d_2}$.
By the property of real-analytic functions, its zero set has $d_2$-dimensional Lebesgue measure zero, so $\mu\!\left(\operatorname*{arg\,max}_{a\in A}\langle\theta,g(s,a)\rangle\right)=0,$ satisfying Assumption~\ref{asp:app:mdp:max_set_affine}.

Independently choose next-state frequency vectors $\ell'_1,\ldots,\ell'_m\in\mathbb Z^{d_1}\setminus\{\mathbf0\}$, and define $h:S\to\mathbb R^p$ by $h_{2j-1}(s')=\frac{1}{\sqrt m} \cos\bigl((\ell'_j)^\top s'\bigr), j=1,\ldots,m$ and $h_{2j}(s')=\frac{1}{\sqrt m}\sin\bigl((\ell'_j)^\top s'\bigr),j=1,\ldots,n.$
The nonzero integer frequencies ensure that $\int_Sh(s')\,\nu(\mathrm ds')=\mathbf0$.

Fix $\varepsilon\in(0,1)$ and define the transition kernel $P(\mathrm ds'| s,a) =\left[1+\varepsilon\langle g(s,a),h(s')\rangle\right]\nu(\mathrm ds').$
We also obtain $P(S| s,a)=\textstyle\int_S\left[1+\varepsilon\langle g(s,a),h(s')\rangle\right]\nu(\mathrm ds')=1+\varepsilon\left\langle g(s,a),\int_Sh(s')\,\nu(\mathrm ds')\right\rangle=1.$
Define the finite signed measures $\psi_i(\mathrm ds'):=\varepsilon h_i(s')\,\nu(\mathrm ds')$, $i=1,\ldots,p$.
Then $\psi_i(S)=0$ and $P(\mathrm ds'| s,a)=\nu(\mathrm ds')+\sum_{i=1}^p g_i(s,a)\psi_i(\mathrm ds').$
Thus the transition kernel is linear in the features up to the action-independent baseline measure $\nu$.
Choose $w\in\mathbb R^p$ with $\|w\|_2\leq1/4$, and define $c(s,a):=\frac12+\langle w,g(s,a)\rangle.$
For $\pi\in\Pi_\mu$, recall that $Q_\tau^\pi(s,a)=c(s,a)+\gamma\int_SV_\tau^\pi(s')P(\mathrm ds'| s,a).$

Set $\boldsymbol\theta(\pi)_i:=-\frac1\tau\left(w_i+\gamma\int_SV_\tau^\pi(s')\psi_i(\mathrm ds')\right), i=1,\ldots,p.$
These integrals are finite because $V_\tau^\pi$ and the features are bounded and $\psi_i=\varepsilon h_i\nu$.
Then we have 
\[
\langle\boldsymbol\theta(\pi),g(s,a)\rangle=-\frac1\tau Q_\tau^\pi(s,a)+\frac{1}{2\tau}+\frac{\gamma}{\tau(2\pi)^{d_1}}\int_{[-\pi,\pi]^{d_1}}V_\tau^\pi(s')\,\mathrm ds'.
\]
In this case, the state-dependent baseline can be chosen as $b^\pi(s)=-\frac1\tau(\frac12+\gamma\int_S V_\tau^\pi(s')\,\nu(\mathrm ds')),$ which is constant in $s$ in this particular example.
\end{example}

\subsection*{Simplex features}

The structure of linear MDPs motivates us to further consider simplex features.
Assumption~\ref{asp:app:mdp:max_set_simplex} explains what we mean by simplex features and adds an additional property that allows us to get unboundedness of $\theta \mapsto \operatorname{KL}(\pi_\theta|\mu)$ in a subspace orthogonal to the $\mathbf 1$-vector.
For any $v\in \mathbb R^p$ we write $v_\perp := v - p^{-1}\langle v, \mathbf 1\rangle \mathbf 1$.

\begin{restatable}{assumption}{aspSimplexMaxSet}\label{asp:app:mdp:max_set_simplex}
The feature basis $g:S\times A \to \mathbb R^p$ takes values in the probability simplex $\Delta_{p-1}$ for every $s$ and $a$, i.e., $\sum_{i=1}^p g_i(s,a) = 1 $ and for all $s \in S,a \in A$, $g(s,a) \geq 0$. 
Moreover, for each fixed
$s\in S$, assume that $a\mapsto g(s,a)\in\mathbb R^p$ is continuous
and that, for every $u\in\mathbb R^p$ such that $u\neq 0$ and $u \perp \mathbf{1}$,
\[
\mu \big(\arg\max_{a\in A} u^\top g(s,a)\big)=0 .
\]
\end{restatable}
Note the assumption on the measure of the maximizer set of $u^\top g(s,a)$ is a little different between Assumption~\ref{asp:app:mdp:max_set_simplex}
and Assumption~\ref{asp:app:mdp:max_set_affine} in that we consider different vectors $u$ in each.

\begin{restatable}[Radial unboundedness of KL divergence in a subspace orthogonal to $\mathbf 1$]{theorem}{KLUBSIMPLEX}\label{thm:main:mdp:Radial_unboundedness_of_relative_entropy_in_subspace_simplex}
Let Assumption~\ref{asp:app:mdp:max_set_simplex} hold.
Fix $s\in S$.
Then 
\[
\lim_{\theta\in \mathbb R^p:\|\theta_\perp\|_2\to\infty}
\mathrm{KL}\bigl(\pi_\theta(\cdot| s)\,|\,\mu\bigr)=\infty.
\]
\end{restatable}

The proof of Theorem \ref{thm:main:mdp:Radial_unboundedness_of_relative_entropy_in_subspace_simplex} can be found in Section \ref{sec:pf_Radial_unboundedness_of_relative_entropy_in_subspace_simplex}.
Theorem \ref{thm:main:mdp:Radial_unboundedness_of_relative_entropy_in_subspace_simplex} combined with classical arguments for constructing ODE solutions with Lyapunov functions will give existence of the solution to the gradient flow~\eqref{eq:intro:continuous_GF}.

\begin{restatable}[Existence of solution to the gradient flow with simplex features]{lemma}{ExistSolSimplex}
\label{lem:main:mdp:exist_sol_gf_simplex}
Let Assumption \ref{asp:app:mdp:max_set_simplex} hold.
Then there exists a unique global solution $\{\theta_t\}_{t \geq 0}$ to (\ref{eq:intro:continuous_GF}).
Moreover, $\sup_{t \geq 0} \big|\log \frac{\mathrm d \pi_{\theta_t}}{\mathrm d\mu}\big|_{B_b(S \times A)} < \infty$.
\end{restatable}

The proof of Lemma \ref{lem:main:mdp:exist_sol_gf_simplex} can be found in Section \ref{sec:pf_exist_sol_gf_simplex}.

\begin{restatable}{lemma}{aspInitSimplex}
\label{lemma:app:mdp:initial_singularity_simplex}
Let Assumption~\ref{asp:app:mdp:max_set_simplex} hold.
Then for all $\theta_0 \in \mathbb R^p$ and $s \in S$, $\int_A g(s,a)g^\top(s,a)\pi_{\theta_0} (\mathrm{d}a | s)$ is positive definite.
\end{restatable}

The proof of Lemma \ref{lemma:app:mdp:initial_singularity_simplex} can be found in Section \ref{sec:pf_lin_cov_simplex}.

\begin{restatable}[Linear convergence rate using simplex features]{theorem}{LinCovSimplex}\label{thm:main:lin_cov_simplex}
Let Assumptions~\ref{asp:app:mdp:Q_realisability} and~\ref{asp:app:mdp:max_set_simplex} hold.
Let $\{\theta_t\}_{t \geq 0}$ be the solution to gradient flow (\ref{eq:intro:continuous_GF}) and $R :=\textstyle \sup_{t\ge0}\big|\log\tfrac{\mathrm d\pi_{\theta_t}}{\mathrm d\mu}\big|_{B_b(S\times A)}$ given by Lemma~\ref{lem:main:mdp:exist_sol_gf_simplex}.
Then there exists $C_{\theta_0} > 0$ such that 
for $C_R(\theta_t)$ from Theorem~\ref{thm:main: non-uniform_PL_ineq} we have $\sup_{t \geq 0} C_R(\theta_t) \leq C_{\theta_0}$ and
\[
\textstyle
0 \leq V_\tau^{\pi_{\theta_t}}(\rho) - V_\tau^{\pi^*}(\rho) \leq e^{-t C_{\theta_0}^{-1}} 
(V_\tau^{\pi_{\theta_0}}(\rho) - V_\tau^{\pi^*}(\rho)).
\]
\end{restatable}
The exact form of $C_{\theta_0}$ can be found in Section \ref{sec:pf_lin_cov_simplex}.

We conclude with examples of simplex features satisfying Assumption~\ref{asp:app:mdp:max_set_simplex}.

\begin{example}[Bernstein polynomial features]\label{ex:Bernstein_Poly_MDP}
Let $S\subset\mathbb R^{d_1}$ be nonempty and compact, and let $A\subset\mathbb R^{d_2}$ be compact with nonempty interior. 
Assume $\mu$ is absolutely continuous with respect to the appropriate $d_2$-dimensional Lebesgue measure restricted to $A$.
Here, nonempty interior means that there exist $a_0\in A$ and $r>0$ such that the open ball $B(a_0,r)\subset A$. 
Fix $e\in\mathbb R^{d_2}\setminus\{\mathbf 0\}$ and let $q:S\to\mathbb R$ be continuous. Define $\widetilde h(s,a):=q(s)+e^\top a$.

Let $m_h:=\min_{(s,a)\in S\times A}\widetilde h(s,a)$ and $M_h:=\max_{(s,a)\in S\times A}\widetilde h(s,a)$. 
These extrema exist because $S\times A$ is compact and $\widetilde h$ is continuous. 
Moreover, $M_h>m_h$. 
Indeed, since $A$ has nonempty interior, there exist $a_0\in A$ and $r>0$ such that $B(a_0,r)\subset A$. 
Let $v:=e/\|e\|$, and define $a_+:=a_0+(r/2)v$ and $a_-:=a_0-(r/2)v$. 
Then $a_+,a_-\in A$, and for any fixed $s_0\in S$, we have $\widetilde h(s_0,a_+)-\widetilde h(s_0,a_-)=e^\top(a_+-a_-)=r\|e\|>0$. 
Hence $\widetilde h$ is not constant on $S\times A$, so its maximum and minimum are distinct. 
Define $h:S\times A\to[0,1]$ by $h(s,a):=(\widetilde h(s,a)-m_h)/(M_h-m_h)$.

Fix an integer $k\geq 1$. 
For $\ell=0,\dots,k$, define the one-dimensional Bernstein basis function $b_{\ell,k}(w):=\binom{k}{\ell}w^\ell(1-w)^{k-\ell}$ for $w\in[0,1]$. 
Define the feature $g:S\times A\to\mathbb R^{k+1}$ by $g(s,a):=(b_{0,k}(h(s,a)),\dots,b_{k,k}(h(s,a)))^\top$. 
Then $g_\ell(s,a)\geq 0$ for every $\ell$, and $\sum_{\ell=0}^k g_\ell(s,a)=1$ for every $(s,a)\in S\times A$.
For $u=(u_0,\dots,u_k)\in\mathbb R^{k+1}$, define $f_u(s,a):=u^\top g(s,a)=\sum_{\ell=0}^ku_\ell b_{\ell,k}(h(s,a))$. 
Equivalently, $f_u(s,a)=P_u(h(s,a))$, where $P_u(w):=\sum_{\ell=0}^ku_\ell b_{\ell,k}(w)$.
Assume that $u\neq \mathbf 0$ and $u^\top\mathbf 1=\sum_{\ell=0}^ku_\ell=0$. 
Then $P_u : [0,1] \to \mathbb R$ is non-constant. 
To see this, suppose that $P_u\equiv c$ on $[0,1]$. 
Since $\sum_{\ell=0}^k b_{\ell,k}(w)=1$, the constant polynomial $c$ satisfies $c=\sum_{\ell=0}^k c\,b_{\ell,k}(w)$. 
By uniqueness of expansion in the Bernstein basis, $u_\ell=c$ for every $\ell=0,\dots,k$. 
The condition $u^\top\mathbf 1=0$ gives $0=\sum_{\ell=0}^ku_\ell=(k+1)c$, so $c=0$, and hence $u=\mathbf0$, contradicting $u \neq \mathbf 0$.
Therefore $P_u : [0,1] \to \mathbb R$ is non-constant.

For each fixed state $s\in S$, define $M_u(s):=\max_{a\in A}f_u(s,a)$ and $\mathcal A_u(s):=\operatorname*{arg\,max}_{a\in A}f_u(s,a)$. 
The maximum exists because $A$ is compact and $a\mapsto f_u(s,a)$ is continuous. 
Since $P_u : [0,1] \to \mathbb R$ is a non-constant univariate polynomial, $P_u(w)-M_u(s)$ is not the zero polynomial. Hence the set $H_u(s):=\{w\in\mathbb R:P_u(w)=M_u(s)\}$ is finite.
Now, if $a\in\mathcal A_u(s)$, then $P_u(h(s,a))=M_u(s)$, so $h(s,a)\in H_u(s)$. 
Therefore $\mathcal A_u(s)\subseteq \bigcup_{w\in H_u(s)}\{a\in A:h(s,a)=w\}$. 
For fixed $s$ and $w$, the level set $\{a\in A:h(s,a)=w\}$ is contained in the affine hyperplane $\{a\in\mathbb R^{d_2}:e^\top a=m_h+(M_h-m_h)w-q(s)\}$. 
Since $e\neq  \mathbf 0$, this hyperplane has $d_2$-dimensional Lebesgue measure zero. 
Thus $\mathcal A_u(s)$ is contained in a finite union of measure-zero sets, and hence $\mu \left(\arg \max_{a\in A} u^\top g(s,a) \right) = \mu(\mathcal A_u(s))=0$.
Note that $p=k+1$ in this example.
\end{example}

\begin{example}[Softmax kernel features]
\label{ex:ntk}
Let $S=A=[0,1]^k$, fix $\beta>0$, an NTK $K$ on $S\times A$, and landmarks $\xi_j=(q_j,z_j)\in S\times A,j=1,\ldots,p.$
Assume that for all $s,q\in S$ and $z\in A$, the map $a\mapsto K((s,a),(q,z))$ extends to a real-analytic function on an open connected neighborhood of $A$.
Define $g_j(s,a) = \textstyle \frac{\exp\left(\beta K((s,a),\xi_j)\right)}{\sum_{i=1}^p\exp\left(\beta K((s,a),\xi_i)\right)}.$
For fixed $s$, define $h_{s,j}(a) := \textstyle \exp \left(\beta K((s,a),\xi_j)\right),  H_s(a):=\sum_{j=1}^p h_{s,j}(a)$.
Thus $g_j(s,a)=\tfrac{h_{s,j}(a)}{H_s(a)}$.
Assume $\mu$ is absolutely continuous with respect to the appropriate $k$-dimensional Lebesgue measure restricted to $A$.

If the functions $h_{s,1},\ldots,h_{s,p}$ are linearly independent as functions of $a$, then for every $u\in\mathbf 1^\perp\setminus\{\mathbf0\}$, the function $a\mapsto u^\top g(s,a)$ is non-constant.
Indeed, suppose for contradiction that $u^\top g(s,\cdot)\equiv c$.
Multiplying this identity by $H_s$ gives $\sum_{j=1}^p (u_j-c)h_{s,j}(a)=0.$
If $h_{s,1},\ldots,h_{s,p}$ are linearly independent, then $u_j=c, \text{for every }j$.
Hence $u=c\mathbf 1_p$. 
Since $u\perp\mathbf 1_p$, it follows that $c=0$ and $u=0$.

Then we can show that $\mu \left(\arg \max_{a\in A} u^\top g(s,a) \right)=0.$
In fact, fix $s\in S$ and $u\in\mathbf 1^\perp\setminus\{\mathbf0\}$, and define $f_{s,u}(a):=u^\top g(s,a)$.
The function $f_{s,u}$ is real analytic on an open connected neighborhood of $A$, because it is a quotient of real-analytic functions whose denominator is strictly positive.
By the preceding argument, we know $f_{s,u}$ is non-constant.
Since $A$ is compact, $M_{s,u} := \max_{a\in A}f_{s,u}(a)$ exists. 
The function $a \mapsto f_{s,u}(a)-M_{s,u}$ is a nonzero real-analytic function, and $\arg \max_{a\in A}f_{s,u}(a) = \left\{ a\in A: f_{s,u}(a)-M_{s,u}=0 \right\}.$
The zero set of a nonzero real-analytic function on a connected open set
has Lebesgue measure zero.
Hence $\mu \left( \arg \max_{a\in A}u^\top g(s,a) \right)=0.$

Consider the Gaussian Radial Basis Function (RBF) random-feature NTK. First of all let  $K_S $ be any measurable state NTK on $S$.
In particular, it need not be strictly positive definite. 
Next, we will construct $K$ as a limit: $K((s,a),(q,z)) = K_S(s,q) + K_A(a,z)$, where $K_A(a,z) = \lim_{n\to \infty} K_n(a,z)$.
Let us now construct $K_n$ and hence obtain $K_A$ from the law of large numbers.
Set $F_{\theta}^{A}(a) = \textstyle \sqrt{\frac{2}{n}} \sum_{r=1}^{n} \theta_r\cos(\omega_r^\top a+b_r).$
Only the output weights  $\theta_r $ are trained, while
$\omega_r\overset{\mathrm{i.i.d.}}{\sim} N(0,\ell^{-2}I_k)$ with $\ell >0$, and $b_r\overset{\mathrm{i.i.d.}}{\sim} \operatorname{Unif}[0,2\pi]$
are sampled mutually independently and held fixed.
Since $\nabla_\theta F_{\theta}^{A}(a) = \textstyle\sqrt{\frac{2}{n}} \bigl( \cos(\omega_1^\top a+b_1),\ldots, \cos(\omega_n^\top a+b_n) \bigr),$ we write $K_n(a,z)  =\textstyle \nabla_\theta F^A_\theta(a)^\top \nabla_\theta F^A_\theta(z) = \frac{2}{n}\sum_{r=1}^{n} \cos(\omega_r^\top a+b_r) \cos(\omega_r^\top z+b_r).$
Due to the law of large numbers for each fixed $(a,z)$, almost surely, $K_A(a,z) := \lim_{n\to \infty} K_n(a,z) = 2\mathbb E_{\omega,b} \left[ \cos(\omega^\top a+b)
\cos(\omega^\top z+b) \right] = \mathbb E_{\omega} \left[\cos\bigl(\omega^\top(a-z)\bigr)\right]= \exp\left(-\frac{\|a-z\|^2}{2\ell^2}\right).$
Let the landmark actions  $z_1,\ldots,z_p $ be pairwise distinct.
For a fixed state  $s$, $h_{s,j}(a)=c_{s,j}\exp\bigl(\beta K_A(a,z_j)\bigr), c_{s,j} :=\exp\bigl(\beta K_S(s,q_j)\bigr)>0.$

We now verify that $h_{s,1},\ldots,h_{s,p} $ are linearly independent as functions of $a$.
Define $\mathcal K := \left[
K_A(z_i,z_j)
\right]_{i,j=1}^{p}$.
The Gaussian RBF kernel is strictly positive definite: $\mathcal K \succ 0$. 
Consider the entry-wise exponential matrix $E := \left[
\exp\bigl(\beta K_A(z_i,z_j)\bigr)
\right]_{i,j=1}^{p}$.
It can be verified that $E \succ 0$ by the Taylor series expansion of the exponential function combined with the Schur product theorem. 
Element-wise exponentiation preserves strict positive definiteness.
The collocation matrix obtained is $\left[
h_{s,j}(z_i)\right]_{i,j=1}^{p}=E \operatorname{diag}
\left(c_{s,1},\ldots,c_{s,p}\right),$ where both factors on the right-hand side are invertible.
Hence $h_{s,1},\ldots,h_{s,p} $ are linearly independent as functions of $a$.
Then we have that $\mu \left(\arg \max_{a\in A} u^\top g(s,a) \right)=0.$
\end{example}

\subsection*{Example of a feature basis not providing radial unboundedness} 

In the bandit setting, where the MDP has a singleton state space, we have the following example using simplex features that do not give radial unboundedness of the $\operatorname{KL}$ term in the subspace orthogonal to $\mathbf 1$.
The detailed calculation can be found in Section \ref{sec:ex}.
\begin{example}[Hat functions do not give radial unboundedness in $\mathbf 1^\perp $]
\label{ex:hat_func_short}
Let $A=[0,1]$, let $\mu$ be the uniform distribution on $A$, and consider the uniform grid with nodes $x_i=i/3$, $i=0,1,2,3$, and mesh size $\bar h=\frac13$. 
The associated one-dimensional finite-element (hat) basis functions are  $g_i(a):=\max\Bigl\{0,\,1-\tfrac{|a-x_i|}{\bar h}\Bigr\}, a\in A,\ i=0,1,2,3,$ so that $g_i$ is piecewise linear on the elements $[x_{j-1},x_j]$, satisfies $g_i(x_j)=\delta_{ij}$, and is supported on $[x_{i-1},x_{i+1}]\cap A$. Set $g(a):=(g_0(a),\dots,g_3(a))^\top$, so $p=4$.
Since $g_i\ge0$ and $\sum_{i=0}^3 g_i\equiv1$ on $A$, these are simplex features.
Let $\theta=(-1,1,1,-1)^\top$.
Then $\theta^\top \mathbf 1=0$, and the induced function $f_\theta(a):=\theta^\top g(a)$ is
$
f_\theta(a) = 6a-1\,\,\text{if}\, a\in\left[0,\tfrac13\right),\,\,\,
f_\theta(a) = 1\,\,\text{if}\, a\in\left[\tfrac13,\tfrac23\right),\,\,\,
f_\theta(a) = 5-6a\,\,\text{if}\, a\in\left[\tfrac23,1\right].
$
Hence $\mu\bigl(\operatorname*{arg\,max}_{a\in A}f_\theta(a)\bigr)=\mu([\tfrac13,\tfrac23])=\tfrac13>0$, and Assumption~\ref{asp:app:mdp:max_set_simplex} fails.

This also shows that radial unboundedness of the entropy may fail without the zero-measure maximizer assumption.  For $\beta>0$, let $\pi_\beta$ be the probability measure whose density with respect to $\mu$ is $p_\beta(a):=\frac{\mathrm d\pi_\beta}{\mathrm d\mu}(a)=e^{\beta f_\theta(a)}/Z_\beta$, where $Z_\beta:=\int_0^1 e^{\beta f_\theta(a)}\,\mathrm{d}a$. 
Since $f_\theta$ is linear on the two exterior intervals and equals $1$ on $[\frac13,\frac23]$, the change of variables $y=f_\theta(a)$ gives
$
Z_\beta = \textstyle e^\beta\left[
\frac13+\frac{1}{3\beta}\bigl(1-e^{-2\beta}\bigr)\right]=\frac{e^\beta}{3}\left[1+\frac{1-e^{-2\beta}}{\beta}\right].
$
Moreover, we have $\mathbb E_{\pi_\beta}[f_\theta(a)]=\partial_\beta\log Z_\beta$, and therefore
$
\operatorname{KL}(\pi_\beta|\mu)= \textstyle \beta\,\mathbb E_{\pi_\beta}[f_\theta(a)]-\log Z_\beta=\beta\,\partial_\beta\log Z_\beta-\log Z_\beta .
$
From the exact expression for $Z_\beta$, we have $\log Z_\beta=\beta-\log 3+\log(1+(1-e^{-2\beta})/\beta)$. Hence, as $\beta\to\infty$,
$
\log Z_\beta
= \textstyle \beta-\log 3+\frac1\beta+\mathcal O(\beta^{-2}),
\ \beta\,\partial_\beta\log Z_\beta=\textstyle\beta-\frac1\beta+\mathcal O(\beta^{-2}).
$
Consequently, $\operatorname{KL}(\pi_\beta|\mu) = \log 3-\frac{2}{\beta}+\mathcal O(\beta^{-2})$,
and in particular $\lim_{\beta\to\infty}\operatorname{KL}(\pi_\beta|\mu)=\log 3$, which is finite.
\end{example}

\allowdisplaybreaks
\section{Proofs}
\subsection{Basic notation and definitions}
\label{sec:nota and def}
For matrix $G \in \mathbb{R}^{p \times p}$, denote $\lambda_{\min} (G)$ the smallest eigenvalue of $G$. 
Let $\mathbb S^{p-1} := \left\{u\in\mathbb R^p:\|u\|_2=1\right\}.$ 
Let $ \mathbf{1} = (1,1, \dots,1) \in \mathbb{R}^p$.
For vectors $\mathbf{a}, \mathbf{b} \in \mathbb{R}^p$, the inner product $\langle \mathbf{a}, \mathbf{b}\rangle$ of $\mathbf{a}\ \text{and} \ \mathbf{b}$ is $\mathbf{a}^\top\mathbf{b}$.
And we use both notations throughout the paper.
Let $I_p \in \mathbb R^{p \times p}$ be the identity matrix.

Let $(E,d)$ denote a complete separable metric space (i.e., a Polish space). 
For a given measure $\rho$ on $E$ and $1\leq p<\infty$, let $L^p(E,\rho)$ denote the space of $p$-integrable functions modulo equality $\rho$-almost everywhere. 
Let $L^\infty(E,\rho)$ denote the space of essentially bounded functions.
We always equip a Polish space with its Borel sigma-field $\mathcal{B}(E).$ 
Denote by $B_b(E)$ the space of bounded strongly measurable functions $f:E\rightarrow \mathbb{R}$ endowed with the supremum norm $|f|_{B_b(E)}=\sup_{x\in E}|f(x)|$. 
Denote by $\mathcal{M}(E)$ the Banach space of signed measures (finite) $\mu$ on $E$ endowed with the total variation norm $|\mu|_{\mathcal{M}(E)}=|\mu|(E)$, where $|\mu|$ is the total-variation measure.
We note that if $\mu=f d\rho$, where $\rho \in \mathcal{M}_+(E)$ is a non-negative measure and $f\in L^1(E,\rho)$, then  $|\mu|_{\mathcal{M}(E)}=|f|_{L^1(E,\rho)}$. 
We denote by $\mathcal{P}(E)\subset\mathcal{M}(E)$ the convex subset of probability measures on $E$. 
For $\mu,\mu'\in \mathcal{P}(E)$ such that $\mu$ is absolutely continuous with respect to $\mu'$, the relative entropy of $\mu$ with respect to $\mu'$ (or KL divergence of $\mu$ relative to $\mu'$) is defined by 
\[
\operatorname{KL}(\mu|\mu')=\int_{E}\log\frac{\mathrm d\mu}{\mathrm d\mu'}(x)\mu(\mathrm dx)\,.
\]

It is convenient to have notation for measurable functions $k:E_1\rightarrow \mathcal{M}(E_2)$ for given Polish spaces $(E_1,d_1)$ and $(E_2,d_2)$. 
For example, $P:S\times A\rightarrow \mathcal{P}(S)$ will denote a controlled transition probability and $\pi: S\rightarrow \mathcal{P}(A)$ a stochastic policy. 
Denote by $b\mathcal{K}(E_1|E_2)$ the Banach space of bounded signed kernels $k: E_2 \rightarrow \mathcal{M}(E_1)$ endowed with the norm $|k|_{b\mathcal{K}(E_1|E_2)}=\sup_{x\in E_2}|k(x)|_{\mathcal{M}(E_1)}$; that is, $k(U|\cdot): E_2\rightarrow \mathbb{R}$ is measurable for all $U\in \mathcal{B}(E_1)$ and $k(\cdot|x)\in \mathcal{M}(E_1)$ for all $x\in E_2$.
For a fixed positive reference measure $\mu \in \mathcal{M}(E_1)$, we denote by $b\mathcal{K}_{\mu}(E_1|E_2)$ the space of bounded kernels that are absolutely continuous with respect to $\mu$. 

Every kernel $k\in b\mathcal{K}(E_1|E_2)$ induces bounded linear operators $T_k \in \mathcal{L}(\mathcal{M}(E_2),\mathcal{M}(E_1))$ and $S_k \in \mathcal{L}(B_b(E_1),B_b(E_2))$ defined by
\[
T_k\mu(\mathrm d y)=\mu k(\mathrm d y)=\int_{E_2}\mu(\mathrm d x)k(\mathrm d y|x)
\]
and 
\[
S_kf(x)=\int_{E_1}k(\mathrm d y|x)f(y)\,,
\]
respectively.
Moreover, by Example 2.3 and Proposition 3.1 in~\citet{kunze2011pettis}, we have
\begin{align*}
|k|_{b\mathcal{K}(E_1|E_2)}&=\sup_{x\in E_2}\underset{|h|_{B_b(E_1)}\le 1}{\sup_{h\in B_b(E_1)}}\int_{E_1}h(y)k(\mathrm d y|x)\\
&=|S_k|_{\mathcal{L}(B_b(E_1),B_b(E_2))}\\
&=|T_k|_{\mathcal{L}(\mathcal{M}(E_2),\mathcal{M}(E_1))}\,,
\end{align*}
where the latter are operator norms. 
Thus, $b\mathcal{K}(E|E)$ is a Banach algebra with the product defined via composition of the corresponding linear operators; in particular, for a given $k\in b\mathcal{K}(E|E)$,
\[
\begin{split}
& T_k^n\mu(\mathrm d y)=\mu k^n(\mathrm d y) = \int_{E^{n}}\mu(\mathrm d x_0)k(\mathrm dx_1|x_0)\cdots k(\mathrm dx_{n-1}|x_{n-2}) k(\mathrm d y|x_{n-1})\,.
\end{split}
\]
Notice that if $f\in L^\infty(E_1,\mu)$ and $k\in b\mathcal{K}_{\mu}(E_1|E_2)$, then for all $x\in E_2$,
\begin{align}
|S_kf(x)|&=\left|\int_{E_1}\mu(\mathrm d y)\frac{\mathrm dk}{\mathrm d\mu}(y|x)f(y)\right|\notag\\
&\le |f|_{L^\infty(E_1,\mu)} \left|\frac{\mathrm dk}{\mathrm d\mu}(\cdot|x)\right|_{L^1(E_1,\mu)}\notag \\
&\le |f|_{L^\infty(E_1,\mu)} |k|_{b\mathcal{K}(E_1|E_2)}\,. \label{ineq:app:essup_kernel}
\end{align}

We denote by $\mathcal{P}(E_1|E_2)$ the convex subset of $P\in b\mathcal{K}(E_1|E_2)$ such that $P(\cdot|x)\in \mathcal{P}(E_1)$ for all $x\in E_2$; such kernels are referred to as stochastic kernels. 
A stochastic kernel $P\in \mathcal{P}(E_1|E_2)$ is said to be strongly Feller if $\int_{E_1}P(\mathrm d y|x)f(y)$ is continuous in $x\in E_2$ for all $f\in B_b(E_1)$. 
For a fixed positive reference measure $\mu \in \mathcal{M}(E_1)$, we denote by $\mathcal{P}_{\mu}(E_1|E_2)$ the space of kernels that are absolutely continuous with respect to $\mu$. 
A bounded kernel $k\in b\mathcal{K}(E_1|E_2)$ is thus strongly Feller if the range of $S_k$ lies in the space of continuous functions on $E_2$. 
\subsection{Basic results on entropy regularized MDPs}
\label{sec:basic}

The following lemma (see, e.g., Lemma 2.3~\citealp{kerimkulov2025fisher}) is crucial for addressing this non-convexity issue. 
\begin{lemma}[Performance difference]
\label{lem:app:mdp:performance_diff}
For all $\rho \in \mathcal{P}(S)$ and $\pi,\pi'\in \Pi_{\mu}$, 
\begin{align*}
&V^{\pi}_\tau(\rho)-V^{\pi'}_\tau(\rho) \\
& = \frac{1}{1-\gamma}\int_S \bigg[\int_A\left(Q^{\pi'}_{\tau}(s,a)+\tau \log \frac{\mathrm d \pi'}{\mathrm d\mu}(a|s)\right)(\pi-\pi')(\mathrm{d}a|s) + \tau    \operatorname{KL}(\pi(\cdot | s)|\pi'(\cdot | s)) \bigg]d^{\pi}_\rho(\mathrm{d}s)\,.
\end{align*}
\end{lemma}

Define the proximal policy
\begin{equation}\label{eq:app:mdp:proximal_policy}
\pi_{\pi^{\prime}}(\mathrm{d}a | s)=\frac{1}{Z_{\pi^{\prime}}(s)} \exp \left(-\frac{1}{\tau}\left(Q^{\pi^{\prime}}_\tau(s, a)-V^{\pi^{\prime}}_
\tau(s)\right)\right) \mu(\mathrm{d}a),
\end{equation}
where $Z_{\pi^{\prime}}(s):=\int_A \exp \left(-\frac{1}{\tau}\left(Q^{\pi^{\prime}}_\tau\left(s, a^{\prime}\right)-V^{\pi^{\prime}}_\tau(s)\right)\right) \mu\left(\mathrm{d}a^{\prime}\right)$.
Then we have the following sandwich inequality for the suboptimality gap in terms of KL divergence:
\begin{lemma}[Sandwich Inequality]\label{lem:sandwich-inequality}
Let $\pi_{\bar\pi}$ denote the proximal policy step associated with a policy
$\bar\pi$. 
Then, for any policies $\pi,\pi'\in\Pi_\mu$ and any initial distribution $\rho\in\mathcal P(S)$,
\begin{equation}
V_\tau^{\pi'}(\rho)-V_\tau^\pi(\rho)
\le
\frac{\tau}{1-\gamma}
\int_S
\mathrm{KL}(\pi'(\cdot | s)| \pi_{\pi'}(\cdot |s))
\,d_\rho^\pi(\mathrm{d}s).
\label{eq:sandwich-general-bound}
\end{equation}
Moreover, if $\pi^*$ is an optimal policy, then for any policy $\pi'$ and any
initial distribution $\rho$,
\begin{equation}
\frac{\tau}{1-\gamma}
\int_S
\mathrm{KL}(\pi'(\cdot | s)| \pi^*(\cdot | s))
\,d_\rho^{\pi'}(\mathrm{d}s)
=
V_\tau^{\pi'}(\rho)-V_\tau^{\pi^*}(\rho)
\le
\frac{\tau}{1-\gamma}
\int_S
\mathrm{KL}(\pi'(\cdot | s)| \pi_{\pi'}(\cdot |s))
\,d_\rho^{\pi^*}(\mathrm{d}s).
\label{eq:sandwich-inequality}
\end{equation}
\end{lemma}

\begin{proof}
The flat derivative of the objective can be written in terms of the proximal
policy (\ref{eq:app:mdp:proximal_policy}) as
\begin{align}
\frac{\delta V_\tau^{\pi'}}{\delta \pi}
&=
Q_\tau^{\pi'}-V_\tau^{\pi'}
+\tau \log \frac{\mathrm d\pi'}{\mathrm d\mu}
\nonumber \\
&=
\tau \log \frac{\mathrm d\pi'}{\mathrm d\mu}
-\tau\log\exp\left(-\frac{1}{\tau}\left(Q_\tau^{\pi'}-V_\tau^{\pi'}\right)\right)+\tau \log Z_{\pi'}-\tau \log Z_{\pi'}
\nonumber \\
&=
\tau \log \frac{\mathrm d\pi'}{\mathrm d\mu}-\tau \log \frac{\mathrm d\pi_{\pi'}}{\mathrm d\mu}-\tau \log Z_{\pi'}
=
\tau \log \frac{\mathrm d\pi'}{\mathrm d\pi_{\pi'}}-\tau \log Z_{\pi'} .
\label{eq:sandwich-flat-derivative}
\end{align}

By the regularized performance-difference Lemma \ref{lem:app:mdp:performance_diff}, for any $\pi,\pi' \in \Pi_\mu$ and
any $\rho$,
\begin{align}
V_\tau^\pi(\rho)
&=
V_\tau^{\pi'}(\rho)
+
\frac{1}{1-\gamma}\int_S\left[\int_A\frac{\delta V_\tau^{\pi'}}{\delta \pi}(s,a)(\pi-\pi')(\mathrm{d}a| s)
+
\tau \operatorname{KL}(\pi(\cdot | s)| \pi'(\cdot | s))\right]d_\rho^\pi(\mathrm{d}s).
\label{eq:sandwich-performance-difference}
\end{align}
Using the expression of the flat derivative, we obtain
\begin{align}
V_\tau^\pi(\rho)
&=
V_\tau^{\pi'}(\rho)
+
\frac{1}{1-\gamma}\int_S\bigg[\int_A\left(Q_\tau^{\pi'}(s,a)
+
\tau \log \frac{\mathrm d\pi'}{\mathrm d\mu}(a|s)-V_\tau^{\pi'}(s)\right)(\pi-\pi')(\mathrm{d}a| s)
\nonumber \\
&\hspace{4cm}
+
\tau \operatorname{KL}(\pi(\cdot | s)| \pi'(\cdot |s))\bigg]d_\rho^\pi(\mathrm{d}s)
\nonumber \\
&=
V_\tau^{\pi'}(\rho)
+
\frac{1}{1-\gamma}
\int_S
\bigg[\int_A\left(Q_\tau^{\pi'}(s,a)
+
\tau \log \frac{\mathrm d\pi'}{\mathrm d\mu}(a|s)\right)(\pi-\pi')(\mathrm{d}a| s)
\nonumber \\
&\hspace{4cm}+\tau \operatorname{KL}(\pi(\cdot | s)| \pi'(\cdot |s))\bigg]d_\rho^\pi(\mathrm{d}s)\nonumber \\
&=
V_\tau^{\pi'}(\rho)+
\frac{1}{1-\gamma}\int_S\bigg[\int_A\left(Q_\tau^{\pi'}(s,a)+\tau \log \frac{\mathrm d\pi'}{\mathrm d\mu}(a | s)\right)(\pi-\pi')(\mathrm{d}a| s)
\nonumber \\
&\hspace{4cm}
+\int_A\tau \log \frac{\mathrm d\pi}{\mathrm d\pi'}(a | s)\,\pi(\mathrm{d}a| s)\bigg]d_\rho^\pi(\mathrm{d}s)
\nonumber \\
&=
V_\tau^{\pi'}(\rho)
+\frac{1}{1-\gamma}\int_S\bigg[\int_A Q_\tau^{\pi'}(s,a)(\pi-\pi')(\mathrm{d}a| s)+\tau \operatorname{KL}(\pi(\cdot |s)| \mu)
\nonumber \\
&\hspace{4cm}-\tau\int_A\log \frac{\mathrm d\pi'}{\mathrm d\mu}(a|s)\,\pi'(\mathrm{d}a| s)\bigg]d_\rho^\pi(\mathrm{d}s).
\label{eq:sandwich-expanded-difference}
\end{align}

Minimizing the term inside the integral over $\pi$ for every $s\in S$ and
recalling the proximal policy step gives
\begin{align}
V_\tau^\pi(\rho)
&\ge V_\tau^{\pi'}(\rho)+\frac{1}{1-\gamma}
\int_S\left[\int_A\frac{\delta V_\tau^{\pi'}}{\delta \pi}(s,a)(\pi_{\pi'}-\pi')(\mathrm{d}a| s)+\tau \operatorname{KL}(\pi_{\pi'}(\cdot | s)| \pi'(\cdot | s))\right]d_\rho^\pi(\mathrm{d}s).
\label{eq:sandwich-proximal-lower-bound}
\end{align}
From \eqref{eq:sandwich-flat-derivative}, for any $\pi,\pi'$,
\begin{align}
&\int_A
\frac{\delta V_\tau^{\pi'}}{\delta \pi}(s,a)(\pi_{\pi'}-\pi')(\mathrm{d}a| s)
+
\tau \operatorname{KL}(\pi_{\pi'}(\cdot | s)| \pi'(\cdot | s))
\nonumber \\
&=
\tau
\int_A\log \frac{\mathrm d\pi'}{\mathrm d\pi_{\pi'}}(a|s)(\pi_{\pi'}-\pi')(\mathrm{d}a| s)
+
\tau\int_A\log \frac{\mathrm d\pi_{\pi'}}{\mathrm d\pi'}(a|s)\,\pi_{\pi'}(\mathrm{d}a| s)
\nonumber \\
&=
-\tau\int_A\log \frac{\mathrm d\pi_{\pi'}}{\mathrm d\pi'}(a|s)(\pi_{\pi'}-\pi')(\mathrm{d}a| s)
+
\tau\int_A\log \frac{\mathrm d\pi_{\pi'}}{\mathrm d\pi'}(a|s)\,\pi_{\pi'}(\mathrm{d}a| s)
\nonumber \\
&=
\tau\int_A\log \frac{\mathrm d\pi_{\pi'}}{\mathrm d\pi'}(a|s)\,\pi'(\mathrm{d}a| s)
=-\tau \operatorname{KL}(\pi'(\cdot | s)| \pi_{\pi'}(\cdot |s)).
\label{eq:sandwich-kl-cancellation}
\end{align}
Combining \eqref{eq:sandwich-proximal-lower-bound} and
\eqref{eq:sandwich-kl-cancellation}, we get
\begin{equation}
V_\tau^\pi(\rho)\ge V_\tau^{\pi'}(\rho)-\frac{\tau}{1-\gamma}\int_S\mathrm{KL}(\pi'(\cdot | s)| \pi_{\pi'}(\cdot |s))\,d_\rho^\pi(\mathrm{d}s).
\label{eq:sandwich-lower-bound}
\end{equation}
Equivalently,
\begin{equation}
V_\tau^{\pi'}(\rho)-V_\tau^\pi(\rho)\le\frac{\tau}{1-\gamma}\int_S\mathrm{KL}(\pi'(\cdot | s)| \pi_{\pi'}(\cdot |s))\,d_\rho^\pi(\mathrm{d}s),
\label{eq:sandwich-value-gap-upper}
\end{equation}
which proves \eqref{eq:sandwich-general-bound}.

Finally, the optimality condition for $\pi^*$ gives, for all $s \in S,a \in A$,
\begin{equation}
\frac{\delta V_\tau^{\pi^*}}{\delta \pi}=Q_\tau^{\pi^*}-V_\tau^{\pi^*}+\tau \log \frac{\mathrm d\pi^*}{\mathrm d\mu}=0.
\label{eq:sandwich-optimality-condition}
\end{equation}
Using this in the regularized performance-difference identity yields
\begin{equation}
V_\tau^{\pi'}(\rho)-V_\tau^{\pi^*}(\rho)
= \frac{\tau}{1-\gamma}\int_S\mathrm{KL}(\pi'(\cdot |s)| \pi^*(\cdot | s))
\,d_\rho^{\pi'}(\mathrm{d}s).
\label{eq:sandwich-optimal-gap}
\end{equation}
Taking $\pi=\pi^*$ in \eqref{eq:sandwich-value-gap-upper} and combining with
\eqref{eq:sandwich-optimal-gap} gives
\eqref{eq:sandwich-inequality}.
\end{proof}

We introduce the following dynamic programming principle (see \citealp[Theorem~B.1]{kerimkulov2025fisher}).
\begin{theorem}[Dynamic programming principle]\label{thm:app:dpp}
Let $\tau>0$. The optimal value function $V_\tau^*$ is the unique bounded solution of the following Bellman equation:
\[
V_\tau^*(s)=\inf _{m \in \mathcal{P}(A)} \int_A\left(c(s, a)+\tau \log \frac{\mathrm d m}{\mathrm{d} \mu}(a)+\gamma \int_S V_\tau^*\left(s^{\prime}\right) P\left(\mathrm{d}s^{\prime} | s, a\right)\right) m(\mathrm{d}a), \quad \forall s \in S\,.
\]
Consequently, for all $s \in S$,
\[
V_\tau^*(s)=-\tau \log \int_A \exp \left(-\frac{1}{\tau} Q_\tau^*(s, a)\right) \mu(\mathrm{d}a)\,,
\]
where $Q^*_\tau \in B_b(S \times A)$ is defined by
\[
Q_\tau^*(s, a)=c(s, a)+\gamma \int_S V_\tau^*\left(s^{\prime}\right) P\left(\mathrm{d}s^{\prime} | s, a\right), \quad \forall(s, a) \in S \times A \,.
\]
Moreover, there is an optimal policy $\pi_\tau^* \in \Pi_\mu$ given by
\[
\pi_\tau^*(\mathrm{d}a|s)=\exp \left(-\left(Q_\tau^*(s, a)-V_\tau^*(s)\right) / \tau\right) \mu(\mathrm{d}a), \quad \forall s \in S \,.
\]
\end{theorem}

\begin{lemma}\label{lem:bounds}
For any $\theta \in \mathbb{R}^p$, we have
\begin{gather*}
\left|V^{\pi_\theta}_{\tau}(s)\right|\le \frac{1}{1-\gamma} \left(1+\tau \bigg|\log \frac{\mathrm d\pi_{\theta}}{\mathrm d \mu}\bigg|_{B_b(S \times A)} \right),\\
|Q^{\pi_{\theta}}_{\tau}(s,a)| \le \frac{1}{1-\gamma} \left(1+\gamma\tau \bigg|\log \frac{\mathrm d\pi_{\theta}}{\mathrm d \mu}\bigg|_{B_b(S \times A)} \right).
\end{gather*}
\end{lemma}
\begin{proof}

By the bound on cost function and Eq.~\eqref{eq:intro:value_function_occ}, we obtain
\begin{equation}\label{eq:bound_on_V_pi}
\begin{aligned}
|V^{\pi_\theta}_{\tau}(s)|&=\frac{1}{1-\gamma}\left|\int_{S}\int_A \left(c(s',a)+\tau \log \frac{\mathrm d\pi_{\theta}}{\mathrm d\mu}(a|s')\right)\pi_{\theta}(\mathrm{d}a|s')d^{\pi_\theta}(\mathrm{d}s'|s)\right|\\
&\le \frac{1}{1-\gamma} \left(1+\tau \bigg|\log \frac{\mathrm d\pi_{\theta}}{\mathrm d \mu}\bigg|_{B_b(S \times A)} \right)\,.
\end{aligned}
\end{equation}
To estimate the state-action value function, by Eq.~\eqref{def:intro:Qpi}, we have
\begin{align*}
|Q^{\pi_{\theta}}_{\tau}(s,a)| &\le \frac{1}{1-\gamma} \left(1+\gamma\tau \bigg|\log \frac{\mathrm d\pi_{\theta}}{\mathrm d \mu}\bigg|_{B_b(S \times A)} \right)\,.
\end{align*}
\end{proof}

\begin{definition}[Fisher information matrix]\label{def:app:mdp:FIM}
For fixed $s\in S$ and $\theta\in\mathbb R^p$, the Fisher information matrix (FIM) at state $s$ is
\begin{equation}
\label{eq:FIM}
\begin{aligned}
G^{\pi_\theta}(s)
&:=\int_A
\left(\nabla\log\frac{\mathrm d\pi_\theta}{\mathrm d\mu}(a| s)\right) \left(\nabla\log\frac{\mathrm d\pi_\theta}{\mathrm d\mu}(a| s)\right)^\top \pi_\theta(\mathrm da| s)\\
&=\int_A \bigl(g(s,a)-\int_A g(s,a')\pi_\theta(\mathrm da'| s)\bigr) \bigl(g(s,a)-\int_A g(s,a')\pi_\theta(\mathrm da'| s)\bigr)^\top \pi_\theta(\mathrm da| s).
\end{aligned}
\end{equation}
For $\rho\in\mathcal P(S)$, define the occupancy-weighted FIM by $G^{\pi_\theta}(\rho) :=\int_SG^{\pi_\theta}(s)d_\rho^{\pi_\theta}(\mathrm ds).$
In particular, for arbitrary $\pi \in \Pi_\mu$, $G^\pi(s) = \operatorname{Cov}_{a\sim\pi(\cdot|s)}(g(s,a))$.
\end{definition}

In the remainder of the paper, $\nabla$ means $\nabla_\theta$ unless stated otherwise.

\begin{lemma}[Gradient of log policy]\label{lem:app:mdp:grad_of_log_policy}
For any $\theta \in \mathbb{R}^p$, $s \in S$, $a \in A$
\begin{equation}\label{eq:app:mdp:grad_of_log_policy}
\nabla \log\frac{\mathrm d \pi_{\theta}}{\mathrm d \mu}(a|s) = g(s,a) - \int_A g(s, a^\prime) \pi_{\theta}( \mathrm{d}a^\prime | s). 
\end{equation}
\end{lemma}
\begin{proof}
For any  $s \in S$, $a \in A$, $\log\frac{\mathrm d \pi_{\theta}}{\mathrm d \mu}(a|s) = \langle\theta,g(s,a)\rangle - \log \int_A \exp\left(\langle\theta,g(s,a)\rangle\right)\mu(\mathrm{d}a),$ then taking gradient concludes the proof.
\end{proof}

\begin{lemma}\label{lem:app:mdp:bounded_grad_of_log_policy}
For any $\theta \in \mathbb{R}^p$ and any $s \in S$ and $a \in A$,
\begin{equation}
\left\|\nabla \log\frac{\mathrm d \pi_{\theta}}{\mathrm d \mu} (a|s)\right\|_2 \leq 2.
\end{equation}
\end{lemma}

\begin{proof}
By Lemma~\ref{lem:app:mdp:grad_of_log_policy}, $\|g(s,a)-\int_A g(s,a')\pi_\theta(\mathrm da'|s)\|_2\le\|g(s,a)\|_2+\int_A\|g(s,a')\|_2\pi_\theta(\mathrm da'|s)\le2$.
\end{proof}

\begin{lemma}[Lipschitz continuity of the gradient of the log policy]
\label{lem: smooth of log policy}
For any $\theta,\theta'\in\mathbb R^p$, $s\in S$, and $a\in A$,
\begin{equation}
\label{eq: smooth of log policy}
\left\|\nabla\log\frac{\mathrm d\pi_{\theta'}}{\mathrm d\mu}(a|s)-\nabla\log\frac{\mathrm d\pi_\theta}{\mathrm d\mu}(a|s)\right\|_2\le \|\theta'-\theta\|_2.
\end{equation}
\end{lemma}

\begin{proof}
Fix $\theta,\theta'\in\mathbb R^p$, $s\in S$, and $a\in A$,
and put $\theta^\varepsilon=(1-\varepsilon)\theta+\varepsilon\theta'$
for $\varepsilon\in[0,1]$.
By Lemma~\ref{lem:app:mdp:grad_of_log_policy},
\begin{align*}
&\nabla\log\frac{\mathrm d\pi_{\theta'}}{\mathrm d\mu}(a|s)-\nabla\log\frac{\mathrm d\pi_\theta}{\mathrm d\mu}(a|s)=-\int_A g(s,a')
(\pi_{\theta'}-\pi_\theta)(\mathrm da'|s).
\end{align*}
Differentiating the feature mean along the segment gives
\begin{align*}
\frac{\mathrm d}{\mathrm d\varepsilon}\int_A g(s,a')\pi_{\theta^\varepsilon}(\mathrm da'|s)
&=\Bigg[\int_A g(s,a')\Bigg(g(s,a')-\int_A g(s,a'')\pi_{\theta^\varepsilon}(\mathrm da''|s)
\Bigg)^\top\pi_{\theta^\varepsilon}(\mathrm da'|s)\Bigg]
(\theta'-\theta)\\
&=G^{\pi_{\theta^\varepsilon}}(s)(\theta'-\theta).
\end{align*}
The last equality follows from \eqref{eq:FIM}, since the centered
feature has zero mean under $\pi_{\theta^\varepsilon}(\cdot|s)$.
Differentiation under the integral is justified by boundedness of $g$: the softmax density and its derivative with respect to $\varepsilon$ are uniformly bounded on this parameter segment.
Therefore, the fundamental theorem of calculus yields
\begin{align*}
&\nabla\log\frac{\mathrm d\pi_{\theta'}}{\mathrm d\mu}(a|s)-\nabla\log\frac{\mathrm d\pi_\theta}{\mathrm d\mu}(a|s)=-\int_0^1G^{\pi_{\theta^\varepsilon}}(s)(\theta'-\theta)\,\mathrm d\varepsilon.
\end{align*}

For every $v\in\mathbb R^p$ and $\varepsilon\in[0,1]$,
\eqref{eq:FIM} and the feature bound
$\sup_{s,a}\|g(s,a)\|_2\le1$ imply
\begin{align*}
0\le v^\top G^{\pi_{\theta^\varepsilon}}(s)v&=\int_A\langle v,g(s,a')\rangle^2\pi_{\theta^\varepsilon}(\mathrm da'|s)-\left(\int_A\langle v,g(s,a')\rangle\pi_{\theta^\varepsilon}(\mathrm da'|s)\right)^2\\
&\le\int_A\|v\|_2^2\|g(s,a')\|_2^2\pi_{\theta^\varepsilon}(\mathrm da'|s)\le\|v\|_2^2.
\end{align*}
Thus $0\preceq G^{\pi_{\theta^\varepsilon}}(s)\preceq I_p$.
Since this matrix is symmetric, it follows that
\begin{align*}
&\left\|\nabla\log\frac{\mathrm d\pi_{\theta'}}{\mathrm d\mu}(a|s)-\nabla\log\frac{\mathrm d\pi_\theta}{\mathrm d\mu}(a|s)
\right\|_2\le\int_0^1\left\|G^{\pi_{\theta^\varepsilon}}(s)(\theta'-\theta)\right\|_2\,\mathrm d\varepsilon\le\|\theta'-\theta\|_2.
\end{align*}
\end{proof}

We next recall the Lipschitz continuity of the occupancy kernel with respect to the policy (see \citealp[Lemma~A.5]{leahy2022convergence}).
\begin{lemma}\label{lem:app:mdp:lip_occ_kernel}
For given $\pi,\pi' \in \mathcal{P}(A|S)$, we have
\[
|d^{\pi'}-d^{\pi}|_{b\mathcal{K}(S|S)}\le \frac{\gamma}{1-\gamma}|\pi'-\pi|_{b\mathcal{K}(A|S)}\,.
\]
\end{lemma}

\begin{corollary}[Lipschitz continuity of the occupancy measure in the parameter]\label{cor:app:mdp:Lip_occ}
For given $\theta,\theta' \in \mathbb{R}^p$, we have
\[
|d^{\pi_{\theta'}}-d^{\pi_{\theta}}|_{b\mathcal{K}(S|S)}\le \frac{2\gamma}{1-\gamma}\|\theta' - \theta\|_2\,.
\]
\end{corollary}
\begin{proof}
By Lemma \ref{lem:app:mdp:lip_occ_kernel}, it is enough to show that
\begin{equation}\label{ineq:Lip_pi}
|\pi_{\theta'}-\pi_{\theta}|_{b\mathcal{K}(A|S)} \le 2\|\theta' - \theta\|_2\,.
\end{equation}
We have 
\begin{align*}
|\pi_{\theta'}-\pi_{\theta}|_{b\mathcal{K}(A|S)}&=\sup_{s\in S}|\pi_{\theta'}(\cdot|s)-\pi_{\theta}(\cdot|s)|_{\mathcal{M}(A)}=\sup_{s\in S}\underset{|h|_{B_b(A)}\le 1}{\sup_{h\in B_b(A)}}\int_{A}h(a)\left(\pi_{\theta'}-\pi_{\theta}\right)(\mathrm{d}a|s)\,.
\end{align*}
Let $\theta^{\varepsilon}=\varepsilon \theta'+(1-\varepsilon)\theta$, $\varepsilon\in [0,1]$. Let $s\in S$ and $h\in B_b(A)$ be arbitrarily given. 
Using Lemma \ref{lem:app:mdp:grad_of_log_policy}, we find that 
\[
\int_{A}h(a)\left(\pi_{\theta'}-\pi_{\theta}\right)(\mathrm{d}a|s)= \left\langle\int_0^1 \int_{A} h(a)\nabla \log \frac{\mathrm d \pi_{\theta^\varepsilon}}{\mathrm d \mu} (a|s) \pi_{\theta^\varepsilon}(\mathrm{d}a|s)\mathrm d\varepsilon, \theta^\prime - \theta\right\rangle,
\]
which yields (\ref{ineq:Lip_pi}) by Lemma \ref{lem:app:mdp:bounded_grad_of_log_policy} and concludes the proof.
\end{proof}

For a related policy-gradient formula in the mean-field setting, see \citet[Lemma~2.3]{leahy2022convergence}.
The proof of Proposition~\ref{prp:app:mdp:Gradient_of_the_Objective} is given in Section~\ref{sec:add_proof}.
It verifies that the chain rule holds for $\pi \mapsto V^\pi_\tau$ and $\theta \mapsto \pi_\theta$.

\begin{proposition}[Gradient of the Objective]\label{prp:app:mdp:Gradient_of_the_Objective} 
For any $\theta \in \mathbb{R}^p$, 
\begin{equation}\label{eq:wts_func_deriv_V}
\nabla V_\tau^{\pi_\theta}(\rho) = \frac{1}{1-\gamma}\int_S\int_A \left(Q^{\pi_\theta}_\tau(s,a) + \tau \log \frac{\mathrm d \pi_{\theta}}{\mathrm d \mu}(a|s)\right)\nabla \log \frac{\mathrm d\pi_\theta}{\mathrm d\mu}(a| s)\pi_{\theta}(\mathrm{d}a|s)d_\rho^{\pi_{\theta}}(\mathrm{d}s).
\end{equation}
\end{proposition}

\begin{proposition}[Local Lipschitz continuity of
$\theta\mapsto\nabla V_\tau^{\pi_\theta}(\rho)$]
\label{prp:app:mdp:local_lipchitiz_conitinuity_mdp}
Let $R>0$, and let $\theta,\theta'\in\mathbb R^p$ satisfy
\[
\left|\log\frac{\mathrm d\pi_\theta}{\mathrm d\mu}
\right|_{B_b(S\times A)}\le R,\qquad\left|\log\frac{\mathrm d\pi_{\theta'}}{\mathrm d\mu}\right|_{B_b(S\times A)}\le R.
\]
Then
\[
\left\|\nabla V_\tau^{\pi_\theta}(\rho)-\nabla V_\tau^{\pi_{\theta'}}(\rho)\right\|_2\le C_{\gamma,\tau,R}\|\theta-\theta'\|_2,
\]
where
\[
C_{\gamma,\tau,R} := \frac{(5+3\gamma)(1+\tau R)}{(1-\gamma)^3}+\frac{\tau}{1-\gamma}.
\]
In particular, $\theta\mapsto\nabla V_\tau^{\pi_\theta}(\rho)$
is locally Lipschitz continuous on $\mathbb R^p$.
\end{proposition}
The proof of Proposition~\ref{prp:app:mdp:local_lipchitiz_conitinuity_mdp} can be found in Section~\ref{sec:proof_local_lip}.

\subsection{Proof of the main results}
\label{sec:proof}

\subsubsection{Proof of Theorem \ref{thm:main: non-uniform_PL_ineq}}
\label{sec:pf_non_uniform_PL}
Let us restate Theorem \ref{thm:main: non-uniform_PL_ineq} in its full version.

\textbf{Theorem \ref{thm:main: non-uniform_PL_ineq}.}
Let Assumption~\ref{asp:app:mdp:Q_realisability} hold.
Let $R>0$.
Let
\[
\Theta_R = \left\{\theta: \pi_\theta \in \Pi_\mu \ \text{and}\  \left\vert \log \frac{\mathrm d \pi_\theta}{\mathrm d\mu} \right\vert_{B_b(S\times A)} \leq R \right\}\,.
\]
Then for any $\theta \in \Theta_R$, there exists $C_R(\theta)>0$ such that
\begin{equation}
\label{eq:pl_theorem_appendix}
0 \leq V_\tau^{\pi_\theta}(\rho) - V_\tau^{\pi^*}(\rho) \leq C_R(\theta)\left\|\nabla_\theta V_\tau^{\pi_\theta}(\rho)\right\|^2_2,
\end{equation}
where 
\begin{equation}
C_R(\theta) = \frac{ \exp\left({\max\left(\frac{2}{(1-\gamma)\tau} \left(1+\gamma\tau R\right), R\right)}\right)}{2(1-\gamma)\tau\lambda^2_{\text{fsa}}(\theta)} \, , \ \label{eq:non_uni_PL_cons_affine}
\end{equation}
and
\begin{equation}
\label{eq:def_lambda_theta_fsa}
\lambda_{\text{fsa}}(\theta) := \int_S \lambda_{\min}(G^{\pi_{{\theta}}}(s))\rho(\mathrm{d}s).
\end{equation}
Moreover, if the feature basis $g:S\times A \to \mathbb R^p$ takes values in $\Delta_{p-1}$ for every $s$ and $a$, i.e., $\sum_{i=1}^p g_i(s,a) = 1 $ and for all $s \in S,a \in A$, $g(s,a) \geq 0$, then we may replace $\lambda_{\text{fsa}}(\theta)$ in~\eqref{eq:pl_theorem_appendix} with
\begin{equation}
\label{eq:def_lambda_theta_simplex}
\lambda_{\text{smp}}(\theta) := \int_{S}\lambda_{\min}\left(\int_{A} g(s,a)g^\top(s,a)\pi_\theta(\mathrm{d}a|s)\right) \rho(\mathrm{d}s)\,.    
\end{equation}
We adopt the convention $C_R(\theta) = \infty$ if $\lambda_{\text{fsa}}(\theta)=0$ or $\lambda_{\text{smp}} (\theta)= 0$ using simplex features.

\begin{proof}
Define proximal policy ${\pi}_{\pi_\theta}(\mathrm{d}a|s) \propto \exp\left(-\frac{1}{\tau}Q^{\pi_\theta}_\tau(s,a) \right)\mu(\mathrm{d}a)$ as~\eqref{eq:app:mdp:proximal_policy} where by Lemma \ref{lem:bounds}, we have
\[
|Q^{\pi_{\theta}}_{\tau}(s,a)| \le \tfrac{1}{1-\gamma} \left(1+\gamma\tau R\right),
\]
then we have
\[
\big|\log \frac{\mathrm d{\pi}_{\pi_{\theta}}}{\mathrm d\mu}(a|s)\big| \leq \tfrac{2}{(1-\gamma)\tau} \left(1+\gamma\tau R\right).
\]

Also, from Assumption \ref{asp:app:mdp:Q_realisability}, we know there exist $b^{\pi_\theta}\in B_b(S)$ and $\boldsymbol{\theta}(\pi_\theta)$ such that $\langle \boldsymbol{\theta}(\pi_\theta), g(s,a)\rangle = -\frac{1}{\tau} Q^{\pi_\theta}_\tau(s,a) - b^{\pi_\theta}(s)$. 

Hence by Lemma \ref{lem:sandwich-inequality} and Lemma \ref{lem:app:kl-logit_ineq_affine}, we have
\begin{align}
\label{eq:sandwich_ineq}
V_\tau^{\pi_\theta}(\rho) - V_\tau^{\pi^*}(\rho) 
\leq \frac{\tau}{1-\gamma}\int_S \operatorname{KL}(\pi_\theta (\cdot |s) | {\pi}_{\pi_\theta}(\cdot |s))d_\rho^{\pi^*}(\mathrm{d}s) 
& \leq \frac{\tau e^{\hat{R}} }{2(1-\gamma) }\left\|{\theta} - {\boldsymbol{\theta}(\pi_\theta)} \right\|_2^2\, , 
\end{align}
where $\hat{R} = \max\left(\frac{2}{(1-\gamma)\tau} \left(1+\gamma\tau R\right), R\right)$.

On the other hand, by Proposition \ref{prp:app:mdp:Gradient_of_the_Objective},
\begin{align*}
& \nabla_\theta V_\tau^{\pi_\theta}(\rho) \nonumber \\
& = \frac{1}{1-\gamma}\int_S\int_A \left(Q^{\pi_\theta}_\tau(s,a) + \tau \log \frac{\mathrm d \pi_{\theta}}{\mathrm d \mu}(a|s)\right)\nabla \log \frac{\mathrm d\pi_\theta}{\mathrm d\mu}(a| s)\pi_{\theta}(\mathrm{d}a|s)d_{\rho}^{\pi_{\theta}}(\mathrm d s) \nonumber \\
& = \frac{1}{1-\gamma}\int_S\int_A \left(Q^{\pi_\theta}_\tau(s,a) + \tau b^{\pi_\theta}(s)+ \tau \log \frac{\mathrm d \pi_{\theta}}{\mathrm d \mu}(a|s)\right)\nabla \log \frac{\mathrm d\pi_\theta}{\mathrm d\mu}(a| s)\pi_{\theta}(\mathrm{d}a|s)d_{\rho}^{\pi_{\theta}}(\mathrm d s) \nonumber \\
& = \frac{\tau}{1-\gamma}\int_S\int_A \left \langle \theta - \boldsymbol{\theta}(\pi_\theta),g(s,a)\right  \rangle \nabla \log \frac{\mathrm d\pi_\theta}{\mathrm d\mu}(a| s)\pi_{\theta}(\mathrm{d}a|s)d_{\rho}^{\pi_{\theta}}(\mathrm d s) \nonumber \\
& = \frac{\tau}{1-\gamma}\int_S\int_A \left \langle \theta - \boldsymbol{\theta}(\pi_\theta),g(s,a)- \int g(s,a^\prime) \pi_\theta(\mathrm{d}a^\prime |s)\right  \rangle  \nabla \log \frac{\mathrm d\pi_\theta}{\mathrm d\mu}(a| s)\pi_{\theta}(\mathrm{d}a|s)d_{\rho}^{\pi_{\theta}}(\mathrm d s) \nonumber \\
& = \frac{\tau}{1-\gamma}\int_S\int_A \left \langle \theta - \boldsymbol{\theta}(\pi_\theta),\nabla \log \frac{\mathrm d\pi_\theta}{\mathrm d\mu}(a| s)\right  \rangle \nabla \log \frac{\mathrm d\pi_\theta}{\mathrm d\mu}(a| s)\pi_{\theta}(\mathrm{d}a|s)d_{\rho}^{\pi_{\theta}}(\mathrm d s)\,, \nonumber \\
\end{align*}
where the second, third and fourth equalities follow from Lemma \ref{lem:app:mdp:grad_of_log_policy} and the fact that $\int_A (g(s,a) - \int_A g(s,a^\prime)\pi_\theta(\mathrm{d}a^\prime | s))\pi_\theta(\mathrm{d}a |s) = \mathbf 0$. 
Hence
\begin{equation}
\label{eq:grad_for_PL}
\nabla_\theta V_\tau^{\pi_\theta}(\rho) = \frac{\tau}{1-\gamma} \int_{S} G^{\pi_\theta}(s)d_{\rho}^{\pi_{\theta}}(\mathrm{d}s)\cdot(\theta -\boldsymbol{\theta}(\pi_\theta))\, .      
\end{equation}
Left multiplying the identity~\eqref{eq:grad_for_PL} by $(\theta-\boldsymbol{\theta}(\pi_\theta))^\top$, we have
\begin{equation}
\label{eq:prod_of_grad_for_PL}
(\theta-\boldsymbol{\theta}(\pi_\theta))^\top\nabla_\theta V_\tau^{\pi_\theta}(\rho) = \frac{\tau}{1-\gamma} (\theta-\boldsymbol{\theta}(\pi_\theta))^\top\int_{S} G^{\pi_\theta}(s)d_{\rho}^{\pi_{\theta}}(\mathrm{d}s)\cdot(\theta -\boldsymbol{\theta}(\pi_\theta))\, .      
\end{equation}

Moreover, note that for any $E \in \mathcal{B}(S)$ and $\pi \in \mathcal{P}(A | S)$, we have
\begin{equation}
\label{eq:kernel_occ_geq_rho}
\begin{split}
d_\rho^\pi(E) & =\int_S d^\pi\left(E | s^{\prime}\right) \rho\left(\mathrm{d}s^{\prime}\right)=\int_S(1-\gamma) \sum_{n=0}^{\infty} \gamma^n P_\pi^n\left(E | s^{\prime}\right) \rho\left(\mathrm{d}s^{\prime}\right) \\
& \geq \int_S(1-\gamma) P_\pi^0\left(E | s^{\prime}\right) \rho\left(\mathrm{d}s^{\prime}\right)=(1-\gamma) \int_S \delta_{s^{\prime}}(E) \rho\left(\mathrm{d}s^{\prime}\right)=(1-\gamma) \rho(E)\,. 
\end{split}
\end{equation}

This implies that $\rho \ll  d^{\pi}_\rho$ and $\frac{\mathrm d \rho}{\mathrm d d^{\pi}_\rho}$ exists.
By contradiction we conclude that $\frac{\mathrm d \rho}{\mathrm d d^{\pi}_\rho} \le \frac{1}{1-\gamma}$ for $d^{\pi}_\rho $ a.s. 
Indeed, suppose that there exists $E\in \mathcal{B}(S)$ such that $d^{\pi}_\rho (E)>0$ and $\frac{\rm d \rho}{\rm d d^{\pi}_\rho}> \tfrac{1}{1-\gamma} $ on $E$.
Then 
\[
\rho (E)=\int_E \frac{\rm d \rho}{\rm d d^{\pi}_\rho} (s) d^{\pi}_\rho(ds) >\frac{1}{1-\gamma }d^{\pi}_\rho (E)\,.
\]
But that is $d^{\pi}_\rho (E) < (1-\gamma)\rho(E)$ which contradicts~\eqref{eq:kernel_occ_geq_rho}. 

Using that the FIM is positive semi-definite, back to~\eqref{eq:prod_of_grad_for_PL} we have
\[
\begin{split}
 ( {\theta} -  {\boldsymbol{\theta}(\pi_\theta)})^\top\nabla_{\theta} V_\tau^{\pi_{ {\theta}}}(\rho) & \geq \frac{\tau}{1-\gamma}\left(\operatorname*{ess\,sup}_{s \in S}\frac{\mathrm d \rho}{\mathrm d d^{\pi_{\theta}}_{\rho}}(s)\right)^{-1}\int_S ( {\theta} -  {\boldsymbol{\theta}(\pi_\theta)})^\top G^{\pi_{ {\theta}}}(s) ( {\theta} -  {\boldsymbol{\theta}(\pi_\theta)})\rho(\mathrm{d}s) \\
& \geq \frac{\tau}{1-\gamma} \left(\operatorname*{ess\,sup}_{s \in S}\frac{\mathrm d \rho}{\mathrm d d^{\pi_{\theta}}_{\rho}}(s)\right)^{-1}\int_S  \lambda_{\min}(G^{\pi_{ {\theta}}}(s))  \rho(\mathrm{d}s) \| {\theta} -  {\boldsymbol{\theta}(\pi_\theta)}\|_2^2    \,.
\end{split}
\]
We note, using the Cauchy--Schwarz inequality, that 
\[
\| {\theta} -  {\boldsymbol{\theta}(\pi_\theta)}\|_2 \|\nabla_{\theta} V^{\pi_{ {\theta}}}_\tau(\rho)\|_2 \geq ( {\theta} -  {\boldsymbol{\theta}(\pi_\theta)})^\top \nabla V_\tau^{\pi_{ {\theta}}}(\rho)
\]
and so, dividing by $\| {\theta} -  {\boldsymbol{\theta}(\pi_\theta)}\|_2$ if $\| {\theta} -  {\boldsymbol{\theta}(\pi_\theta)}\|_2 > 0$ (otherwise the claim is trivial), we have
\begin{equation}
\label{eq:123xyz}
\left\|\nabla_{\theta} V_\tau^{\pi_{ {\theta}}}(\rho)\right\|_2 \geq \frac{\tau}{1-\gamma} \left(\operatorname*{ess\,sup}_{s \in S}\frac{\mathrm d \rho}{\mathrm d d^{\pi_{\theta}}_{\rho}}(s)\right)^{-1}\int_S  \lambda_{\min}(G^{\pi_{ {\theta}}}(s)) \rho(\mathrm ds) \| {\theta} -  {\boldsymbol{\theta}(\pi_\theta)}\|_2.
\end{equation}

From~\eqref{eq:kernel_occ_geq_rho} and~\eqref{eq:123xyz} we thus get
\[
\left\|\nabla_{\theta} V_\tau^{\pi_{ {\theta}}}(\rho)\right\|_2 \geq \tau \int_S  \lambda_{\min}(G^{\pi_{ {\theta}}}(s)) \rho(\mathrm{d}s) \| {\theta} -  {\boldsymbol{\theta}(\pi_\theta)}\|_2.
\]
Thus
\[
\| {\theta} -  {\boldsymbol{\theta}(\pi_\theta)}\|_2 \leq \tau^{-1}\bigg(\int_S  \lambda_{\min}(G^{\pi_{ {\theta}}}(s)) \rho(\mathrm{d}s)\bigg)^{-1}\left\|\nabla_{\theta} V_\tau^{\pi_{ {\theta}}}(\rho)\right\|_2 .
\]
Due to~\eqref{eq:sandwich_ineq} and this we get
\begin{align*}
V_\tau^{\pi_\theta}(\rho) - V_\tau^{\pi^*}(\rho)
& \leq \frac{\tau e^{\hat{R}} }{2(1-\gamma)}\| {\theta} -  {\boldsymbol{\theta}(\pi_\theta)}\|_2^2\\
& \leq \frac{e^{\hat{R}}}{2(1-\gamma)\tau }\left(\int_{S} \lambda_{\min}(G^{\pi_{ {\theta}}}(s))\rho(\mathrm{d}s)\right)^{-2} \left\|\nabla_{\theta} V_\tau^{\pi_{ {\theta}}}(\rho)\right\|_2^2. 
\end{align*}
Thus we have shown that~\eqref{eq:pl_theorem_appendix} holds.

If, additionally,  we know that the feature basis $g:S\times A \to \mathbb R^p$ satisfies $\sum_{i=1}^p g_i(s,a) = 1 $ and for all $s \in S,a \in A$, $g(s,a) \geq 0$ then we return to~\eqref{eq:sandwich_ineq} to proceed.
By Lemma \ref{lem:app:kl-logit_ineq}, we have
\begin{align*}
V_\tau^{\pi_\theta}(\rho) - V_\tau^{\pi^*}(\rho) & \leq \frac{\tau}{1-\gamma}\int \operatorname{KL}(\pi_\theta(\cdot|s)\,|\,\pi_{\pi_\theta}(\cdot|s))d_\rho^{\pi^*}(\mathrm{d}s) \\
& \leq \frac{\tau e^{\hat{R}}}{2(1-\gamma)}\left\|\theta - \boldsymbol{\theta}(\pi_\theta) - \frac{(\theta - \boldsymbol{\theta}(\pi_\theta))^\top\mathbf 1}{p} \cdot\mathbf 1\right\|_{2}^2,
\end{align*}
where $\hat{R} = \max\left(\frac{2}{(1-\gamma)\tau} \left(1+\gamma\tau R\right), R\right)$.
On the other hand, by~\eqref{eq:grad_for_PL} and Proposition~\ref{prp:app:mdp:singularity_of_fim_mdp_simplex},
\begin{align}
 \nabla_\theta V_\tau^{\pi_\theta}(\rho)
& = \frac{\tau}{1-\gamma} \int_{S} G^{\pi_\theta}(s)d_{\rho}^{\pi_{\theta}}(\mathrm{d}s)\cdot(\theta -\boldsymbol{\theta}(\pi_\theta)) \nonumber\\
& = \frac{\tau}{1-\gamma} \int_{S} G^{\pi_\theta}(s)d_{\rho}^{\pi_{\theta}}(\mathrm{d}s)\cdot\left(\theta -\boldsymbol{\theta}(\pi_\theta)- \frac{(\theta - \boldsymbol{\theta}(\pi_\theta))^\top\textbf{1}}{p}\textbf{1}\right).
\label{eq:grad_for_PL_smp}
\end{align}
Left multiplying the identity~\eqref{eq:grad_for_PL_smp} by $\left(\theta -\boldsymbol{\theta}(\pi_\theta)- \frac{(\theta - \boldsymbol{\theta}(\pi_\theta))^\top\textbf{1}}{p}\textbf{1}\right)^\top$, we have
\begin{align*}
&\left(\theta -\boldsymbol{\theta}(\pi_\theta)- \frac{(\theta - \boldsymbol{\theta}(\pi_\theta))^\top\textbf{1}}{p}\textbf{1}\right)^\top \nabla_\theta V_\tau^{\pi_\theta}(\rho)\\ & =
\frac{\tau}{1-\gamma}\left(\theta -\boldsymbol{\theta}(\pi_\theta)- \frac{(\theta - \boldsymbol{\theta}(\pi_\theta))^\top\textbf{1}}{p}\textbf{1}\right)^\top \int_{S} G^{\pi_\theta}(s)d_{\rho}^{\pi_{\theta}}(\mathrm{d}s)\cdot\left(\theta -\boldsymbol{\theta}(\pi_\theta)- \frac{(\theta - \boldsymbol{\theta}(\pi_\theta))^\top\textbf{1}}{p}\textbf{1}\right)
\end{align*}
We note, using Cauchy--Schwarz inequality and Corollary \ref{cor:app:mdp:spectrum_of_fim} , that
\begin{align*}
& \left\|\theta - \boldsymbol{\theta}(\pi_\theta) - \frac{(\theta - \boldsymbol{\theta}(\pi_\theta))^\top\textbf{1}}{p}\textbf{1}\right\|_2\left\|\nabla_\theta V_\tau^{\pi_\theta}(\rho)\right\|_2 \\
& \geq  \frac{\tau }{1-\gamma}\int_{S}\lambda_{\min}\left(\int_{A} g(s,a)g^\top(s,a)\pi_\theta(\mathrm{d}a|s)\right)d^{\pi_{\theta}}_{\rho}(\mathrm{d}s)\left\|\theta - \boldsymbol{\theta}(\pi_\theta) - \frac{(\theta - \boldsymbol{\theta}(\pi_\theta))^\top\textbf{1}}{p}\textbf{1}\right\|_2^2
\end{align*}
and so, dividing by $\left\|\theta - \boldsymbol{\theta}(\pi_\theta) - \frac{(\theta - \boldsymbol{\theta}(\pi_\theta))^\top\textbf{1}}{p}\textbf{1}\right\|_2$ if it is strictly positive (otherwise the claim is trivial), we have
{\small
\begin{align*}
& \left\|\nabla_\theta V_\tau^{\pi_\theta}(\rho)\right\|_2 \\
& \geq \frac{\tau }{1-\gamma}\int_{S}\lambda_{\min}\left(\int_{A} g(s,a)g^\top(s,a)\pi_\theta(\mathrm{d}a|s)\right)d^{\pi_{\theta}}_{\rho}(\mathrm{d}s)\left\|\theta - \boldsymbol{\theta}(\pi_\theta) - \frac{(\theta - \boldsymbol{\theta}(\pi_\theta))^\top\textbf{1}}{p}\textbf{1}\right\|_2 \\
& \geq \frac{\tau }{1-\gamma} \left(\operatorname*{ess\,sup}_{s\in S} \frac{\mathrm d \rho}{\mathrm d d_\rho^{\pi_\theta}}(s) \right)^{-1}\int_{S}\lambda_{\min}\left(\int_{A} g(s,a)g^\top(s,a)\pi_\theta(\mathrm{d}a|s)\right)\rho(\mathrm{d}s)\left\|\theta - \boldsymbol{\theta}(\pi_\theta) - \frac{(\theta - \boldsymbol{\theta}(\pi_\theta))^\top\textbf{1}}{p}\textbf{1}\right\|_2 \\
& \geq \tau  \int_{S}\lambda_{\min}\left(\int_{A} g(s,a)g^\top(s,a)\pi_\theta(\mathrm{d}a|s)\right)\rho(\mathrm{d}s)\left\|\theta - \boldsymbol{\theta}(\pi_\theta) - \frac{(\theta - \boldsymbol{\theta}(\pi_\theta))^\top\textbf{1}}{p}\textbf{1}\right\|_2,
\end{align*}
where the last inequality is by~\eqref{eq:kernel_occ_geq_rho}. 
Thus we have
\begin{align*}
 V_\tau^{\pi_\theta}(\rho) - V_\tau^{\pi^*}(\rho)
& \leq \frac{\tau e^{\hat{R}} }{2(1-\gamma) }\left\|\theta - \boldsymbol{\theta}(\pi_\theta) - \frac{(\theta - \boldsymbol{\theta}(\pi_\theta))^\top\textbf{1}}{p}\textbf{1}\right\|_{2}^2\\
& \leq \frac{e^{\hat{R}}}{2(1-\gamma)\tau \lambda^2_\text
{\text{smp}}(\theta)} \left\|\nabla V_\tau^{\pi_\theta}(\rho)\right\|_2^2 \, .
\end{align*}
This concludes the second situation.
}
\end{proof}

\subsubsection{Proof of Theorem \ref{thm:main:mdp:Radial_unboundedness_of_relative_entropy}}
\label{sec:pf_Radial_unboundedness_of_relative_entropy}

\aspAffineMaxSet*

\begin{lemma}[Full affine span at a fixed state]\label{rem:app:mdp:affine_span}
For any fixed $s\in S$, the following statements are equivalent:
\begin{enumerate}
\item For any $v \in \mathbb R^{p}$ such that $\|v\|_2 = 1$ the map $A \ni a \mapsto v^\top {g(s,a)}$ is not a constant.
\item  ${g(s,A)}$ has full affine dimension, i.e., $\operatorname{span}\{{g(s,a)}-{g(s,a')}:a,a'\in A\}=\mathbb R^{p}$.
\end{enumerate}
\end{lemma}

\begin{proof}
To see this, let us first show that \textup{1.} implies \textup{2.}
We proceed by contradiction.
If the span in \textup{2.} was a proper subspace $V_s \subset \mathbb R^{p}$, then $V_s^\perp$ contains a non-zero vector. 
Hence we can choose some $m\in V_s^\perp\setminus\{\mathbf 0\}$ such that
\[
m^\top({g(s,a)} - {g(s,a')})=0,\qquad \forall a,a'\in A.
\]
But then $a\mapsto (m/\|m\|_2)^\top {g(s,a)} = (m/\|m\|_2)^\top {g(s,a')}$ which is a constant thus contradicting \textup{1.}

Now we show \textup{2.} implies \textup{1.}
Again we proceed by contradiction.
If $u^\top {g(s,\cdot)}$ is constant for some $u\in \mathbb R^{p}$ such that $\|u\|_2 = 1$ then $u^\top({g(s,a)} - {g(s,a')})=0$ for all $a,a'\in A$.
So $u\in V_s^\perp=\{\mathbf 0\}$, a contradiction. 
\end{proof}

\KLUBAFFINE*
\begin{proof}
Fix $s$ and apply Lemma~\ref{lem:app:uniform-radial-kl} on $E=\mathbb R^{p}$ with $h(a)={g(s,a)}$. 
Assumption~\ref{asp:app:mdp:max_set_affine} supplies exactly the required zero-measure maximizer condition.
Hence the KL divergence tends to infinity whenever $\|\theta\|_2\to\infty$.
\end{proof}

\subsubsection{Proof of Lemma \ref{lem:main:mdp:exist_sol_gf_affine}}
\label{sec:pf_exist_sol_gf_affine}

\ExistSolAffine*

\begin{proof}
Write $Z_{\pi_\theta}(s) := \int_A \exp(\theta^\top g(s,a))\mu(\mathrm d a)$.
Define
\[
\mathcal K_{\theta_0}:=\left\{\theta\in\mathbb R^p:V_\tau^{\pi_\theta}(\rho)\leq V_\tau^{\pi_{\theta_0}}(\rho)\right\}.
\]

Since the vector field is locally Lipschitz by Proposition~\ref{prp:app:mdp:local_lipchitiz_conitinuity_mdp}, we know from Picard--Lindel\"of theorem there exists a unique solution on $[0, T^*)$ with $T^* > 0$.
Note that for $t \in [0, T^*)$,
\[
\frac{\mathrm d}{\mathrm dt}V_\tau^{\pi_{\theta_t}}(\rho)=-\left\|\nabla V_\tau^{\pi_{\theta_t}}(\rho)\right\|_2^2 \leq 0.
\]
The non-increasing objective implies that $\theta_t\in\mathcal K_{\theta_0}$ throughout the maximal interval of existence.

We show that $\mathcal K_{\theta_0}$ is compact.
It is closed because $\theta\mapsto V_\tau^{\pi_\theta}(\rho)$ is continuous by Proposition~\ref{prp:app:mdp:Gradient_of_the_Objective}.
If it were unbounded, there would exist $\{\theta_n\}_n \subset \mathcal K_{\theta_0}$ with $\|\theta_n\|_2\to\infty$ as $n\to \infty$.
By~\eqref{eq:intro:value_function_occ}, the bound on cost function and~\eqref{eq:kernel_occ_geq_rho} we have
\begin{align*}
V^{\pi_\theta}_\tau(\rho)\ & \ge\ -\frac{1}{1-\gamma}+\tau\int_S \operatorname{KL}(\pi_\theta(\cdot|s)|\mu)\,\rho(\mathrm ds).
\end{align*}
Theorem~\ref{thm:main:mdp:Radial_unboundedness_of_relative_entropy} gives, for every $s\in S$,
\[
\mathrm{KL}\bigl(\pi_{\theta_n}(\cdot|s)\mid\mu\bigr) \to +\infty\,\text{as} \, n \to \infty,
\]
since $\|\theta_n\|_2\to\infty$.
Using the objective lower bound established above, nonnegativity of KL divergence, and Fatou's lemma, we obtain
\[
\liminf_{n\to\infty} V_\tau^{\pi_{\theta_n}}(\rho)\geq-\frac1{1-\gamma}+\tau\int_S\liminf_{n\to\infty}\mathrm{KL}\bigl(\pi_{\theta_n}(\cdot|s)\mid\mu\bigr)\,\rho(\mathrm ds)=+\infty.
\]
This contradicts the defining sublevel bound $V_\tau^{\pi_{\theta_n}}(\rho)\leq V_\tau^{\pi_{\theta_0}}(\rho)$.
Thus $\mathcal K_{\theta_0}$ is bounded and hence compact.

We now show that the maximal existence time satisfies $T^*=\infty$.
Since $\theta\mapsto\nabla V_\tau^{\pi_\theta}(\rho)$ is continuous by Proposition~\ref{prp:app:mdp:local_lipchitiz_conitinuity_mdp} and $\mathcal K_{\theta_0}$ is compact, the constant
\[
M:=\max_{\theta\in\mathcal K_{\theta_0}} \left\|\nabla V_\tau^{\pi_\theta}(\rho)\right\|_2
\]
is finite. 
As $\theta_t\in\mathcal K_{\theta_0}$ for every $t\in[0,T^*)$, the integral form of \eqref{eq:intro:continuous_GF} gives, for $0\le s<t<T^*$,
\[
\|\theta_t-\theta_s\|_2 = \bigg\| \int_s^t \nabla V_\tau^{\pi_{\theta_r}}(\rho)\,\mathrm dr \bigg\|_2 \le M(t-s).
\]
Suppose, for contradiction, that the maximal existence time $T^* < \infty$.
The preceding estimate shows that $\theta_t$ is Cauchy as $t\uparrow T^*$, so there exists
\[
\bar\theta:=\lim_{t\uparrow T^*}\theta_t \in\mathcal K_{\theta_0},
\]
since $\mathcal K_{\theta_0}$ is closed.
By Proposition~\ref{prp:app:mdp:local_lipchitiz_conitinuity_mdp} and the Picard--Lindel\"of theorem, there exist $\varepsilon>0$ and a local solution $\eta:[0,\varepsilon)\to\mathbb R^p$ of
\[
\frac{\mathrm d}{\mathrm du}\eta_u = -\nabla V_\tau^{\pi_{\eta_u}}(\rho), \ \eta_0=\bar\theta.
\]
Define
\[
\widetilde\theta_t:=
\begin{cases}
\theta_t, & 0\le t<T^*,\\
\eta_{t-T^*}, & T^*\le t<T^*+\varepsilon.
\end{cases}
\]
Passing to the limit in the integral equation for $\theta$ is justified by the bound on the gradient and yields
\[
\bar\theta = \theta_0-\int_0^{T^*} \nabla V_\tau^{\pi_{\theta_r}}(\rho)\,\mathrm dr.
\]
Combining this identity with the integral equation for $\eta$, we obtain
\[
\widetilde\theta_t=\theta_0-\int_0^t \nabla V_\tau^{\pi_{\widetilde\theta_r}}(\rho)\,\mathrm dr, \ 0\le t<T^*+\varepsilon.
\]
Since $t \mapsto \widetilde \theta_t$ and the vector field are continuous, this identity implies that $t \mapsto \widetilde \theta_t$ is continuously differentiable and solves \eqref{eq:intro:continuous_GF} at $t=T^*$.
This contradicts the maximality of $T^*$.
Therefore $T^*=\infty$, and local uniqueness implies uniqueness of the global solution.

Moreover, since $\theta_t\in\mathcal K_{\theta_0}$ for all $t\ge0$, compactness gives
\[
\sup_{t\ge0}\|\theta_t\|_2 \le \max_{\theta\in\mathcal K_{\theta_0}}\|\theta\|_2 =:C_{\tau,\theta_0}<\infty.
\]
Then for all $s \in S, a \in A$, any $t \geq 0$,
\begin{align*}
\left|\log\frac{\mathrm d \pi_{\theta_t}}{\mathrm d \mu}(a|s)\right|
& = \left| \langle \theta_t, g(s,a)\rangle - \log Z_{\pi_{\theta_t}}(s)\right| \\
& \leq \sup_{s,a}\|g(s,a)\|_2\|\theta_t\|_2 + \left|\log \int_{A} e^{\langle \theta_t,g(s,a)\rangle}\mu(\mathrm{d}a)\right| \leq 2C_{\tau,\theta_0} < \infty\,.
\end{align*}
\end{proof}

\subsubsection{Proof of Theorem \ref{thm:main:lin_cov_affine}}
\label{sec:pf_lin_cov_affine}
\aspInitAffine* 

\begin{proof}[Proof of Lemma~\ref{lemma:app:mdp:initial_singularity_mdp}]
For every $v\in\mathbb{R}^{p}$,
\[
\begin{aligned}
& v^\top G^{\pi_{{\theta}}}(s)v \\
& = \int_A v^\top \Big({g(s,a)} - \int_A {g(s,a^\prime)}\pi_{{\theta}}(\mathrm{d}a' | s)\Big) \Big({g(s,a)} - \int_A {g(s,a^\prime)}\pi_{{\theta}}(\mathrm{d}a' | s)\Big)^\top v \ \pi_{{\theta}}(\mathrm{d}a | s) \\
& = \operatorname{Var}_{a\sim \pi_{{\theta}}(\cdot| s)}
\bigl(v^\top {g(s,a)}\bigr).
\end{aligned}
\]
If this variance were zero for some $v\neq \mathbf0$, then $v^\top {g(s,\cdot)}$ would be constant $\pi_{{\theta}}(\cdot| s)$-almost surely.
Since the softmax density is strictly positive with respect to $\mu$, it would also be constant $\mu$-almost everywhere. 
Continuity of $a \mapsto {g(s,a)}$ and full support of $\mu$ would then imply that it is constant on all of $A$.
Its maximizer set would consequently be all of $A$, contradicting the zero-measure maximizer condition in Assumption~\ref{asp:app:mdp:max_set_affine}. 
Therefore, $G^{\pi_{{\theta}}}(s)\succ0$ for every ${\theta} \in \mathbb R^{p}$.
\end{proof}

\LinCovAffine*  
\begin{proof}
From Lemma \ref{lem:main:mdp:exist_sol_gf_affine}, there exists $R := \textstyle \sup_{t\ge0}\big|\log\tfrac{\mathrm d\pi_{\theta_t}}{\mathrm d\mu}\big|_{B_b(S\times A)}$ such that for $t \geq 0$, 
\[
\theta_t \in \Theta_{R} = \bigg\{\theta: \pi_\theta \in \Pi_\mu \ \text{and}\  \Big|\log \frac{\mathrm d \pi_\theta}{\mathrm d\mu}\Big|_{B_b(S\times A)} \leq R \bigg\},
\] 
which means that $\inf_{s,a}\frac{\mathrm d \pi_{\theta_t}}{\mathrm d \mu}(a|s) \geq e^{-R}$.

From Lemmas~\ref{lemma:app:mdp:initial_singularity_mdp} and~\ref{lemapp:mdp:lower_bound_log_density}, we know that there exists 
\[
\lambda_{\text{fsa}, \theta_0} := e^{-R-2\|{\theta}_0\|_2}\int_{S}   \lambda_{\min} \left(G^{\pi_{{\theta}_0}}(s)\right)\rho(\mathrm{d}s) >0
\] 
such that 
\begin{equation}\label{eq:positive_lambda_affine}
\int_{S}\lambda_{\min} (G^{\pi_{{\theta}_t}}(s)) \rho(\mathrm{d}s) \geq \lambda_{\text{fsa}, \theta_0} >0.
\end{equation}

From Theorem \ref{thm:main: non-uniform_PL_ineq}, let $\hat{R} = \max\left(\frac{2}{(1-\gamma)\tau} \left(1+\gamma\tau R\right), R\right)$, we know
\begin{align*}
\sup_{t \geq 0} C_R(\theta_t) = \sup_{t \geq 0} \frac{e^{\hat R}}{2\tau(1-\gamma) } \bigg[\int_S \lambda_{\min}(G^{\pi_{{\theta}_t}}(s)) \rho(\mathrm{d}s)\bigg]^{-2} \leq \frac{e^{\hat R}}{2\tau \lambda^2_{\text{fsa},\theta_0} (1-\gamma)} := C_{\theta_0}.
\end{align*}
Then from Proposition \ref{prp:app:mdp:Gradient_of_the_Objective}, and Theorem \ref{thm:main: non-uniform_PL_ineq} again,
\begin{align*}
\frac{\mathrm d}{\mathrm dt} V_\tau^{\pi_{\theta_t}}(\rho) &= - \left\|\nabla V_\tau^{\pi_{\theta_t}}(\rho)\right\|^2_2 \leq -C_{\theta_0}^{-1} \left(V_\tau^{\pi_{\theta_t}}(\rho) - V_\tau^{\pi^*}  (\rho) \right)
\end{align*}
then using Grönwall's inequality concludes the proof.
\end{proof}

\subsubsection{Proof of Theorem \ref{thm:main:mdp:Radial_unboundedness_of_relative_entropy_in_subspace_simplex}}
\label{sec:pf_Radial_unboundedness_of_relative_entropy_in_subspace_simplex}

\aspSimplexMaxSet*

Recall that for any $v\in \mathbb R^p$ we will write $v_\perp := v - p^{-1}\langle v, \mathbf 1\rangle \mathbf 1$.

\KLUBSIMPLEX*
\begin{proof}
Consider $p \geq 2$ since $p=1$ is trivial.
Since $\mathbf1^\top g(s,a)=1$, one has $\pi_\theta=\pi_{\theta_\perp}$.
Fix $s$ and apply Lemma~\ref{lem:app:uniform-radial-kl} on the Euclidean space $E=\mathbf1^\perp$ with $h(a)=g(s,a)-p^{-1}\mathbf1$. 
For a unit vector $u\perp\mathbf1$, $u^\top h(a) = u^\top g(s,a)$.  Assumption~\ref{asp:app:mdp:max_set_simplex} supplies exactly the required zero-measure maximizer condition.
Hence the KL divergence tends to infinity whenever $\|\theta_\perp\|_2\to\infty$.
\end{proof}

\subsubsection{Proof of Lemma \ref{lem:main:mdp:exist_sol_gf_simplex}}
\label{sec:pf_exist_sol_gf_simplex}

\ExistSolSimplex*

\begin{proof}
Write $Z_{\pi_\theta}(s) := \int_A \exp(\theta^\top g(s,a))\mu(\mathrm d a)$.
Under Assumption \ref{asp:app:mdp:max_set_simplex}, by Proposition \ref{prp:app:mdp:Gradient_of_the_Objective} and Lemma \ref{lem:app:mdp:grad_of_log_policy}, we know that
\begin{align*}
\frac{\mathrm d}{\mathrm dt}\left(\mathbf{1}^{\top} \theta_t\right) &= \frac{-1}{1-\gamma}\int_{S\times A}\left(Q^{\pi_{\theta_t}}_\tau(s,a) + \tau \log \frac{\mathrm d \pi_{{\theta_t}}}{\mathrm d \mu}(a|s)\right) \mathbf{1}^\top\nabla \log \frac{\mathrm d\pi_{\theta_t}}{\mathrm d\mu}(a| s)\pi_{{\theta_t}}(\mathrm{d}a|s)d_\rho^{\pi_{{\theta_t}}}(\mathrm{d}s)= 0,
\end{align*}
implying that $\mathbf{1}^\top \theta_t = \mathbf{1}^\top \theta_0$ for any $t \geq 0$.
Consider $\theta_t = (\theta_{t})_{\perp} + \frac{\theta_t^\top \mathbf{1}}{p}\mathbf{1}$, where $(\theta_{t})_{\perp} \perp \mathbf{1}$. 
Define
\[
\mathcal K_{\theta_0}:=\left\{\theta\in\mathbb R^p:V_\tau^{\pi_\theta}(\rho)\leq V_\tau^{\pi_{\theta_0}}(\rho),
\quad\mathbf1^\top\theta=\mathbf1^\top\theta_0\right\}.
\]

Since the vector field is locally Lipschitz by Proposition~\ref{prp:app:mdp:local_lipchitiz_conitinuity_mdp}, we know from Picard--Lindel\"of theorem there exists a unique solution on $[0, T^*)$ with $T^* > 0$.
Note that for $t \in [0, T^*)$,
\[
\frac{\mathrm d}{\mathrm dt}V_\tau^{\pi_{\theta_t}}(\rho)=-\left\|\nabla V_\tau^{\pi_{\theta_t}}(\rho)\right\|_2^2 \leq 0.
\]
The non-increasing objective and conservation of $\mathbf1^\top\theta_t$ imply that $\theta_t\in\mathcal K_{\theta_0}$ throughout the maximal interval of existence.

We show that $\mathcal K_{\theta_0}$ is compact.
It is closed because $\theta\mapsto V_\tau^{\pi_\theta}(\rho)$ is continuous by Proposition~\ref{prp:app:mdp:Gradient_of_the_Objective}.
If it were unbounded, there would exist $\theta_t \in\mathcal K_{\theta_0}$ with $\|\theta_t\|_2\to\infty$.
Since $\mathbf1^\top\theta_t=\mathbf1^\top\theta_0$, we would have $\|(\theta_t)_\perp\|_2\to\infty$.
By~\eqref{eq:intro:value_function_occ}, the bound on cost function, \eqref{eq:kernel_occ_geq_rho} and $\pi_\theta = \pi_{\theta_\perp}$, we have
\begin{align*}
V^{\pi_\theta}_\tau(\rho)\ & \ge\ -\frac{1}{1-\gamma}+\tau\int_S \operatorname{KL}(\pi_\theta(\cdot|s)|\mu)\,\rho(\mathrm ds)\\
&=\ -\frac{1}{1-\gamma}+\tau\int_S \operatorname{KL}(\pi_{\theta_\perp}(\cdot|s)|\mu)\,\rho(\mathrm ds).
\end{align*}
Theorem~\ref{thm:main:mdp:Radial_unboundedness_of_relative_entropy_in_subspace_simplex} gives, for every $s\in S$,
\[
\mathrm{KL}\bigl(\pi_{(\theta_t)_\perp}(\cdot|s)\mid\mu\bigr)\longrightarrow+\infty.
\]
Using the objective lower bound established above, nonnegativity of relative entropy, and Fatou's lemma, we obtain
\[
\begin{aligned}
\liminf_{t\to\infty} V_\tau^{\pi_{\theta_t}}(\rho)\geq-\frac1{1-\gamma}+\tau\int_S\liminf_{t\to\infty}\mathrm{KL}\bigl(\pi_{\theta_t}(\cdot|s)\mid\mu\bigr)\,\rho(\mathrm ds)=+\infty.\end{aligned}
\]
This contradicts the defining sublevel bound.
Thus $\mathcal K_{\theta_0}$ is bounded and hence compact.

We know that the maximal existence time satisfies $T^*=\infty$ by the same argument in the proof of Lemma~\ref{lem:main:mdp:exist_sol_gf_affine}.

Moreover, since $\theta_t\in\mathcal K_{\theta_0}$ for all $t\ge0$, compactness gives
\[
\sup_{t\ge0}\|\theta_t\|_2 \le \max_{\theta\in\mathcal K_{\theta_0}}\|\theta\|_2 =:C_{\tau,\theta_0}<\infty.
\]

Then for all $s \in S, a \in A$, any $t \geq 0$,
\begin{align*}
\left|\log\frac{\mathrm d \pi_{\theta_t}}{\mathrm d \mu}(a|s)\right|
& = \left| \langle \theta_t, g(s,a)\rangle - \log Z_{\pi_{\theta_t}}(s)\right| \\
& \leq \sup_{s,a}\|g(s,a)\|_2\|\theta_t\|_2 + \left|\log \int_{A} e^{\langle \theta_t,g(s,a)\rangle}\mu(\mathrm{d}a)\right|
\leq 2C_{\tau,\theta_0} < \infty\,.
\end{align*}
\end{proof}

\subsubsection{Proof of Theorem  \ref{thm:main:lin_cov_simplex}}
\label{sec:pf_lin_cov_simplex}
\aspInitSimplex*

\begin{proof}[Proof of Lemma~\ref{lemma:app:mdp:initial_singularity_simplex}]
Let $\mathcal{M}_\theta(s) := \int_{A} g(s,a)g(s,a)^\top \pi_\theta(\mathrm{d}a| s).$
We will argue by contradiction.
Suppose that $v^\top \mathcal{M}_\theta(s)v=0$ for some $v \neq \boldsymbol{0}$. 
Since $v^\top g g^\top v = (\sum_{i=1}^p v_i g_i)(\sum_{i=1}^p g_i v_i) \geq 0$ we have $v^\top g(s,a)=0$ for $\pi_\theta(\cdot| s)$-almost every $a$, and hence everywhere by strict positivity of softmax density for $\theta\in \mathbb R^p$, continuity of $a \mapsto g(s,a)$, and full support of $\mu$. 
Write $v=u+c\textbf{1}, u\perp\textbf{1}, c\in \mathbb R.$
Because simplex features satisfy $\textbf{1}^\top g(s,a)=1$, we obtain $u^\top g(s,a)=-c \ \text{for all }a.$
So $u^\top g(s,\cdot)$ is constant, so its maximizer set is all of $A$, contradicting Assumption~\ref{asp:app:mdp:max_set_simplex} unless $u=\boldsymbol{0}$.
Thus $u= \boldsymbol{0}$ consequently $c=0$ and therefore $v=\boldsymbol{0}$. 
This contradicts our initial hypothesis that $v\neq  \boldsymbol{0}$.
Thus by contradiction
$\mathcal{M}_\theta(s)\succ 0$.
\end{proof}

\LinCovSimplex*

\begin{proof}
From Lemma \ref{lem:main:mdp:exist_sol_gf_simplex}, there exists $R := \textstyle \sup_{t\ge0}\big|\log\tfrac{\mathrm d\pi_{\theta_t}}{\mathrm d\mu}\big|_{B_b(S\times A)}$ such that for $t \geq 0$,
\[
\theta_t \in \Theta_{R} = \bigg\{\theta: \pi_\theta \in \Pi_\mu \ \text{and}\  \Big|\log \frac{\mathrm d \pi_\theta}{\mathrm d\mu}\Big|_{B_b(S\times A)} \leq R \bigg\},
\] 
which means that $\inf_{s,a}\frac{\mathrm d \pi_{\theta_t}}{\mathrm d \mu}(a|s) \geq e^{-R}$.
From Lemmas \ref{lem:app:mdp:lower_bound_log_density_simplex} and \ref{lemma:app:mdp:initial_singularity_simplex}, we know that there exists 
\begin{equation}\label{eq:positive_lambda_simplex}
\lambda_{\text{\text{smp}}, \theta_0} = e^{-R-2\|\theta_0\|_2}  \int_{S} \lambda_{\min}\left(\int_{A} g(s,a)g^\top(s,a)\pi_{\theta_0} (\mathrm{d}a | s)\right) \rho(\mathrm{d}s) >0
\end{equation}
such that 
\[
\int_{S} \lambda_{\min} \Big(\int_{A} g(s,a)g^\top(s,a)\pi_{\theta_t} (\mathrm{d}a | s)\Big) \rho(\mathrm{d}s) \geq \lambda_{\text{\text{smp}}, \theta_0} >0.
\]

From Theorem \ref{thm:main: non-uniform_PL_ineq}, let $\hat{R} = \max\left(\frac{2}{(1-\gamma)\tau} \left(1+\gamma\tau R\right), R\right)$, we know
\begin{align*}
\sup_{t \geq 0} C_R(\theta_t)
& = \sup_{t \geq 0} \frac{e^{\hat R}}{2\tau (1-\gamma) } \left[\int_{S}\lambda_{\min}\left(\int_{A} g(s,a)g^\top(s,a)\pi_{\theta_t}(\mathrm{d}a|s)\right) \rho(\mathrm{d}s)\right]^{-2}\\
& \leq \frac{e^{\hat R}}{2\tau \lambda^2_{\text{\text{smp}}, \theta_0} (1-\gamma)}:= C_{\theta_0}.
\end{align*}
Then from Proposition \ref{prp:app:mdp:Gradient_of_the_Objective}, and Theorem \ref{thm:main: non-uniform_PL_ineq} again,
\begin{align*}
\frac{\mathrm d}{\mathrm dt} V_\tau^{\pi_{\theta_t}}(\rho) &= - \left\|\nabla V_\tau^{\pi_{\theta_t}}(\rho)\right\|^2_2 \leq -C_{\theta_0}^{-1} \left(V_\tau^{\pi_{\theta_t}}(\rho) - V_\tau^{\pi^*}  (\rho) \right)
\end{align*}
then using Grönwall's inequality concludes the proof.
\end{proof}

\subsection{More discussions about Example~\ref{ex:trig_func_mdp}}
\label{sec:trig_mdp}
Throughout this section, let $S=[-\pi,\pi]^{d_1}$ and $A=[-\pi,\pi]^{d_2}$, where $d_1,d_2\geq1$, and let $\mu(\mathrm da)=\frac{1}{(2\pi)^{d_2}}\,\mathrm da, \nu(\mathrm ds')=\frac{1}{(2\pi)^{d_1}}\,\mathrm ds',$ where $\mu$ is the reference probability measure on the action space and $\nu$ is the uniform probability measure on the state space.
Fix $\gamma\in[0,1)$ and $\tau>0$.
Fix an integer $p\geq2$, and put $m:=\lceil p/2\rceil$ and $n:=\lfloor p/2\rfloor$.
Choose action-frequency vectors $k_1,\ldots,k_m\in\mathbb Z^{d_2}\setminus\{\mathbf0\}$ such that $k_i\neq\pm k_j$ whenever $i\neq j$, and arbitrary state-frequency vectors $\ell_1,\ldots,\ell_m\in\mathbb R^{d_1}$.
Define $g:S\times A\to\mathbb R^p$ by
\begin{equation}
\begin{aligned}
g_{2j-1}(s,a) &=\frac{1}{\sqrt m}\cos(k_j^\top a+\ell_j^\top s), &&j=1,\ldots,m,\\
g_{2j}(s,a) &=\frac{1}{\sqrt m}\sin(k_j^\top a+\ell_j^\top s), &&j=1,\ldots,n.
\end{aligned}
\label{eq:features}
\end{equation}
If $p=2m$, the feature map consists of $m$ complete sine--cosine pairs.
If $p=2m-1$, it consists of $m-1$ complete pairs and the additional coordinate $g_{2m-1}(s,a) =\frac{1}{\sqrt m}\cos(k_m^\top a+\ell_m^\top s).$
In both cases, $ \sup_{s,a}\|g(s,a)\|_2\leq1$.

Since the nonzero integer action frequencies are distinct modulo sign, Fourier orthogonality on $A$ gives, for every fixed $s\in S$, $\int_Ag(s,a)\,\mu(\mathrm da)=\mathbf0, \int_Ag(s,a)g(s,a)^\top\,\mu(\mathrm da)=\frac{1}{2m}I_p.$
Thus, for every $\theta\neq\mathbf0$, the function $a\mapsto\langle\theta,g(s,a)\rangle$ has zero mean and second moment $\|\theta\|_2^2/(2m)>0$, and hence is nonconstant.
The difference between this function and its maximum on $A$ is therefore a nontrivial real-analytic function on $\mathbb R^{d_2}$.
Its zero set has $d_2$-dimensional Lebesgue measure zero, so
\[
\mu\!\left(\operatorname*{arg\,max}_{a\in A}\langle\theta,g(s,a)\rangle\right)=0.
\]

Independently choose next-state frequency vectors $\ell'_1,\ldots,\ell'_m\in\mathbb Z^{d_1}\setminus\{\mathbf0\}$, and define $h:S\to\mathbb R^p$ by
\[
\begin{aligned}
h_{2j-1}(s')&=\frac{1}{\sqrt m} \cos\bigl((\ell'_j)^\top s'\bigr), &&j=1,\ldots,m,\\
h_{2j}(s')&=\frac{1}{\sqrt m}\sin\bigl((\ell'_j)^\top s'\bigr),&&j=1,\ldots,n.
\end{aligned}
\]
Then $\|h(s')\|_2\leq1$, and the nonzero integer frequencies ensure that $\int_Sh(s')\,\nu(\mathrm ds')=\mathbf0$.

Fix $\varepsilon\in(0,1)$ and define the transition kernel
\begin{equation}
P(\mathrm ds'| s,a) =\left[1+\varepsilon\langle g(s,a),h(s')\rangle\right]\nu(\mathrm ds').
\label{eq:transition-kernel}
\end{equation}
By the bounds on $g$ and $h$ and the Cauchy--Schwarz inequality, $\left|\langle g(s,a),h(s')\rangle\right| \leq\|g(s,a)\|_2\|h(s')\|_2\leq1.$
Consequently, $1-\varepsilon \leq\frac{\mathrm dP(\cdot| s,a)}{\mathrm d\nu}(s')\leq1+\varepsilon.$
In particular, the density is strictly positive.
We also obtain
\begin{align}
P(S| s,a)=\int_S\left[1+\varepsilon\langle g(s,a),h(s')\rangle\right]\nu(\mathrm ds')=1+\varepsilon\left\langle g(s,a),\int_Sh(s')\,\nu(\mathrm ds')\right\rangle=1.
\label{eq:transition-normalization}
\end{align}
Since the density is continuous in $(s,a,s')$, \eqref{eq:transition-kernel} defines a measurable probability kernel.

Define the finite signed measures $\psi_i(\mathrm ds'):=\varepsilon h_i(s')\,\nu(\mathrm ds')$, $i=1,\ldots,p$.
Then $\psi_i(S)=0$ and
\begin{equation}
P(\mathrm ds'| s,a)=\nu(\mathrm ds')+\sum_{i=1}^p g_i(s,a)\psi_i(\mathrm ds').
\label{eq:linear-transition-baseline}
\end{equation}
Thus the transition kernel is linear in the features up to the action-independent baseline measure $\nu$.

When $p=2m$, the identity $\cos x\cos y+\sin x\sin y=\cos(x-y)$ gives
\begin{equation}
\begin{split}
& P(\mathrm ds'| s,a) = \textstyle\frac{1}{(2\pi)^{d_1}}\left[1+\frac{\varepsilon}{m}\sum_{j=1}^m\cos\bigl(k_j^\top a+\ell_j^\top s-(\ell'_j)^\top s'
\bigr)\right]\mathrm ds'.
\label{eq:even-dimensional-kernel}
\end{split}
\end{equation}
When $p=2m-1$, the same construction gives
\begin{equation}
\begin{aligned}
&P(\mathrm ds'| s,a)\\
&=\textstyle \frac{1}{(2\pi)^{d_1}}\Big[1+\frac{\varepsilon}{m}\Big\{\sum_{j=1}^{m-1}\cos\bigl(k_j^\top a+\ell_j^\top s-(\ell'_j)^\top s'\bigr)+\cos(k_m^\top a+\ell_m^\top s)\cos\bigl((\ell'_m)^\top s'\bigr)\Big\}\Big]\mathrm ds'.
\end{aligned}
\label{eq:odd-dimensional-kernel}
\end{equation}

Choose $w\in\mathbb R^p$ with $\|w\|_2\leq1/4$, and define
\begin{equation}
c(s,a):=\frac12+\langle w,g(s,a)\rangle.
\label{eq:cost}
\end{equation}
By the bound on $g$, for every $(s,a)\in S\times A$, $\frac14\leq c(s,a)\leq\frac34$.

For $\pi\in\Pi_\mu$, recall that
\begin{equation}
Q_\tau^\pi(s,a)=c(s,a)+\gamma\int_SV_\tau^\pi(s')P(\mathrm ds'| s,a).
\label{eq:q-definition}
\end{equation}
Set
\begin{equation}
\boldsymbol\theta(\pi)_i:=-\frac1\tau\left(w_i+\gamma\int_SV_\tau^\pi(s')\psi_i(\mathrm ds')\right),\qquad i=1,\ldots,p.
\label{eq:realizing-parameter}
\end{equation}
These integrals are finite because $V_\tau^\pi$ and the features are bounded and $\psi_i=\varepsilon h_i\nu$.
Substituting \eqref{eq:linear-transition-baseline} and \eqref{eq:cost} into \eqref{eq:q-definition} yields
\begin{align}
Q_\tau^\pi(s,a)&=\frac12+\langle w,g(s,a)\rangle+\gamma\int_SV_\tau^\pi(s')\,\nu(\mathrm ds')+\gamma\sum_{i=1}^p g_i(s,a)\int_SV_\tau^\pi(s')\,\psi_i(\mathrm ds')
\nonumber\\
&=\frac12+\gamma\int_S V_\tau^\pi(s')\,\nu(\mathrm ds')+\sum_{i=1}^p\left(w_i+\gamma\int_S V_\tau^\pi(s')\,\psi_i(\mathrm ds')\right)g_i(s,a).
\label{eq:q-expansion}
\end{align}
Using \eqref{eq:realizing-parameter}, this becomes
\[
Q_\tau^\pi(s,a)=\frac12+\gamma\int_SV_\tau^\pi(s')\,\nu(\mathrm ds')-\tau\langle\boldsymbol\theta(\pi),g(s,a)\rangle.
\]
After rearrangement, since $\nu(\mathrm ds')=(2\pi)^{-d_1}\mathrm ds'$, we obtain
\begin{equation}
\langle\boldsymbol\theta(\pi),g(s,a)\rangle=-\frac1\tau Q_\tau^\pi(s,a)+\frac{1}{2\tau}+\frac{\gamma}{\tau(2\pi)^{d_1}}\int_SV_\tau^\pi(s')\,\mathrm ds'.
\label{eq:q-realizability-identity}
\end{equation}
In this case, the state-dependent baseline can be chosen as
\begin{equation}
b^\pi(s)=-\frac1\tau \left(\frac12+\gamma\int_S V_\tau^\pi(s')\,\nu(\mathrm ds')\right),
\label{eq:baseline}
\end{equation}
which is constant in $s$ in this particular example.

Also we can construct an example with a baseline that depends on $s$ nontrivially.
Set $\eta \in (0, 1 - \varepsilon)$ and $r \in \mathbb Z^{d_1}\setminus\{\mathbf0\}$.
Define 
\[
\bar \nu(\mathrm d s'|s) := [1 + \eta \cos(r^\top(s' - s))] \nu(\mathrm d s').
\]
and the transition kernel
\[
P(\mathrm d s'|s,a) = \bar \nu(\mathrm d s'|s) + \sum_{i=1}^p g_i(s,a)\psi_i(\mathrm ds'),
\]
with $\frac{\mathrm{d} P}{\mathrm d \nu} \geq 1 - \eta -\varepsilon >0$.
Note that 
\[
\int_S P(\mathrm d s'|s,a) = 1 + \int_S \eta \cos(r^\top(s' - s)) \nu(\mathrm ds') = 1.
\]
In this case, we use the same $\boldsymbol\theta(\pi)_i=-\frac1\tau\left(w_i+\gamma\int_SV_\tau^\pi(s')\psi_i(\mathrm ds')\right).$
By the same calculations, we obtain
\begin{equation*}
\langle\boldsymbol\theta(\pi),g(s,a)\rangle=-\frac1\tau Q_\tau^\pi(s,a)-b^\pi(s),
\end{equation*}
where $b^\pi(s) := -\left(\frac{1}{2\tau}+\frac{\gamma}{\tau}\int_SV_\tau^\pi(s')\, \bar \nu (\mathrm ds'|s)\right)$ is state-dependent.

\subsection{More discussions about Example~\ref{ex:hat_func_short}}
\label{sec:ex}

Let $\mu$ be uniform on $[0,1]$.
Let $0=x_0<x_1<\cdots <x_N=1$, and consider the standard one-dimensional finite-element hat functions
\[
g_i(a)=
\begin{cases}
\dfrac{a-x_{i-1}}{x_i-x_{i-1}}, & a\in [x_{i-1},x_i],\\[0.8em]
\dfrac{x_{i+1}-a}{x_{i+1}-x_i}, & a\in [x_i,x_{i+1}],\\[0.8em]
0, & \text{otherwise},
\end{cases}
\qquad 1\leq i\leq N-1,
\]
together with the boundary functions
\[
g_0(a)=
\begin{cases}
\dfrac{x_1-a}{x_1-x_0}, & a\in [x_0,x_1],\\[0.8em]
0, & \text{otherwise},
\end{cases}
\qquad
g_N(a)=
\begin{cases}
\dfrac{a-x_{N-1}}{x_N-x_{N-1}}, & a\in [x_{N-1},x_N],\\[0.8em]
0, & \text{otherwise}.
\end{cases}
\]
Set $g(a):=(g_0(a),g_1(a),\dots,g_N(a))^\top$. We now take $N=3$ and choose the concrete grid $x_0=0,\ x_1=\frac13,\ x_2=\frac23,\ x_3=1$.

Let $\theta=(-1,1,1,-1)^\top$. Then $\theta^\top \mathbf 1=-1+1+1-1=0$. Define $f_\theta(a):=\theta^\top g(a)$. Therefore
\[
f_\theta(a)=
\begin{cases}
6a-1, & a\in \left[0,\frac13\right],\\[0.8em]
1, & a\in \left[\frac13,\frac23\right],\\[0.8em]
5-6a, & a\in \left[\frac23,1\right].
\end{cases}
\]
Thus $f_\theta(a)\leq 1$ for every $a\in[0,1]$, and the maximizer set is $\operatorname*{arg\,max}_{a\in[0,1]} f_\theta(a)=\left[\frac13,\frac23\right]$, which has positive Lebesgue measure $\mu([\frac13,\frac23])=\frac13$. Consequently, Assumption~\ref{asp:app:mdp:max_set_simplex} is not satisfied in this example.

For $\beta>0$, define the probability measure $\pi_\beta$ by $p_\beta(a):=\frac{\mathrm d\pi_\beta}{\mathrm d\mu}(a)=e^{\beta f_\theta(a)}/Z_\beta$, where $Z_\beta:=\int_0^1 e^{\beta f_\theta(a)}\,\mathrm{d}a$. 
Since $p_\beta$ integrates to one, the relative entropy of $\pi_\beta$ with respect to $\mu$ is $\operatorname{KL}(\pi_\beta| \mu):=\int_0^1 p_\beta(a)\log p_\beta(a)\,\mathrm{d}a$. 
Moreover, since $\log p_\beta(a)=\beta f_\theta(a)-\log Z_\beta$, we have
$\operatorname{KL}(\pi_\beta| \mu)=\beta\,\mathbb E_{\pi_\beta}[f_\theta(a)]-\log Z_\beta .$
In the present grid, $L_1:=x_1-x_0=\frac13$, $L:=x_2-x_1=\frac13$, $L_2:=x_3-x_2=\frac13$.

We first compute $Z_\beta$. On $[0,\frac13]$, set $y=f_\theta(a)=6a-1$. Then $\int_0^{1/3} e^{\beta f_\theta(a)}\,\mathrm{d}a=\frac16\int_{-1}^1 e^{\beta y}\,\mathrm d y$. On the middle interval $[\frac13,\frac23]$, one has $f_\theta(a)=1$, so $\int_{1/3}^{2/3} e^{\beta f_\theta(a)}\,\mathrm{d}a=\frac13 e^\beta$. On $[\frac23,1]$, set $y=f_\theta(a)=5-6a$. Then $\mathrm d y=-6\,\mathrm{d}a$, $\int_{2/3}^{1} e^{\beta f_\theta(a)}\,\mathrm{d}a=\frac16\int_{-1}^1 e^{\beta y}\,\mathrm d y$. Combining the three pieces gives
\[
Z_\beta
=
\frac13\int_{-1}^1 e^{\beta y}\,\mathrm d y+\frac13 e^\beta .
\]
Since $\int_{-1}^1 e^{\beta y}\,\mathrm d y=(e^\beta-e^{-\beta})/\beta$, we obtain
\[
Z_\beta
=
\frac13\frac{e^\beta-e^{-\beta}}{\beta}+\frac13 e^\beta
=
\frac{e^\beta}{3}\left[1+\frac{1-e^{-2\beta}}{\beta}\right].
\]
Therefore $\log Z_\beta=\beta-\log 3+\log(1+(1-e^{-2\beta})/\beta)$. Since $e^{-2\beta}$ is exponentially small as $\beta\to\infty$, and since $\log(1+u)=u-\frac{u^2}{2}+\mathcal O(u^3)$ as $u\to 0$, we get
\[
\log Z_\beta
=
\beta-\log 3+\frac1\beta-\frac{1}{2\beta^2}
+
\mathcal O\!\left(\frac1{\beta^3}\right).
\]

Next define $N_\beta:=\int_0^1 f_\theta(a)e^{\beta f_\theta(a)}\,\mathrm{d}a$. Then $\mathbb E_{\pi_\beta}[f_\theta(a)]=N_\beta/Z_\beta$. We compute $N_\beta$ using the same changes of variables. On the two exterior intervals, the variable $y=f_\theta(a)$ runs linearly from $-1$ to $1$, while on the middle interval $f_\theta(a)=1$. Hence
\[
N_\beta=\frac13\int_{-1}^1 y e^{\beta y}\,\mathrm d y+\frac13 e^\beta .
\]
Moreover,
\[
\int_{-1}^1 y e^{\beta y}\,\mathrm d y=\left[\frac{y e^{\beta y}}{\beta}-\frac{e^{\beta y}}{\beta^2}\right]_{-1}^{1}
=
\frac{e^\beta}{\beta}-\frac{e^\beta}{\beta^2}+\frac{e^{-\beta}}{\beta}+\frac{e^{-\beta}}{\beta^2}.
\]
Thus
\[
N_\beta
=
\frac{e^\beta}{3}\left[1+\frac1\beta-\frac1{\beta^2}
+
e^{-2\beta}\left(\frac1\beta+\frac1{\beta^2}\right)\right].
\]

Dividing the exact formula for $N_\beta$ by the exact formula for $Z_\beta$, we find
\[
\mathbb E_{\pi_\beta}[f_\theta(a)]
=
\frac{
1+\dfrac1\beta-\dfrac1{\beta^2}
+
e^{-2\beta}\left(\dfrac1\beta+\dfrac1{\beta^2}\right)
}{
1+\dfrac{1-e^{-2\beta}}{\beta}
}.
\]
Since $e^{-2\beta}$ is exponentially small, it may be absorbed into any algebraic remainder. Therefore
\[
\mathbb E_{\pi_\beta}[f_\theta(a)]
=
\frac{1+\dfrac1\beta-\dfrac1{\beta^2}}{1+\dfrac1\beta}
+
\mathcal O(e^{-2\beta}).
\]
The rational term equals $1-\frac{1/\beta^2}{1+1/\beta}$. Using $1/(1+1/\beta)=1-1/\beta+\mathcal O(1/\beta^2)$, we obtain
\[
\mathbb E_{\pi_\beta}[f_\theta(a)]
=
1-\frac1{\beta^2}+\frac1{\beta^3}
+
\mathcal O\!\left(\frac1{\beta^4}\right),
\]
and consequently
\[
\beta\,\mathbb E_{\pi_\beta}[f_\theta(a)]
=
\beta-\frac1\beta+\frac1{\beta^2}
+
\mathcal O\!\left(\frac1{\beta^3}\right).
\]

Substituting the expansions of $\beta\,\mathbb E_{\pi_\beta}[f_\theta(a)]$ and $\log Z_\beta$ into the entropy identity gives
\[
\begin{aligned}
\operatorname{KL}(\pi_\beta| \mu) = \beta\,\mathbb E_{\pi_\beta}[f_\theta(a)] - \log Z_\beta = \log 3 - \frac{2}{\beta}+\frac{3}{2\beta^2}+
\mathcal O\!\left(\frac1{\beta^3}\right).
\end{aligned}
\]
In particular,
\[
\operatorname{KL}(\pi_\beta| \mu)
=
\log 3-\frac{2}{\beta}+\mathcal O\!\left(\frac1{\beta^2}\right).
\]
Therefore, $\lim_{\beta\to\infty}\operatorname{KL}(\pi_\beta| \mu)=\log 3$.
Hence, in this concrete example, the relative entropy remains bounded as $\beta\to\infty$. Therefore, without Assumption~\ref{asp:app:mdp:max_set_simplex}, one cannot in general expect radial unboundedness of the KL divergence.

\subsection{Additional useful results}
\label{sec:extra results}

\begin{definition}[Linear functional derivative]
\label{def:app:flat_derivative}
Let $\mathcal C$ be a convex subset of
\[
\left\{ m\in\mathcal P(A):
m\ll\mu,\quad\frac{\mathrm dm}{\mathrm d\mu}\in L^\infty(A,\mu)\right\}.
\]
Equip $\mathcal C$ with the topology induced by the $L^\infty(A,\mu)$ norm of the densities; that is, $m_n\to m$ in $\mathcal C$ if
\[
\left\|\frac{\mathrm dm_n}{\mathrm d\mu}-\frac{\mathrm dm}{\mathrm d\mu}\right\|_{L^\infty(A,\mu)}\longrightarrow 0.
\]

A functional $F:\mathcal C\to\mathbb R^d$ is said to admit a linear functional derivative if there exists a continuous map
\[
\frac{\delta F}{\delta m}:\mathcal C\longrightarrow L^\infty(A,\mu;\mathbb R^d)
\]
such that, for every $m,m'\in\mathcal C$,
\begin{equation}
\label{eq: flat derivative01 }
\begin{aligned}
F(m')-F(m)
&=\int_0^1\int_A
\frac{\delta F}{\delta m}\bigl((1-\lambda)m+\lambda m'\bigr)(a)\left(\frac{\mathrm dm'}{\mathrm d\mu}(a)-\frac{\mathrm dm}{\mathrm d\mu}(a)\right)
\,\mu(\mathrm da)\,\mathrm d\lambda .
\end{aligned}
\end{equation}
\end{definition}

\begin{proposition}[Strong convexity of negative entropy]
\label{prp:app:Strong_convexity_of_negative_entropy}
Let $W>0$ and define
\[
\mathcal P_W^A :=\left\{\pi\in\mathcal P(A):\pi\ll\mu,\quad\left\|\frac{\mathrm d\pi}{\mathrm d\mu}\right\|_{L^\infty(A,\mu)}\leq W^{-1}\right\}.
\]
Define
\[
F(\pi):=\int_A\frac{\mathrm d\pi}{\mathrm d\mu}(a)\log\frac{\mathrm d\pi}{\mathrm d\mu}(a)\,\mu(\mathrm da),
\]
with $0\log0=0$.
Then $F$ is $W$-strongly convex with respect to the $L^2(A,\mu)$ distance between densities: for every
$\pi,\pi'\in\mathcal P_W^A$ and $\lambda\in[0,1]$,
\[
\begin{aligned}
F((1-\lambda)\pi+\lambda\pi')
\leq&(1-\lambda)F(\pi)+\lambda F(\pi')-\frac W2\lambda(1-\lambda)\left\|\frac{\mathrm d\pi'}{\mathrm d\mu}-\frac{\mathrm d\pi}{\mathrm d\mu}\right\|_{L^2(A,\mu)}^2.
\end{aligned}
\]

For brevity, write $f_\pi=\mathrm d\pi/\mathrm d\mu$ and $f_{\pi'}=\mathrm d\pi'/\mathrm d\mu$.
If $f_\pi>0$ $\mu$-almost everywhere and
\[
(1+\log f_\pi)(f_{\pi'}-f_\pi)\in L^1(A,\mu),
\]
then
\[
\begin{aligned}F(\pi')\geq{}&
F(\pi)+\int_A(1+\log f_\pi(a))(f_{\pi'}(a)-f_\pi(a))\,\mu(\mathrm da)+\frac W2\|f_{\pi'}-f_\pi\|_{L^2(A,\mu)}^2.
\end{aligned}
\]

On any convex subclass $\mathcal C\subseteq\mathcal P_W^A$ whose members satisfy $\log(\mathrm d\pi/\mathrm d\mu)\in L^\infty(A,\mu)$, the restriction of $F$ admits a linear functional derivative in the sense of Definition~\ref{def:app:flat_derivative}, given by
\[
\frac{\delta F(\pi)}{\delta\pi}(a)
= 1+\log f_\pi(a), \qquad \mu\text{-almost everywhere}.
\]
\end{proposition}

\begin{proof}
Let $\varphi(x)=x\log x$ for $x>0$ and
$\varphi(0)=0$.
On $(0,W^{-1}]$,
\[
\frac{\mathrm d^2}{\mathrm dx^2}\left(\varphi(x)-\frac W2x^2\right)=\frac1x-W\geq0.
\]
By continuity at zero, $x\mapsto\varphi(x)-Wx^2/2$ is convex on $[0,W^{-1}]$.
Hence, for $x,y\in[0,W^{-1}]$ and $\lambda\in[0,1]$,
\[
\begin{aligned}
\varphi((1-\lambda)x+\lambda y)
\leq{}& (1-\lambda)\varphi(x)+\lambda\varphi(y) -\frac W2\lambda(1-\lambda)(x-y)^2.
\end{aligned}
\]
Apply this pointwise with
$x=f_\pi(a)$ and $y=f_{\pi'}(a)$ and integrate with respect to $\mu$. This proves strong convexity on the entire class, including densities that vanish.

For $x>0$ and $y\in[0,W^{-1}]$, the supporting-line inequality for the same convex function gives
\[
\varphi(y) \geq \varphi(x)+(1+\log x)(y-x)+\frac W2(y-x)^2.
\]
Under the stated integrability assumption, its integral gives the second assertion.

Finally, let $\pi,\pi'\in\mathcal C$.
Both densities are bounded away from zero.
Therefore the densities $(1-\lambda)f_\pi+\lambda f_{\pi'}$, $\lambda\in[0,1]$, have a common positive lower bound and the common upper bound $W^{-1}$.
The scalar fundamental theorem of calculus and Fubini's theorem yield
\[
\begin{aligned}
F(\pi')-F(\pi) =\int_0^1\int_A\left[1+\log\bigl((1-\lambda)f_\pi(a)+\lambda f_{\pi'}(a)\bigr)\right](f_{\pi'}(a)-f_\pi(a))
\,\mu(\mathrm da)\,\mathrm d\lambda.
\end{aligned}
\]
This is the identity in
Definition~\ref{def:app:flat_derivative}.

To verify continuity of the derivative, fix $\pi\in\mathcal C$ and choose $c>0$ such that $f_\pi\geq c$ $\mu$-almost everywhere.
Whenever $\|f_{\pi'}-f_\pi\|_{L^\infty(A,\mu)}<c/2$, the mean-value theorem gives
\[
\|\log f_{\pi'}-\log f_\pi\|_{L^\infty(A,\mu)}\leq\frac2c\|f_{\pi'}-f_\pi\|_{L^\infty(A,\mu)}.
\]
Thus the derivative is continuous in the required norm.
\end{proof}

We next recall the KL--logit inequality (see \citealp[Lemma~27]{mei2020global} for the tabular case).
\begin{lemma}[KL logit inequality for softmax policy]\label{lem:app:kl-logit_ineq_affine}
Fix $s\in S$. 
If $\theta,\theta' \in \mathbb{R}^p$ satisfy
\[
\left|\frac{\mathrm d\pi_\theta}{\mathrm d\mu}(\cdot|s)\right|_{ B_b( A)} \le W^{-1},
\qquad
\left|\frac{\mathrm d\pi_{\theta'}}{\mathrm d\mu}(\cdot|s)\right|_{ B_b( A)}\le W^{-1}
\]
for some $W>0$, then
\[
\mathrm{KL}\left(\pi_\theta(\cdot|s)\,|\,\pi_{\theta'}(\cdot|s)\right)\le\frac{1}{2W}\|\theta-\theta'\|_2^2 .
\]
\end{lemma}

\begin{proof}
For simplicity, write $f_\theta(a|s):=\frac{\mathrm d\pi_\theta}{\mathrm d\mu}(a|s), f_{\theta'}(a|s):=\frac{\mathrm d\pi_{\theta'}}{\mathrm d\mu}(a|s).$
And 
\[
Z_\theta(s) := \int_A \exp(\theta^\top g(s,a))\mu(\mathrm d a).
\]
By definition, both densities are strictly positive $\mu$-a.e. Moreover, by assumption, $0<f_\theta(a|s),f_{\theta'}(a|s)\le W^{-1} \ \mu\text{-a.e.}.$

Define the negative entropy functional $F(f):=\int_{ A} f(a)\log f(a)\,\mu(\mathrm{d}a).$
By Proposition \ref{prp:app:Strong_convexity_of_negative_entropy}, $F$ is
$W$--strongly convex on the set of densities bounded by $W^{-1}$. Hence,
\[
D_F(f_{\theta'},f_\theta) \ge \frac{W}{2} \|f_{\theta'}-f_\theta\|_{L^2(A,\mu)}^2,
\]
where
\[
D_F(f_{\theta'},f_\theta)
:=
F(f_{\theta'})
-
F(f_\theta)
-
\int_{ A}
\left(1+\log f_\theta(a|s)\right)
\left(f_{\theta'}(a|s)-f_\theta(a|s)\right)
\,\mu(\mathrm{d}a).
\]
Since both $f_\theta$ and $f_{\theta'}$ are probability densities, $\int_{ A}\left(f_{\theta'}(a|s)-f_\theta(a|s)\right)\,\mu(\mathrm{d}a)=0.$
Therefore,
\[
D_F(f_{\theta'},f_\theta)
=
\mathrm{KL}\left(\pi_{\theta'}(\cdot|s)\, |\,\pi_\theta(\cdot|s)\right).
\]

Using the identity
\[
D_F(f_\theta,f_{\theta'})
+
D_F(f_{\theta'},f_\theta)
=
\int_{ A}
\left(\log f_\theta(a|s)-\log f_{\theta'}(a|s)\right)\left(f_\theta(a|s)-f_{\theta'}(a|s)\right)\,\mu(\mathrm{d}a),
\]
we obtain
\begin{align*}
&\mathrm{KL}\left(
\pi_\theta(\cdot|s)\,|\,\pi_{\theta'}(\cdot|s)
\right) \\
&\le
\int_{ A}\left(\log f_\theta(a|s)-\log f_{\theta'}(a|s)\right)\left(f_\theta(a|s)-f_{\theta'}(a|s)\right)
\,\mu(\mathrm{d}a)
-
\frac{W}{2}\|f_\theta-f_{\theta'}\|_{L^2(A,\mu)}^2 .
\end{align*}
By definition,
\[
\log f_\theta(a|s)-\log f_{\theta'}(a|s)
=
\langle \theta-\theta',g(s,a)\rangle
-
\log Z_\theta(s)
+
\log Z_{\theta'}(s).
\]
The normalizing constants vanish after integration because $\int_{ A} \left(f_\theta(a|s)-f_{\theta'}(a|s) \right)\,\mu(\mathrm{d}a)=0.$
Thus,
\begin{align*}
&\mathrm{KL}\left(
\pi_\theta(\cdot|s)\,|\,\pi_{\theta'}(\cdot|s)
\right) \le\int_{ A}\langle \theta-\theta',g(s,a)\rangle\left(f_\theta(a|s)-f_{\theta'}(a|s)\right)\,\mu(\mathrm{d}a)-\frac{W}{2}
\|f_\theta-f_{\theta'}\|_{L^2(A,\mu)}^2 .
\end{align*}
By the Cauchy--Schwarz inequality,
\begin{align*}
&\mathrm{KL}\left(
\pi_\theta(\cdot|s)\,|\,\pi_{\theta'}(\cdot|s)\right) \le\|\langle \theta-\theta',g(s,\cdot)\rangle\|_{L^2(A,\mu)}\|f_\theta-f_{\theta'}\|_{L^2(A,\mu)}-\frac{W}{2}\|f_\theta-f_{\theta'}\|_{L^2(A,\mu)}^2 .
\end{align*}
Using the elementary inequality $xy-\frac{W}{2}y^2\le \frac{x^2}{2W}, x,y\ge 0,$ with $x=\|\langle \theta-\theta',g(s,\cdot)\rangle\|_{L^2(A,\mu)},y=\|f_\theta-f_{\theta'}\|_{L^2(A,\mu)},$ we get
\[
\mathrm{KL}\left(
\pi_\theta(\cdot|s)\,|\,\pi_{\theta'}(\cdot|s)
\right)
\le
\frac{1}{2W}\|\langle \theta-\theta',g(s,\cdot)\rangle\|_{L^2(A,\mu)}^2 .
\]
Finally, by the bound on $g$, $\|\langle \theta-\theta',g(s,\cdot)\rangle\|_{L^2(A,\mu)} \le \|\theta-\theta'\|_2.$
Therefore,
\[
\mathrm{KL}\left(
\pi_\theta(\cdot|s)\,|\,\pi_{\theta'}(\cdot|s)
\right)
\le
\frac{1}{2W}\|\theta-\theta'\|_2^2 .
\]
\end{proof}

\begin{lemma}[KL logit inequality for softmax policy using simplex features]\label{lem:app:kl-logit_ineq}
Fix $s\in S$. 
If $\theta,\theta'\in\mathbb R^p$ satisfy
\[
\left|\frac{\mathrm d\pi_\theta}{\mathrm d\mu}(\cdot|s)\right|_{ B_b( A)}\le W^{-1},
\qquad
\left|\frac{\mathrm d\pi_{\theta'}}{\mathrm d\mu}(\cdot|s)\right|_{ B_b( A)}\le W^{-1}
\]
for some $W>0$ and the feature basis $g:S\times A \to \mathbb R^p$ takes values in $\Delta_{p-1}$ for every $s$ and $a$, i.e., $\sum_{i=1}^p g_i(s,a) = 1 $ and for all $s \in S,a \in A$, $g(s,a) \geq 0$, then for every $c\in\mathbb R$,
\[
\mathrm{KL}\left(\pi_\theta(\cdot|s)\,|\,\pi_{\theta'}(\cdot|s)\right)\le\frac{1}{2W}\left\|\theta-\theta'-c\mathbf 1\right\|_2^2 .
\]
\end{lemma}

\begin{proof}
Write $f_\theta(a|s):=\frac{\mathrm d\pi_\theta}{\mathrm d\mu}(a|s), f_{\theta'}(a|s):=\frac{\mathrm d\pi_{\theta'}}{\mathrm d\mu}(a|s).$
By the logit parametrization, both densities are strictly positive $\mu$-a.e.
Moreover, by assumption, $0<f_\theta(a|s),f_{\theta'}(a|s)\le W^{-1}, \mu\text{-a.e.}$

Let $F(f):=\int_{ A} f(a)\log f(a)\,\mu(\mathrm{d}a)$ be the negative entropy functional. 
By Proposition \ref{prp:app:Strong_convexity_of_negative_entropy}, $F$ is $W$--strongly convex on densities bounded by $W^{-1}$. 
Hence,
\[
F(f_{\theta'})
\ge F(f_\theta) + \int_{ A}\log f_\theta(a|s)\left(f_{\theta'}(a|s)-f_\theta(a|s)\right)\,\mu(\mathrm{d}a)+\frac{W}{2}\left\|f_{\theta'}-f_\theta\right\|_{L^2(A,\mu)}^2,
\]
where we used $\int_{ A}\left(f_{\theta'}(a|s)-f_\theta(a|s)\right)\,\mu(\mathrm{d}a)=0$ to remove the constant term in the first variation.

Therefore,
\begin{align*}
&\mathrm{KL}\left(
\pi_\theta(\cdot|s)\,|\,\pi_{\theta'}(\cdot|s)
\right) \\
&=
\int_{ A}
\log f_\theta(a|s) f_\theta(a|s)\,\mu(\mathrm{d}a)
-
\int_{ A}\log f_{\theta'}(a|s) f_\theta(a|s)\,\mu(\mathrm{d}a) \\
&\le
\int_{ A}
\left(\log f_\theta(a|s)-\log f_{\theta'}(a|s)\right)\left(f_\theta(a|s)-f_{\theta'}(a|s)\right)\,\mu(\mathrm{d}a) -
\frac{W}{2}\left\|f_{\theta'}-f_\theta\right\|_{L^2(A,\mu)}^2 .
\end{align*}

By the logit parametrization,
\[
\log f_\theta(a|s)-\log f_{\theta'}(a|s) =\langle \theta-\theta',g(s,a)\rangle -\log Z_\theta(s)+\log Z_{\theta'}(s).
\]
Note that $\langle \mathbf 1,g(s,a)\rangle=1.$
Thus, for any $c\in\mathbb R$, $\langle \theta-\theta',g(s,a)\rangle=\langle \theta-\theta'-c\mathbf 1,g(s,a)\rangle+c.$
Since both $f_\theta(\cdot|s)$ and $f_{\theta'}(\cdot|s)$ are probability densities, $\int_{ A} \left( f_\theta(a|s)-f_{\theta'}(a|s)\right)\,\mu(\mathrm{d}a)=0.$
Hence the constant term $c-\log Z_\theta(s)+\log Z_{\theta'}(s)$ vanishes after integration against
$f_\theta(\cdot|s)-f_{\theta'}(\cdot|s)$. Therefore,
\begin{align*}
&\mathrm{KL}\left(
\pi_\theta(\cdot|s)\,|\,\pi_{\theta'}(\cdot|s)
\right) \le\int_{ A}\langle \theta-\theta'-c\mathbf 1,g(s,a)\rangle\left(f_\theta(a|s)-f_{\theta'}(a|s)\right)\,\mu(\mathrm{d}a) -\frac{W}{2}\left\|f_{\theta'}-f_\theta\right\|_{L^2(A,\mu)}^2 .
\end{align*}
By the Cauchy--Schwarz inequality,
\begin{align*}
&\mathrm{KL}\left(
\pi_\theta(\cdot|s)\,|\,\pi_{\theta'}(\cdot|s)
\right) \le
\left\|\langle \theta-\theta'-c\mathbf 1,g(s,\cdot)\rangle\right\|_{L^2(A,\mu)}\left\|f_\theta-f_{\theta'}\right\|_{L^2(A,\mu)}-\frac{W}{2}\left\|f_\theta-f_{\theta'}\right\|_{L^2(A,\mu)}^2 .
\end{align*}
Using $xy-\frac{W}{2}y^2\le \frac{x^2}{2W}, x,y\ge 0,$ we get
\[
\mathrm{KL}\left(\pi_\theta(\cdot|s)\,|\,\pi_{\theta'}(\cdot|s)\right)
\le \frac{1}{2W}\left\|\langle \theta-\theta'-c\mathbf 1,g(s,\cdot)\rangle\right\|_{L^2(A,\mu)}^2 .
\]
Finally, by the bound on $g$,
\[
\left\|\langle \theta-\theta'-c\mathbf 1,g(s,\cdot)\rangle\right\|_{L^2(A,\mu)}
\le
\left\|\theta-\theta'-c\mathbf 1\right\|_2 .
\]
Therefore,
\[
\mathrm{KL}\left(\pi_\theta(\cdot|s)\,|\,\pi_{\theta'}(\cdot|s)\right)
\le
\frac{1}{2W}\left\|\theta-\theta'-c\mathbf 1\right\|_2^2 .
\]
\end{proof}

The following Lemma is used to prove the radial unbounded KL divergence in both feature regimes.
\begin{lemma}[Uniform radial KL growth over changing directions]
\label{lem:app:uniform-radial-kl}
Let $E$ be a nonzero finite-dimensional Euclidean space, let $h:A\to E$ be continuous, and let $\mu\in\mathcal P(A)$ have full support.
For $x\in E$, define
\[
\pi_x(\mathrm da) :=\frac{e^{\langle x,h(a)\rangle}} {\int_Ae^{\langle x,h(a')\rangle}\mu(\mathrm da')}\mu(\mathrm da).
\]
If
$\mu(\arg\max_{a\in A}\langle u,h(a)\rangle)=0$ for every unit $u\in E$, then
\[
\mathrm{KL}(\pi_x\,|\,\mu)\rightarrow+\infty \qquad\text{as }\|x\|_2\to\infty.
\]
\end{lemma}

\begin{proof}
For a unit vector $u\in E$, write $\phi_u(a)=\langle u,h(a)\rangle$, $m_u=\max_A\phi_u$, and $A_\varepsilon(u)=\{a:\phi_u(a)>m_u-\varepsilon\}$.  
Compactness of $A$ and continuity give a maximizer, while full support gives $\mu(A_{\varepsilon/2}(u))>0$. 
First, we establish that the softmax measure strictly concentrates on $A_\varepsilon$. 
The probability of sampling outside this set under $\pi_{ru}(\mathrm da) \propto \exp(r\phi_u(a))\mu(\mathrm da)$ is bounded by evaluating the worst-case numerator and restricting the denominator's domain to $A_{\varepsilon/2}$:
\[
\pi_{ru}(A_\varepsilon^c) = \frac{\int_{A_\varepsilon^c} e^{r \phi_u(a)} \mu(\mathrm{d}a)}{\int_A e^{r \phi_u(a)} \mu(\mathrm{d}a)} \le \frac{e^{r(m_u - \varepsilon)}}{\int_{A_{\varepsilon/2}} e^{r \phi_u(a)} \mu(\mathrm{d}a)} \le \frac{e^{r(m_u - \varepsilon)}}{e^{r(m_u - \varepsilon/2)} \mu(A_{\varepsilon/2})} = \frac{e^{-r\varepsilon/2}}{\mu(A_{\varepsilon/2})}.
\]

$\mu(A_{\varepsilon/2}) > 0$ is a strictly positive constant independent of $r$.
And we have $e^{-r\varepsilon/2} \to 0$ as $r \to \infty$. Thus, $\pi_{ru}(A_\varepsilon^c) \to 0$, which implies $\pi_{ru}(A_\varepsilon) \to 1$.

Next, we apply the partition inequality for KL divergence (see \citealp[Lemma~1.4.3]{dupuis2011weak}).
By splitting the action space into the binary partition $\{A_\varepsilon, A_\varepsilon^c\}$, the KL divergence is bounded below by 
\[
\mathrm{KL}(\pi_{ru} | \mu) \ge \pi_{ru}(A_\varepsilon) \log \frac{\pi_{ru}(A_\varepsilon)}{\mu(A_\varepsilon)} + \pi_{ru}(A_\varepsilon^c) \log \frac{\pi_{ru}(A_\varepsilon^c)}{\mu(A_\varepsilon^c)}.
\]

Taking the limit inferior as $r \to \infty$, and substituting $\pi_{ru}(A_\varepsilon ) \to 1$ and $\pi_{ru}(A_\varepsilon^c) \to 0$, we obtain:
\begin{equation}
\label{eq:lim_inf_kl}
\liminf_{r \to \infty} \operatorname{KL}(\pi_{ru} | \mu) \ge 1 \cdot \log \frac{1}{\mu(A_\varepsilon)} + 0 = -\log \mu(A_\varepsilon).
\end{equation}

This lower bound holds for every sufficiently small $\varepsilon > 0$. By the continuity of $h(\cdot)$, the intersection of these sets as $\varepsilon \downarrow 0$ is exactly the maximizer set $\arg \max(\phi_u)$, i.e., $\bigcap_{\varepsilon > 0} A_\varepsilon = \arg \max(\phi_u).$

Moreover, the family $\{A_\varepsilon\}_{\varepsilon > 0}$ is decreasing as $\varepsilon \downarrow 0$.
By continuity from above of the probability measure $\mu$,
\[
\lim_{\varepsilon \downarrow 0} \mu(A_\varepsilon) = \mu\!\left( \bigcap_{\varepsilon > 0} A_\varepsilon \right) = \mu(\arg \max(\phi_u))=0,
\]
hence the right-hand side of~\eqref{eq:lim_inf_kl} tends to $+\infty$. 
Thus the radial divergence holds for each fixed $u$.

It remains to extend the preceding pointwise radial divergence to a statement that is uniform over directions.

For $u\in\mathbb S_E = \{u \in E: \|u\|_2=1\}$, define $H(r,u):= \operatorname{KL}(\pi_{ru}|\mu).$
Since $\log\frac{\mathrm d\pi_{ru}}{\mathrm d\mu}(a) = r\phi_u(a)-Z_u(r),$ where $Z_u(r) = \log \int_A \exp(r \phi_u(a)) \mu(\mathrm{d}a)$, bounded differentiation under the integral gives $Z_u'(r) = \int_A\phi_u(a)\pi_{ru}(\mathrm da)$ and therefore
\begin{equation}
H(r,u) = r Z_u'(r)- Z_u(r).
\label{eq:radial-kl-log-partition}
\end{equation}
Differentiating once more yields
\begin{equation}
\partial_rH(r,u) = r Z_u''(r)  = r\left[ \int_A\phi_u(a)^2\pi_{ru}(\mathrm da) - \left( \int_A\phi_u(a)\pi_{ru}(\mathrm da)\right)^2\right]=r\operatorname{Var}_{\pi_{ru}}(\phi_u)\geq0.
\label{eq:radial-kl-monotone}
\end{equation}
Thus $r\mapsto H(r,u)$ is nondecreasing for every fixed direction $u$.

For fixed $r$, the map $u\mapsto H(r,u)$ is continuous on $\mathbb S_E$.  
Indeed, if $u_n\to u$, then boundedness of the feature map implies
\[
\sup_{a\in A}
|\phi_{u_n}(a)-\phi_u(a)| \leq \|u_n-u\|_2 \sup_{a\in A}\|h(a)\|_2 \rightarrow0.
\]
Hence bounded convergence applies both to $e^{r\phi_{u_n}}$ and to $\phi_{u_n}e^{r\phi_{u_n}}$, showing that $Z_{u_n}(r)\to Z_u(r), Z_{u_n}'(r)\to Z_u'(r).$
The continuity of $H(r,\cdot)$ now follows from \eqref{eq:radial-kl-log-partition}.

Fix $M>0$ and define $O_r:= \left\{ u\in\mathbb S_E:H(r,u)>M \right\}.$
Each $O_r$ is open because $H(r,\cdot)$ is continuous.
Moreover, \eqref{eq:radial-kl-monotone} implies $r_1\leq r_2 \Longrightarrow O_{r_1}\subseteq O_{r_2}.$
By the pointwise radial-divergence result already proved, for every $u\in\mathbb S_E$ there exists some radius $r(u)$ such that $H(r(u),u)>M$. 
Consequently, $\mathbb S_E=\bigcup_{r\geq0}O_r.$

Since $E$ is finite-dimensional, $\mathbb S_E$ is compact. 
The preceding open cover therefore has a finite subcover: $\mathbb S_E = O_{r_1}\cup\cdots\cup O_{r_N}.$
Set $r_M:=\max_{1\leq i\leq N}r_i.$
Because the family $(O_r)_{r\geq0}$ is increasing, $O_{r_i}\subseteq O_{r_M}$ for every $i$, and hence $\mathbb S_E = \bigcup_{i=1}^N O_{r_i} \subseteq O_{r_M}.$
Thus $O_{r_M}=\mathbb S_E$.  
By monotonicity, $O_r=\mathbb S_E$ for every $r\geq r_M$, so $H(r,u)>M \ \text{for every }u\in\mathbb S_E \text{ and every }r\geq r_M.$

For each fixed $r$, continuity of $H(r,\cdot)$ and compactness of $\mathbb S_E$ imply that the minimum over $u$ is attained. 
Since every value is strictly larger than $M$, the minimum is also strictly larger than $M$: $\inf_{u\in\mathbb S_E}H(r,u)>M, \text{for every } r\geq r_M.$

As $M>0$ was arbitrary, we conclude that $\lim_{r\to\infty} \inf_{u\in\mathbb S_E} \operatorname{KL}(\pi_{ru}|\mu) = +\infty.$

In particular, this conclusion remains valid along every sequence $r_n\to\infty$ and every possibly varying sequence of directions $u_n\in\mathbb S_E$.
\end{proof}

\begin{lemma}[Concavity of variance]\label{lem:app:concavity_of_var}
Let $P=qP_1+(1-q)P_2$ with $q\in[0,1]$, where $P_1$ and $P_2$ are
probability measures. 
If $X\in L^2(P_1)\cap L^2(P_2)$, then
\[
\operatorname{Var}_P(X)
\ge
q\operatorname{Var}_{P_1}(X)
+
(1-q)\operatorname{Var}_{P_2}(X).
\]
\end{lemma}

\begin{proof}
Let $\bar{\mu}_i=\mathbb E_{P_i}[X]$ for $i=1,2$. Since
$P=qP_1+(1-q)P_2$, $\mathbb E_P[X]=q\bar{\mu}_1+(1-q)\bar{\mu}_2, \mathbb E_P[X^2]=q\mathbb E_{P_1}[X^2]+(1-q)\mathbb E_{P_2}[X^2].$
Therefore,
\[
\begin{aligned}
\operatorname{Var}_P(X)
&=
q\operatorname{Var}_{P_1}(X)
+
(1-q)\operatorname{Var}_{P_2}(X)
+
q(1-q)(\bar{\mu}_1-\bar{\mu}_2)^2 \\
&\ge
q\operatorname{Var}_{P_1}(X)
+
(1-q)\operatorname{Var}_{P_2}(X).
\end{aligned}
\]
\end{proof}

\begin{lemma}\label{lemapp:mdp:lower_bound_log_density}
Let $\theta^\prime \in \mathbb{R}^{p}$.
For any $\pi \in \Pi_\mu$, if there exists a constant $R > 0$ such that $\left|\log \frac{\mathrm d\pi}{\mathrm d\mu}\right|_{B_b({S \times A})} \leq  R$, then there exists 
\[
\lambda_{\text{fsa}} :=  e^{-R - 2\|\theta^\prime\|_2} \int_{S}  \lambda_{\min} \left(G^{\pi_{\theta^\prime}}(s) \right) \rho(\mathrm{d}s)
\]
such that $\int_{S} \lambda_{\min} (G^\pi(s)) \rho(\mathrm{d}s) \geq \lambda_{\text{fsa}}$.
\end{lemma}

\begin{proof}
Write $Z_{\pi_\theta}(s) := \int_A \exp(\theta^\top g(s,a))\mu(\mathrm d a)$.
For all $s \in S, a \in A$,
\begin{align*}
\left|\log\frac{\mathrm d \pi_{\theta^\prime}}{\mathrm d \mu}(a|s)\right|
& = \left| \langle \theta^\prime, {g(s,a)}\rangle - \log Z_{\pi_{\theta^\prime}}(s)\right| \\
& \leq \sup_{s,a}\|{g(s,a)}\|_2\|\theta^\prime\|_2 + \left|\log \int_{A} e^{\langle \theta^\prime, {g(s,a)}\rangle}\mu(\mathrm{d}a)\right|\\
& \leq 2\|\theta^\prime\|_2,
\end{align*}
so $\left|\frac{\mathrm d \pi_{\theta^\prime}}{\mathrm d \mu}\right|_{B_b(S \times A)} \leq e^{2\|\theta^\prime\|_2}$.
Since $\inf_{s \in S, a\in A} \frac{\mathrm d\pi}{\mathrm d\pi_{\theta^\prime}}(a|s) = \inf_{s \in S, a\in A} \frac{\mathrm d\pi}{\mathrm d\mu} \frac{\mathrm d\mu}{\mathrm d \pi_{\theta^\prime}}(a|s) \geq e^{-R - 2\|\theta^\prime\|_2}  > 0$, then we know $e^{-R - 2\|\theta^\prime\|_2}  < \int_{A} \frac{\mathrm d\pi}{\mathrm d \pi_{\theta^\prime}}(a|s)\pi_{\theta^\prime}(\mathrm{d}a|s) = 1$ since $R >0$, and
\[
\pi(\mathrm{d}a|s) = e^{-R - 2\|\theta^\prime\|_2}  \pi_{\theta^\prime}(\mathrm{d}a|s) + \left(1-e^{-R - 2\|\theta^\prime\|_2} \right) \bar{\pi}(\mathrm{d}a|s),
\]
where $\bar{\pi}(\mathrm{d}a|s) = \frac{\pi(\mathrm{d}a|s)- e^{-R - 2\|\theta^\prime\|_2} \pi_{\theta^\prime}(\mathrm{d}a|s)}{1-e^{-R - 2\|\theta^\prime\|_2} }$ determines a valid probability measure for fixed $s$.

Hence by Lemma \ref{lem:app:concavity_of_var}, for any unit vector $v$,
\begin{align*} 
&\mathrm{Var}_{a\sim\pi(\cdot | s)}(v^{\top}{g(s,a)})\\
& \geq e^{-R - 2\|\theta^\prime\|_2} \mathrm{Var}_{a\sim\pi_{\theta^\prime}(\cdot | s)}(v^{\top}{g(s,a)})  + \left(1-e^{-R - 2\|\theta^\prime\|_2} \right)\mathrm{Var}_{a\sim\bar{\pi}(\cdot | s)}(v^{\top}{g(s,a)}) \\
& \geq e^{-R - 2\|\theta^\prime\|_2} \mathrm{Var}_{a\sim\pi_{\theta^\prime}(\cdot | s)}(v^{\top}{g(s,a)})\\
& =   e^{-R-2\|\theta^\prime\|_2} v^{\top} G^{\pi_{\theta^\prime}}(s) v\\
& \geq e^{-R - 2\|\theta^\prime\|_2}  \lambda_{\min} \left(G^{\pi_{\theta^\prime}}(s)\right) ,
\end{align*}
where in the equality we use the Definition \ref{def:app:mdp:FIM} of FIM,
{\small
\begin{align*}
&\mathrm{Var}_{a\sim\pi_{\theta^\prime}}(v^{\top}{g(s,a)}) \\
&=\int_A \left(v^{\top}{g(s,a)} - \int_A v^{\top} {g(s, a^\prime)} \pi_{\theta^\prime} ( \mathrm{d}a^\prime | s) \right)\left(v^{\top}{g(s,a)} - \int_A v^{\top} {g(s, a^\prime)} \pi_{\theta^\prime} (\mathrm{d}a^\prime | s) \right)\pi_{\theta^\prime}(\mathrm{d}a|s)\\
& = v^{\top} G^{\pi_{\theta^\prime}}(s) v,
\end{align*}
}
which implies that,
\begin{align*}
\lambda_{\min}(G^\pi (s))& = \min_{\|v\|_2 = 1} v^\top G^\pi(s) v \quad \\
& = \min_{\|v\|_2 = 1} \mathrm{Var}_{a\sim\pi(\cdot | s)}(v^{\top}{g(s,a)})\\
& \geq e^{-R - 2\|\theta^\prime\|_2} \lambda_{\min} \left(G^{\pi_{\theta^\prime}}(s)\right),
\end{align*}
then after taking integration by $\rho$ over $S$, there exists 
\[
\lambda_{\text{fsa}} = e^{-R - 2\|\theta^\prime\|_2} \int_{S} \lambda_{\min} \left(G^{\pi_{\theta^\prime}}(s) \right) \rho(\mathrm{d}s)
\]
such that $\int_{S}\lambda_{\min} (G^\pi(s))\rho(\mathrm{d}s) \geq \lambda_{\text{fsa}}$.
\end{proof}

\begin{lemma}\label{lem:app:mdp:lower_bound_log_density_simplex}
Let $\theta^\prime \in \mathbb{R}^p$.
For any $\pi \in \Pi_\mu$, if there exists a constant $R > 0$ such that $\left|\log \frac{\mathrm d\pi}{\mathrm d\mu}\right|_{B_b({S \times A})} \leq  R$, then there exists 
\[
\lambda_{\text{smp}} :=  e^{-R - 2\|\theta^\prime\|_2} \int_{S}  \lambda_{\min}\left(\int g(s,a)g^\top(s,a)\pi_{\theta^\prime} (\mathrm{d}a|s)\right)\rho(\mathrm{d}s)
\]
such that $\int\lambda_{\min} \left(\int g(s,a)g^\top(s,a)\pi (\mathrm{d}a | s)\right)\rho(\mathrm{d}s) \geq \lambda_{\text{smp}}$.
\end{lemma}

\begin{proof}
Write $Z_{\pi_\theta}(s) := \int_A \exp(\theta^\top g(s,a))\mu(\mathrm d a)$.
For all $s \in S, a \in A$,
\begin{align*}
\left|\log\frac{\mathrm d \pi_{\theta^\prime}}{\mathrm d \mu}(a|s)\right|
& = \left| \langle \theta^\prime, g(s,a)\rangle - \log Z_{\pi_{\theta^\prime}}(s)\right| \\
& \leq \sup_{s,a}\|g(s,a)\|_2\|\theta^\prime\|_2 + \left|\log \int e^{\langle \theta^\prime,g(s,a)\rangle}\mu(\mathrm{d}a)\right|\\
& \leq 2\|\theta^\prime\|_2,
\end{align*}
so $\left|\frac{\mathrm d \pi_{\theta^\prime}}{\mathrm d \mu}\right|_{B_b(S \times A)} \leq e^{2\|\theta^\prime\|_2}$.

Since $\inf_{s \in S, a\in A} \frac{\mathrm d\pi}{\mathrm d\pi_{\theta^\prime}}(a|s) = \inf_{s \in S, a\in A} \frac{\mathrm d\pi}{\mathrm d\mu} \frac{\mathrm d\mu}{\mathrm d \pi_{\theta^\prime}}(a|s) \geq e^{-R - 2\|\theta^\prime\|_2} >0$, then we know $e^{-R - 2\|\theta^\prime\|_2} < \int_{A} \frac{\mathrm d\pi}{\mathrm d \pi_{\theta^\prime}}(a|s)\pi_{\theta^\prime}(\mathrm{d}a|s) = 1$ since $R >0$, and
\[
\pi(\mathrm{d}a|s) = e^{-R - 2\|\theta^\prime\|_2} \pi_{\theta^\prime}(\mathrm{d}a|s) + \left(1-e^{-R - 2\|\theta^\prime\|_2}\right) \bar{\pi}(\mathrm{d}a|s),
\]
where $\bar{\pi}(\mathrm{d}a|s) = \frac{\pi(\mathrm{d}a|s)- e^{-R - 2\|\theta^\prime\|_2}\pi_{\theta^\prime}(\mathrm{d}a|s)}{1-e^{-R - 2\|\theta^\prime\|_2}}$ determines a valid probability measure for fixed $s$.

Hence by positive semidefiniteness,
\begin{align*}
&\lambda_{\min}\left(\int_{A} g(s,a)g^\top(s,a)\pi(\mathrm{d}a|s)\right)\\
& = \min_{\|v\|_2 = 1} v^\top \left(\int_{A} g(s,a)g^\top(s,a)\pi(\mathrm{d}a|s)\right) v\\
& = \min_{\|v\|_2 = 1} v^\top \left(\int_{A} g(s,a)g^\top(s,a) \left(e^{-R - 2\|\theta^\prime\|_2}\pi_{\theta^\prime}(\mathrm{d}a | s) + \left(1-e^{-R - 2\|\theta^\prime\|_2}\right) \bar{\pi}(\mathrm{d}a|s)\right)\right) v\\
& \geq e^{-R - 2\|\theta^\prime\|_2} \min_{\|v\|_2 = 1} v^\top \left(\int_{A} g(s,a)g^\top(s,a) \pi_{\theta^\prime}(\mathrm{d}a | s)\right) v \\
& = e^{-R - 2\|\theta^\prime\|_2} \lambda_{\min}\left(\int_{A} g(s,a)g^\top(s,a)\pi_{\theta^\prime} (\mathrm{d}a | s)\right),
\end{align*}
then after taking integration by $\rho$ over $S$, there exists 
\[
\lambda_{\text{smp}} = e^{-R - 2\|\theta^\prime\|_2} \int_{S} \lambda_{\min}\left(\int_{A} g(s,a)g^\top(s,a)\pi_{\theta^\prime} (\mathrm{d}a | s)\right) \rho(\mathrm{d}s)
\]
such that $\int_{S}\lambda_{\min}\left(\int_{A} g(s,a)g^\top(s,a)\pi (\mathrm{d}a | s)\right) \rho(\mathrm ds)\geq \lambda_{\text{smp}}$. 
\end{proof}

\begin{proposition}\label{prp:app:mdp:singularity_of_fim_mdp_simplex}
Consider the feature basis $g:S\times A \to \mathbb R^p$ that takes values in $\Delta_{p-1}$ for every $s$ and $a$, i.e., $\sum_{i=1}^p g_i(s,a) = 1 $ and for all $s \in S,a \in A$, $g(s,a) \geq 0$.
For every $s \in S$ and $\theta \in \mathbb{R}^p$, $\lambda_{\min} (G^{\pi_\theta}(s)) = 0$ with eigenvector $\textbf{1} \in \mathbb{R}^p$.
\end{proposition}
\begin{proof}
Define $\bar{g}_\theta (s):= \int_A g(s,a) \pi_\theta(\mathrm{d}a|s)$, and note that $\bar{g}_\theta^\top (s)\textbf{1} = \int g(s,a)^\top\mathbf 1\pi_\theta(\mathrm{d}a|s) =1$ for every $ s \in S$, then we have
\begin{align*}
G^{\pi_\theta} (s)\textbf{1} &= \int_A \nabla\log\frac{\mathrm d \pi_{\theta}}{\mathrm d \mu}(a|s)\left(\nabla\log\frac{\mathrm d \pi_{\theta}}{\mathrm d \mu}(a|s)\right)^{\top}\mathbf 1\pi_\theta(\mathrm{d}a|s)\\
& = \int_A \left(g(s,a) - \bar{g}_\theta(s)\right)\left(g(s,a) - \bar{g}_\theta(s)\right)^{\top}\mathbf 1\pi_\theta(\mathrm{d}a|s)= \mathbf{0}.
\end{align*}
\end{proof}

We next recall eigenvalue interlacing for a real symmetric matrix under a rank-one perturbation (see \citealp[Section~5]{golub1973some}).
\begin{theorem}[Eigenvalue interlacing]
\label{thm:app:bandit:Eigenvalue_Interlacing}
Let $B\in\mathbb R^{n\times n}$ be symmetric, let $x\in\mathbb R^n$, and set $A=B+xx^\top$.
Let
\[
\eta_1\geq\cdots\geq\eta_n,\qquad\lambda_1\geq\cdots\geq\lambda_n
\]
be the eigenvalues of $A$ and $B$, respectively.
Then
\[
\eta_i\geq\lambda_i\geq\eta_{i+1},\qquad i=1,\ldots,n-1,
\]
and $\eta_n\geq\lambda_n$.
\end{theorem}

The following Corollary \ref{cor:app:mdp:spectrum_of_fim} is similar to Lemma 23 in~\citet{mei2020global}.

\begin{corollary}\label{cor:app:mdp:spectrum_of_fim}
Consider the feature basis $g:S\times A \to \mathbb R^p$ that takes values in $\Delta_{p-1}$ for every $s$ and $a$, i.e., $\sum_{i=1}^p g_i(s,a) = 1 $ and for all $s \in S,a \in A$, $g(s,a) \geq 0$.
Then for every $s \in {S}$, and every $y, \theta \in \mathbb{R}^p$, 
\[
{\left\|G^{\pi_\theta}(s)\left(y-\frac{y^{\top} \mathbf{1}}{p} \cdot \mathbf{1}\right)\right\|_2 \geq \lambda_{\min}\left(\int_{A} g(s,a)g^\top(s, a)\pi_\theta(\mathrm{d}a|s)\right)\cdot\left\|y-\frac{y^{\top} \mathbf{1}}{p} \cdot \mathbf{1}\right\|_2},
\]
and
\[
\left(y-\frac{y^{\top} \mathbf{1}}{p} \cdot \mathbf{1}\right)^\top G^{\pi_\theta}(s)\left(y-\frac{y^{\top} \mathbf{1}}{p} \cdot \mathbf{1}\right) \geq \lambda_{\min}\left(\int_{A} g(s,a)g^\top(s, a)\pi_\theta(\mathrm{d}a|s)\right)\cdot\left\|y-\frac{y^{\top} \mathbf{1}}{p} \cdot \mathbf{1}\right\|_2^2.
\]
\end{corollary}
\begin{proof}
Consider $p \geq 2$ since $p=1$ is trivial.
Let $\eta_i, i=1,2, \ldots, p$ be the eigenvalues of matrix $\int_{A} g(s,a)g^\top(s, a)\pi_\theta(\mathrm{d}a|s)$ in the following order $\eta_1 \geq \eta_2 \geq \cdots \geq \eta_p,$ and denote the eigenvalues of $G^{\pi_\theta}(s)$ as $\lambda_1 \geq \lambda_2 \geq \cdots \geq \lambda_p .$
Then Theorem \ref{thm:app:bandit:Eigenvalue_Interlacing} indicates that, $\eta_{i+1} \leq \lambda_{i} \leq  \eta_i, i=1,2,\ldots, p - 1.$

Since $G^{\pi_\theta}(s)$ is real symmetric, its eigenvectors $\left\{\frac{\mathbf 1}{\sqrt{p}}, v_{p-1}, \ldots, v_1\right\}$ are orthonormal. For any vector $y$, $y$ can be written as linear combination of eigenvectors of $G^{\pi_\theta}(s)$,
\[
\begin{aligned}
y & =a_1 \cdot \frac{\mathbf 1}{\sqrt{p}}+a_{2}v_{p-1}+\cdots+a_p v_1  =\frac{y^{\top} \mathbf{1}}{p} \cdot \mathbf{1}+a_2 v_{p-1}+\cdots+a_p v_1.
\end{aligned}
\]
The last equation is because the representation is unique, and
$a_1=y^{\top} \frac{\mathbf{1}}{\sqrt{p}}=\frac{y^{\top} \mathbf{1}}{\sqrt{p}}.$
Denote $y^{\prime} :=y-\frac{y^{\top} \mathbf{1}}{p} \cdot \mathbf{1}=a_{2} v_{p-1}+\cdots+a_p v_1.$
We have $\left\|y^{\prime}\right\|^2=a_2^2+\cdots+a_{p}^2.$
Therefore because of $\lambda_{p-1}\geq \eta_{p}= \lambda_{\min}\left(\int_{A} g(s,a)g^\top(s,a)\pi_\theta(\mathrm{d}a|s)\right)$,
\[
\begin{aligned}
\left\|G^{\pi_\theta}(s) y^{\prime}\right\|_2 & =\left(a_{2}^2 \lambda_{p-1}^2+\cdots+a_p^2 \lambda_1^2\right)^{\frac{1}{2}} \\
& \geq\left(\left(a_2^2+\cdots+a_{p}^2\right) \cdot \lambda_{p-1}^2\right)^{\frac{1}{2}} \\
& =\lambda_{p-1}\cdot\left\|y^{\prime}\right\|_2 \\
& \geq \lambda_{\min}\left(\int_{A} g(s,a)g^\top(s,a)\pi_\theta(\mathrm{d}a|s)\right) \cdot\left\|y^{\prime}\right\|_2 .
\end{aligned}
\]
Since $y' \perp \mathbf 1$, we also know
\[
(y')^\top G^{\pi_\theta}(s) y' \geq \lambda_{p-1}\|y'\|^2_2 \geq \lambda_{\min}\left(\int_{A} g(s,a)g^\top(s,a)\pi_\theta(\mathrm{d}a|s)\right) \left\|y^{\prime}\right\|_2^2.
\]
\end{proof}

\subsubsection{Proof of Proposition \ref{prp:app:mdp:Gradient_of_the_Objective}}
\label{sec:add_proof}

\begin{proof}
Fix $\theta\in\mathbb R^p$. 
For an arbitrary increment $h\in\mathbb R^p$, put $\theta^t=\theta+th$, $t\in[0,1]$.
For brevity within this proof, write
\[
\psi_\eta(s,a):=\nabla\log\frac{\mathrm d\pi_\eta}{\mathrm d\mu}(a|s),\qquad \eta\in\mathbb R^p.
\]
By Lemmas~\ref{lem:app:mdp:grad_of_log_policy} and~\ref{lem:app:mdp:bounded_grad_of_log_policy},
\[
\psi_\eta(s,a)=g(s,a)-\int_A g(s,a')\pi_\eta(\mathrm da'|s),\qquad\|\psi_\eta(s,a)\|_2\le 2.
\]
In particular,
\[
\int_A\psi_\eta(s,a)\pi_\eta(\mathrm da|s)=0.
\]

\medskip
\noindent\textbf{Step 1: a uniform first-order expansion of the policy.}
Since $\sup_{s,a}\|g(s,a)\|_2\le1$, the score formula and \eqref{ineq:Lip_pi} give
\begin{align*}
\sup_{s,a}\|\psi_{\theta^t}(s,a)-\psi_\theta(s,a)\|_2
&\le |\pi_{\theta^t}-\pi_\theta|_{b\mathcal K(A|S)}\le 2t\|h\|_2.
\end{align*}
Differentiating the explicit softmax density yields
\[
\frac{\mathrm d}{\mathrm dt}\frac{\mathrm d\pi_{\theta^t}}{\mathrm d\mu}(a|s)=\langle\psi_{\theta^t}(s,a),h\rangle\frac{\mathrm d\pi_{\theta^t}}{\mathrm d\mu}(a|s).
\]
Consequently, integrating in $t$ gives the signed-kernel identity
\[
(\pi_{\theta+h}-\pi_\theta)(\mathrm da|s)=\int_0^1\langle\psi_{\theta^t}(s,a),h\rangle\pi_{\theta^t}(\mathrm da|s)\,\mathrm dt.
\]
Here and below, $\langle\psi_\theta,h\rangle\pi_\theta$ denotes the signed kernel $\langle\psi_\theta(s,a),h\rangle\pi_\theta(\mathrm da|s)$.
Subtracting this kernel and using \eqref{ineq:Lip_pi}, we obtain
\begin{align}
&\left|
\pi_{\theta+h}-\pi_\theta-\langle\psi_\theta,h\rangle\pi_\theta\right|_{b\mathcal K(A|S)}\notag\\
&\quad\le\int_0^1\left(\|h\|_2\sup_{s,a}\|\psi_{\theta^t}(s,a)-\psi_\theta(s,a)\|_2+2\|h\|_2|\pi_{\theta^t}-\pi_\theta|_{b\mathcal K(A|S)}\right)\mathrm dt\notag\\
&\quad\le\int_0^1 6t\|h\|_2^2\,\mathrm dt=3\|h\|_2^2.
\label{eq:pg-policy-first-order-remainder}
\end{align}
All these bounds are uniform in the state. The differentiations and integrations above are justified directly by boundedness of $g$: along the segment, the density is at most $\exp(2(\|\theta\|_2+\|h\|_2))$, and the score is bounded by $2$.

\noindent\textbf{Step 2: a uniform quadratic bound on the KL term.}
For each $s\in S$, the log-density ratio has the form
\[
\log\frac{\mathrm d\pi_{\theta+h}}{\mathrm d\pi_\theta}(a|s)=\langle h,g(s,a)\rangle-\log\frac{\int_A e^{\langle\theta+h,g(s,a')\rangle}\mu(\mathrm da')}{\int_A e^{\langle\theta,g(s,a')\rangle}\mu(\mathrm da')}.
\]
The second term does not depend on $a$, and
$(\pi_{\theta+h}-\pi_\theta)(A|s)=0$. Hence
\begin{align*}
&\mathrm{KL}\bigl(\pi_{\theta+h}(\cdot|s)\mid\pi_\theta(\cdot|s)\bigr)+\mathrm{KL}\bigl(\pi_\theta(\cdot|s)\mid\pi_{\theta+h}(\cdot|s)\bigr)\\
&=\int_A\langle h,g(s,a)\rangle(\pi_{\theta+h}-\pi_\theta)(\mathrm da|s)\\
&\le\|h\|_2\,|\pi_{\theta+h}(\cdot|s)-\pi_\theta(\cdot|s)|_{\mathcal M(A)}\le 2\|h\|_2^2.
\end{align*}
By nonnegativity of relative entropy, this implies
\begin{equation}
\sup_{s\in S}
\mathrm{KL}\bigl(\pi_{\theta+h}(\cdot|s)\mid
\pi_\theta(\cdot|s)\bigr)
\le 2\|h\|_2^2.
\label{eq:pg-kl-quadratic-remainder}
\end{equation}

\medskip
\noindent\textbf{Step 3: differentiability of the objective.}
For brevity, put
\[
B_\theta(s,a):=Q_\tau^{\pi_\theta}(s,a)+\tau\log\frac{\mathrm d\pi_\theta}{\mathrm d\mu}(a|s),\qquad M_\theta:=|B_\theta|_{B_b(S\times A)}.
\]
The softmax formula gives
$|\log(\mathrm d\pi_\theta/\mathrm d\mu)|_{B_b(S\times A)}\le2\|\theta\|_2$.
Thus Lemma~\ref{lem:bounds} ensures that $M_\theta<\infty$.
Applying Lemma~\ref{lem:app:mdp:performance_diff} with $\pi=\pi_{\theta+h}$ and $\pi'=\pi_\theta$, we obtain
\begin{align*}
&V_\tau^{\pi_{\theta+h}}(\rho)-V_\tau^{\pi_\theta}(\rho)\\
&\quad=\frac{1}{1-\gamma}\int_S\Bigg[\int_A B_\theta(s,a)(\pi_{\theta+h}-\pi_\theta)(\mathrm da|s)+\tau\,\mathrm{KL}\bigl(\pi_{\theta+h}(\cdot|s)\mid\pi_\theta(\cdot|s)\bigr)\Bigg]d_\rho^{\pi_{\theta+h}}(\mathrm ds).
\end{align*}
Define the bounded linear functional
\[
L_\theta[h]:=\frac{1}{1-\gamma}\int_S\int_AB_\theta(s,a)\langle\psi_\theta(s,a),h\rangle\pi_\theta(\mathrm da|s)d_\rho^{\pi_\theta}(\mathrm ds).
\]
By \eqref{eq:pg-policy-first-order-remainder}, replacing the policy increment by its linear part creates an error bounded by $3M_\theta\|h\|_2^2$, uniformly in $s$.
The KL contribution is bounded by $2\tau\|h\|_2^2$ by \eqref{eq:pg-kl-quadratic-remainder}.
Moreover, Corollary~\ref{cor:app:mdp:Lip_occ} implies
\[
|d_\rho^{\pi_{\theta+h}}-d_\rho^{\pi_\theta}|_{\mathcal M(S)}\le\frac{2\gamma}{1-\gamma}\|h\|_2,
\]
and the linearized integrand satisfies
\[
\sup_{s\in S}\left|\int_A B_\theta(s,a)\langle\psi_\theta(s,a),h\rangle\pi_\theta(\mathrm da|s)\right|\le2M_\theta\|h\|_2.
\]
Combining these three estimates gives
\begin{align*}
&\left|V_\tau^{\pi_{\theta+h}}(\rho)-V_\tau^{\pi_\theta}(\rho)-L_\theta[h]\right|\le\left(\frac{3M_\theta+2\tau}{1-\gamma}+\frac{4\gamma M_\theta}{(1-\gamma)^2}\right)\|h\|_2^2.
\end{align*}
Since the coefficient is finite and independent of $h$, the map $\theta\mapsto V_\tau^{\pi_\theta}(\rho)$ is Fr\'echet differentiable at $\theta$, with derivative  $L_\theta$.
This functional yields
\begin{align*}
\nabla V_\tau^{\pi_\theta}(\rho)
&=\frac{1}{1-\gamma}\int_S\int_A\left(Q_\tau^{\pi_\theta}(s,a)+\tau\log\frac{\mathrm d\pi_\theta}{\mathrm d\mu}(a|s)\right)\nabla\log\frac{\mathrm d\pi_\theta}{\mathrm d\mu}(a|s)\pi_\theta(\mathrm da|s)d_\rho^{\pi_\theta}(\mathrm ds).
\end{align*}
\end{proof}

\subsubsection{Proof of Proposition \ref{prp:app:mdp:local_lipchitiz_conitinuity_mdp}}
\label{sec:proof_local_lip}

\begin{proof}
Fix $\theta,\theta'\in\mathbb R^p$ satisfying the two log-density bounds.
From~\eqref{ineq:Lip_pi}, we have
\begin{align}
|\pi_{\theta'}-\pi_\theta|_{b\mathcal K(A|S)}
&\le 2\|\theta'-\theta\|_2.
\label{eq:local-lip-policy-bound}
\end{align}
Lemma~\ref{lem:app:mdp:lip_occ_kernel} therefore yields
\begin{equation}
|d_\rho^{\pi_\theta}-d_\rho^{\pi_{\theta'}}|_{\mathcal M(S)}
\le\frac{\gamma}{1-\gamma}|\pi_\theta-\pi_{\theta'}|_{b\mathcal K(A|S)}\le\frac{2\gamma}{1-\gamma}\|\theta-\theta'\|_2.
\label{eq:local-lip-occupancy-bound}
\end{equation}

By Lemma~\ref{lem:bounds},
\[
|V_\tau^{\pi_\theta}|_{B_b(S)}\le\frac{1+\tau R}{1-\gamma},\qquad \left|Q_\tau^{\pi_\theta}+\tau\log\frac{\mathrm d\pi_\theta}{\mathrm d\mu}\right|_{B_b(S\times A)}\le\frac{1+\tau R}{1-\gamma}.
\]
The same bounds hold with $\theta$ replaced by $\theta'$.
For each $s\in S$, nonnegativity of relative entropy gives
\begin{align*}
\int_A\log\frac{\mathrm d\pi_{\theta'}}{\mathrm d\mu}(a|s)(\pi_\theta-\pi_{\theta'})(\mathrm da|s)&\le\mathrm{KL}\bigl(\pi_\theta(\cdot|s)\mid\mu\bigr)-\mathrm{KL}\bigl(\pi_{\theta'}(\cdot|s)\mid\mu\bigr)\\
& \le\int_A\log\frac{\mathrm d\pi_\theta}{\mathrm d\mu}(a|s)(\pi_\theta-\pi_{\theta'})(\mathrm da|s).
\end{align*}
Consequently, the two endpoint log-density bounds imply
\[
\left|\mathrm{KL}\bigl(\pi_\theta(\cdot|s)\mid\mu\bigr)-\mathrm{KL}\bigl(\pi_{\theta'}(\cdot|s)\mid\mu\bigr)\right|\le R\,|\pi_\theta(\cdot|s)-\pi_{\theta'}(\cdot|s)|_{\mathcal M(A)}.
\]
Subtracting the two policy Bellman equations \eqref{eq:intro:policy_Bellman}, using $|c|_{B_b(S\times A)}\le1$, and then applying \eqref{eq:local-lip-policy-bound}, we obtain
\begin{align*}
(1-\gamma)
|V_\tau^{\pi_\theta}-V_\tau^{\pi_{\theta'}}|_{B_b(S)}
&\le\left(1+\tau R+\gamma|V_\tau^{\pi_{\theta'}}|_{B_b(S)}\right)|\pi_\theta-\pi_{\theta'}|_{b\mathcal K(A|S)}\le \frac{2(1+\tau R)}{1-\gamma}\|\theta-\theta'\|_2.
\end{align*}
Together with the definition~\eqref{def:intro:Qpi}, this gives
\begin{equation}
\begin{aligned}
|V_\tau^{\pi_\theta}-V_\tau^{\pi_{\theta'}}|_{B_b(S)}&\le\frac{2(1+\tau R)}{(1-\gamma)^2}\|\theta-\theta'\|_2,\\
|Q_\tau^{\pi_\theta}-Q_\tau^{\pi_{\theta'}}|_{B_b(S\times A)} &\le\frac{2\gamma(1+\tau R)}{(1-\gamma)^2}\|\theta-\theta'\|_2.
\end{aligned}
\label{eq:local-lip-value-Q-bounds}
\end{equation}

By Lemma~\ref{lem:app:mdp:grad_of_log_policy},
\[
\int_A\nabla\log\frac{\mathrm d\pi_\theta}{\mathrm d\mu}(a|s)
\pi_\theta(\mathrm da|s)=0.
\]
The difference of the log normalizing constants is independent of $a$.
It therefore cancels when integrated against the score, so
\begin{align}
&\int_A
\left(\log\frac{\mathrm d\pi_\theta}{\mathrm d\mu}(a|s)-\log\frac{\mathrm d\pi_{\theta'}}{\mathrm d\mu}(a|s)\right)\nabla\log\frac{\mathrm d\pi_\theta}{\mathrm d\mu}(a|s)\pi_\theta(\mathrm da|s)\notag\\
&=
\int_A\langle\theta-\theta',g(s,a)\rangle\nabla\log\frac{\mathrm d\pi_\theta}{\mathrm d\mu}(a|s)\pi_\theta(\mathrm da|s) =G^{\pi_\theta}(s)(\theta-\theta').
\label{eq:local-lip-log-cancellation}
\end{align}
In particular, the norm of this expression is at most
$\|\theta-\theta'\|_2$.

Subtracting the two gradient formulas~\eqref{eq:wts_func_deriv_V}
and adding and subtracting the corresponding terms gives
\begin{align*}
&(1-\gamma)
\left(\nabla V_\tau^{\pi_\theta}(\rho)-\nabla V_\tau^{\pi_{\theta'}}(\rho)\right)
\\
&=\int_S\int_A\Bigg(Q_\tau^{\pi_\theta}(s,a)-Q_\tau^{\pi_{\theta'}}(s,a)+\tau\log\frac{\mathrm d\pi_\theta}{\mathrm d\mu}(a|s)-\tau\log\frac{\mathrm d\pi_{\theta'}}{\mathrm d\mu}(a|s)\Bigg)\nabla\log\frac{\mathrm d\pi_\theta}{\mathrm d\mu}(a|s)\pi_\theta(\mathrm da|s)d_\rho^{\pi_\theta}(\mathrm ds)
\\
&\quad+\int_S\int_A\left(Q_\tau^{\pi_{\theta'}}(s,a)
+\tau\log\frac{\mathrm d\pi_{\theta'}}{\mathrm d\mu}(a|s)\right) \left(\nabla\log\frac{\mathrm d\pi_\theta}{\mathrm d\mu}(a|s) -\nabla\log\frac{\mathrm d\pi_{\theta'}}{\mathrm d\mu}(a|s)\right) \pi_\theta(\mathrm da|s)d_\rho^{\pi_\theta}(\mathrm ds)
\\
&\quad+\int_S\int_A
\left(Q_\tau^{\pi_{\theta'}}(s,a) +\tau\log\frac{\mathrm d\pi_{\theta'}}{\mathrm d\mu}(a|s)\right)\nabla\log\frac{\mathrm d\pi_{\theta'}}{\mathrm d\mu}(a|s)(\pi_\theta-\pi_{\theta'})(\mathrm da|s) d_\rho^{\pi_\theta}(\mathrm ds)
\\
&\quad+\int_S\int_A\left(Q_\tau^{\pi_{\theta'}}(s,a) +\tau\log\frac{\mathrm d\pi_{\theta'}}{\mathrm d\mu}(a|s)\right)\nabla\log\frac{\mathrm d\pi_{\theta'}}{\mathrm d\mu}(a|s)\pi_{\theta'}(\mathrm da|s)(d_\rho^{\pi_\theta}-d_\rho^{\pi_{\theta'}})(\mathrm ds).
\end{align*}
Using \eqref{eq:local-lip-log-cancellation} for the logarithmic term,
Lemma~\ref{lem:app:mdp:bounded_grad_of_log_policy} for the score bound,
and Lemma~\ref{lem: smooth of log policy} with Lipschitz constant $1$,
we deduce
\begin{align*}
(1-\gamma)\left\|\nabla V_\tau^{\pi_\theta}(\rho)
-\nabla V_\tau^{\pi_{\theta'}}(\rho)\right\|_2
&\le
2|Q_\tau^{\pi_\theta}-Q_\tau^{\pi_{\theta'}}|_{B_b(S\times A)}+\tau\|\theta-\theta'\|_2 + \frac{1+\tau R}{1-\gamma}\|\theta-\theta'\|_2
\\
&\quad \ +
\frac{2(1+\tau R)}{1-\gamma} |\pi_\theta-\pi_{\theta'}|_{b\mathcal K(A|S)}+\frac{2(1+\tau R)}{1-\gamma}|d_\rho^{\pi_\theta}-d_\rho^{\pi_{\theta'}}|_{\mathcal M(S)}.
\end{align*}
Substituting \eqref{eq:local-lip-policy-bound}, \eqref{eq:local-lip-occupancy-bound}, and \eqref{eq:local-lip-value-Q-bounds}, we conclude that
\begin{align*}
\left\|\nabla V_\tau^{\pi_\theta}(\rho)
-\nabla V_\tau^{\pi_{\theta'}}(\rho)\right\|_2
& \le \frac{1}{1-\gamma}\left(\frac{8\gamma(1+\tau R)}{(1-\gamma)^2}+\frac{5(1+\tau R)}{1-\gamma}+\tau\right)\|\theta-\theta'\|_2\\
&=\left(\frac{(5+3\gamma)(1+\tau R)}{(1-\gamma)^3}+\frac{\tau}{1-\gamma}\right)\|\theta-\theta'\|_2.
\end{align*}
Finally, the softmax formula gives
\[
\left|\log\frac{\mathrm d\pi_\theta}{\mathrm d\mu}
\right|_{B_b(S\times A)}\le2\|\theta\|_2.
\]
Every bounded parameter neighborhood therefore admits a common $R$,
which proves local Lipschitz continuity on $\mathbb R^p$.
\end{proof}

\section{Conclusion}

We prove linear convergence of policy gradient for entropy regularized MDPs with log-linear policies on general state and action spaces under $Q^\pi_\tau$-realizability up to a state-dependent baseline and when suitable basis functions are employed. 
This complements existing results for softmax policy gradient methods in the tabular setting~\citep{mei2020global}.
To obtain our results we have established a non-uniform P{\L} inequality for general state and action spaces and carried out novel Lyapunov function-based analysis allowing control of the non-uniform term.

\subsection*{Acknowledgements}
The second author was partially supported by a grant from the Simons Foundation.
The second author would also like to thank the Isaac Newton Institute for Mathematical Sciences, Cambridge, for support and hospitality during the programme Bridging Stochastic Control And Reinforcement Learning: Theories and Applications, where work on this paper was partially undertaken.
The second and third authors acknowledge funding from the UKRI Prosperity Partnerships grant APP43592: AI$^2$-Assurance and Insurance for Artificial Intelligence, which supported this work.

\bibliographystyle{plainnat}
\bibliography{references}

@inproceedings{mei2020global,
  title={On the global convergence rates of softmax policy gradient methods},
  author={Mei, Jincheng and Xiao, Chenjun and Szepesvari, Csaba and Schuurmans, Dale},
  booktitle={International conference on machine learning},
  pages={6820--6829},
  year={2020},
  organization={PMLR}
}

@article{cayci2024convergence,
  title={Convergence of entropy-regularized natural policy gradient with linear function approximation},
  author={Cayci, Semih and He, Niao and Srikant, Rayadurgam},
  journal={SIAM Journal on Optimization},
  volume={34},
  number={3},
  pages={2729--2755},
  year={2024},
  publisher={SIAM}
}

@article{liu2023polyak,
  title={Polyak--{{\L}}ojasiewicz inequality on the space of measures and convergence of mean-field birth-death processes},
  author={Liu, Linshan and Majka, Mateusz B and Szpruch, {\L}ukasz},
  journal={Applied Mathematics \& Optimization},
  volume={87},
  number={3},
  pages={48},
  year={2023},
  publisher={Springer}
}

@inproceedings{leahy2022convergence,
  title={Convergence of policy gradient for entropy regularized {MDP}s with neural network approximation in the mean-field regime},
  author={Leahy, James-Michael and Kerimkulov, Bekzhan and \v{S}i\v{s}ka, David and Szpruch, {\L}ukasz},
  booktitle={International Conference on Machine Learning},
  pages={12222--12252},
  year={2022},
  organization={PMLR}
}

@article{sutton1999policy,
  title={Policy gradient methods for reinforcement learning with function approximation},
  author={Sutton, Richard S and McAllester, David and Singh, Satinder and Mansour, Yishay},
  journal={Advances in neural information processing systems},
  volume={12},
  year={1999}
}

@article{polyak1963gradient,
  title={Gradient methods for minimizing functionals},
  author={Polyak, Boris Teodorovich and others},
  journal={Zhurnal vychislitel’noi matematiki i matematicheskoi fiziki},
  volume={3},
  number={4},
  pages={643--653},
  year={1963}
}

@article{lojasiewicz1963topological,
  title={A topological property of real analytic subsets},
  author={Lojasiewicz, Stanislaw},
  journal={Coll. du CNRS, Les {\'e}quations aux d{\'e}riv{\'e}es partielles},
  volume={117},
  number={87-89},
  pages={2},
  year={1963}
}

@article{agarwal2021theory,
  title={On the theory of policy gradient methods: Optimality, approximation, and distribution shift},
  author={Agarwal, Alekh and Kakade, Sham M and Lee, Jason D and Mahajan, Gaurav},
  journal={The Journal of Machine Learning Research},
  volume={22},
  number={1},
  pages={4431--4506},
  year={2021},
  publisher={JMLRORG}
}

@article{bhandari2024global,
  title={Global optimality guarantees for policy gradient methods},
  author={Bhandari, Jalaj and Russo, Daniel},
  journal={Operations Research},
  year={2024},
  publisher={INFORMS}
}

@article{kunze2011pettis,
  title={A {P}ettis-type integral and applications to transition semigroups},
  author={Kunze, Markus},
  journal={Czechoslovak mathematical journal},
  volume={61},
  number={2},
  pages={437--459},
  year={2011},
  publisher={Springer}
}

@article{golub1973some,
  title={Some modified matrix eigenvalue problems},
  author={Golub, Gene H},
  journal={SIAM review},
  volume={15},
  number={2},
  pages={318--334},
  year={1973},
  publisher={SIAM}
}

@book{dupuis2011weak,
  title={A weak convergence approach to the theory of large deviations},
  author={Dupuis, Paul and Ellis, Richard S},
  year={1997},
  publisher={John Wiley \& Sons}
}

@inproceedings{ziebart2010modeling,
  title     = {Modeling Interaction via the Principle of Maximum Causal Entropy},
  author    = {Ziebart, Brian D. and Bagnell, J. Andrew and Dey, Anind K.},
  booktitle = {Proceedings of the 27th International Conference on Machine Learning},
  pages= {1255--1262},
  year= {2010},
  publisher = {Omnipress}
}

@inproceedings{haarnoja2017reinforcement,
  title={Reinforcement learning with deep energy-based policies},
  author={Haarnoja, Tuomas and Tang, Haoran and Abbeel, Pieter and Levine, Sergey},
  booktitle={International conference on machine learning},
  pages={1352--1361},
  year={2017},
  organization={PMLR}
}

@inproceedings{haarnoja2018soft,
  title={Soft actor-critic: Off-policy maximum entropy deep reinforcement learning with a stochastic actor},
  author={Haarnoja, Tuomas and Zhou, Aurick and Abbeel, Pieter and Levine, Sergey},
  booktitle={International conference on machine learning},
  pages={1861--1870},
  year={2018},
  organization={Pmlr}
}

@article{neu2017unified,
  title={A unified view of entropy-regularized {M}arkov decision processes},
  author={Neu, Gergely and Jonsson, Anders and G{\'o}mez, Vicen{\c{c}}},
  journal={arXiv preprint arXiv:1705.07798},
  year={2017}
}

@inproceedings{geist2019theory,
  title={A theory of regularized {M}arkov decision processes},
  author={Geist, Matthieu and Scherrer, Bruno and Pietquin, Olivier},
  booktitle={International conference on machine learning},
  pages={2160--2169},
  year={2019},
  organization={PMLR}
}

@article{vieillard2020leverage,
  title={Leverage the average: an analysis of {KL} regularization in reinforcement learning},
  author={Vieillard, Nino and Kozuno, Tadashi and Scherrer, Bruno and Pietquin, Olivier and Munos, R{\'e}mi and Geist, Matthieu},
  journal={Advances in Neural Information Processing Systems},
  volume={33},
  pages={12163--12174},
  year={2020}
}

@article{lan2023policy,
  title={Policy mirror descent for reinforcement learning: Linear convergence, new sampling complexity, and generalized problem classes},
  author={Lan, Guanghui},
  journal={Mathematical programming},
  volume={198},
  number={1},
  pages={1059--1106},
  year={2023},
  publisher={Springer}
}

@article{ju2022policy,
  title={Policy optimization over general state and action spaces},
  author={Ju, Caleb and Lan, Guanghui},
  journal={arXiv preprint arXiv:2211.16715},
  year={2022}
}

@inproceedings{agarwal2020optimality,
  title={Optimality and approximation with policy gradient methods in {M}arkov decision processes},
  author={Agarwal, Alekh and Kakade, Sham M and Lee, Jason D and Mahajan, Gaurav},
  booktitle={Conference on learning theory},
  pages={64--66},
  year={2020},
  organization={PMLR}
}

@inproceedings{li2021softmax,
  title={Softmax policy gradient methods can take exponential time to converge},
  author={Li, Gen and Wei, Yuting and Chi, Yuejie and Gu, Yuantao and Chen, Yuxin},
  booktitle={Conference on Learning Theory},
  pages={3107--3110},
  year={2021},
  organization={PMLR}
}

@article{kerimkulov2025fisher,
  title={A {F}isher--{R}ao Gradient Flow for Entropy-Regularised {M}arkov Decision Processes in {P}olish Spaces},
  author={Kerimkulov, Bekzhan and Leahy, James-Michael and \v{S}i\v{s}ka, David and Szpruch, {\L}ukasz and Zhang, Yufei},
  journal={Foundations of Computational Mathematics},
  pages={1--75},
  year={2025},
  publisher={Springer}
}

@article{Giegrich2024,
author = {Giegrich, Michael and Reisinger, Christoph and Zhang, Yufei},
title = {Convergence of Policy Gradient Methods for Finite-Horizon Exploratory Linear-Quadratic Control Problems},
journal = {SIAM Journal on Control and Optimization},
volume = {62},
number = {2},
pages = {1060-1092},
year = {2024},
doi = {10.1137/22M1533517},

URL = { 
    
        https://doi.org/10.1137/22M1533517
    
    

},
eprint = { 
    
        https://doi.org/10.1137/22M1533517
    
    

}
}

@article{li2021sample,
  title={Sample-efficient reinforcement learning is feasible for linearly realizable {MDP}s with limited revisiting},
  author={Li, Gen and Chen, Yuxin and Chi, Yuejie and Gu, Yuantao and Wei, Yuting},
  journal={Advances in Neural Information Processing Systems},
  volume={34},
  pages={16671--16685},
  year={2021}
}

@inproceedings{zanette2020frequentist,
  title={Frequentist regret bounds for randomized least-squares value iteration},
  author={Zanette, Andrea and Brandfonbrener, David and Brunskill, Emma and Pirotta, Matteo and Lazaric, Alessandro},
  booktitle={International Conference on Artificial Intelligence and Statistics},
  pages={1954--1964},
  year={2020},
  organization={PMLR}
}

@inproceedings{yang2019sample,
  title={Sample-optimal parametric {Q}-learning using linearly additive features},
  author={Yang, Lin and Wang, Mengdi},
  booktitle={International conference on machine learning},
  pages={6995--7004},
  year={2019},
  organization={PMLR}
}

@article{doya2000reinforcement,
  title={Reinforcement learning in continuous time and space},
  author={Doya, Kenji},
  journal={Neural computation},
  volume={12},
  number={1},
  pages={219--245},
  year={2000},
  publisher={MIT Press One Rogers Street, Cambridge, MA 02142-1209, USA journals-info~…}
}

@incollection{van2012reinforcement,
  title={Reinforcement learning in continuous state and action spaces},
  author={Van Hasselt, Hado},
  booktitle={Reinforcement learning: State-of-the-art},
  pages={207--251},
  year={2012},
  publisher={Springer}
}

@article{manna2022learning,
  title={Learning in continuous action space for developing high dimensional potential energy models},
  author={Manna, Sukriti and Loeffler, Troy D and Batra, Rohit and Banik, Suvo and Chan, Henry and Varughese, Bilvin and Sasikumar, Kiran and Sternberg, Michael and Peterka, Tom and Cherukara, Mathew J and others},
  journal={Nature communications},
  volume={13},
  number={1},
  pages={368},
  year={2022},
  publisher={Nature Publishing Group UK London}
}

@inproceedings{kurdyka1998gradients,
  title={On gradients of functions definable in o-minimal structures},
  author={Kurdyka, Krzysztof},
  booktitle={Annales de l'institut Fourier},
  volume={48},
  pages={769--783},
  year={1998}
}

@inproceedings{fazel2018global,
  title={Global convergence of policy gradient methods for the linear quadratic regulator},
  author={Fazel, Maryam and Ge, Rong and Kakade, Sham and Mesbahi, Mehran},
  booktitle={International conference on machine learning},
  pages={1467--1476},
  year={2018},
  organization={PMLR}
}

@article{bu2019lqr,
  title={{LQR} through the lens of first order methods: Discrete-time case},
  author={Bu, Jingjing and Mesbahi, Afshin and Fazel, Maryam and Mesbahi, Mehran},
  journal={arXiv preprint arXiv:1907.08921},
  year={2019}
}

@article{hu2023toward,
  title={Toward a theoretical foundation of policy optimization for learning control policies},
  author={Hu, Bin and Zhang, Kaiqing and Li, Na and Mesbahi, Mehran and Fazel, Maryam and Ba{\c{s}}ar, Tamer},
  journal={Annual Review of Control, Robotics, and Autonomous Systems},
  volume={6},
  number={1},
  pages={123--158},
  year={2023},
  publisher={Annual Reviews}
}

@article{sontag2022remarks,
  title={Remarks on input to state stability of perturbed gradient flows, motivated by model-free feedback control learning},
  author={Sontag, Eduardo D},
  journal={Systems \& Control Letters},
  volume={161},
  pages={105138},
  year={2022},
  publisher={Elsevier}
}

@article{lin2025rethinking,
  title={Rethinking the global convergence of softmax policy gradient with linear function approximation},
  author={Lin, Max Qiushi and Mei, Jincheng and Aghaei, Matin and Lu, Michael and Dai, Bo and Agarwal, Alekh and Schuurmans, Dale and Szepesvari, Csaba and Vaswani, Sharan},
  journal={arXiv preprint arXiv:2505.03155},
  year={2025}
}

@article{fox2015taming,
  title={Taming the noise in reinforcement learning via soft updates},
  author={Fox, Roy and Pakman, Ari and Tishby, Naftali},
  journal={arXiv preprint arXiv:1512.08562},
  year={2015}
}

@article{ren2025spectral,
  author={Ren, Zhaolin and Ren, Tongzheng and Ma, Haitong and Li, Na and Dai, Bo},
  journal={IEEE Transactions on Automatic Control}, 
  title={Stochastic Nonlinear Control via Finite-Dimensional Spectral Dynamics Embedding}, 
  year={2025},
  volume={70},
  number={12},
  pages={7907-7922},
  doi={10.1109/TAC.2025.3580296}}

@article{yuan2022linear,
  title={Linear convergence of natural policy gradient methods with log-linear policies},
  author={Yuan, Rui and Du, Simon S and Gower, Robert M and Lazaric, Alessandro and Xiao, Lin},
  journal={arXiv preprint arXiv:2210.01400},
  year={2022}
}

@article{liu2020improved,
  title={An improved analysis of (variance-reduced) policy gradient and natural policy gradient methods},
  author={Liu, Yanli and Zhang, Kaiqing and Basar, Tamer and Yin, Wotao},
  journal={Advances in Neural Information Processing Systems},
  volume={33},
  pages={7624--7636},
  year={2020}
}

@inproceedings{ma2024skill,
  title={Skill transfer and discovery for sim-to-real learning: A representation-based viewpoint},
  author={Ma, Haitong and Ren, Zhaolin and Dai, Bo and Li, Na},
  booktitle={2024 IEEE/RSJ International Conference on Intelligent Robots and Systems (IROS)},
  pages={8603--8609},
  year={2024},
  organization={IEEE}
}

\end{document}